\pdfoutput = 1
\documentclass[11pt]{article}
\usepackage[utf8]{inputenc}
\usepackage{style}

\newcommand{\rgta}{\rightarrow}

\newcommand{\lt}{\left}
\newcommand{\rt}{\right}
\newcommand{\zo}{[0,1]}
\newcommand{\bits}{\{0,1\}}

\newcommand{\fr}[1]{\ensuremath{\frac{1}{#1}}}

\newcommand{\ext}{{\mathsf{ext}}}

\newcommand{\one}{{\mathbf 1}}

\newcommand{\ECE}{\mathsf{ECE}}
\newcommand{\dCE}{\mathsf{dCE}}

\newcommand{\ldCE}{\underline{\dCE}}
\newcommand{\udCE}{\overline{\dCE}}
\newcommand{\kCE}{\mathsf{kCE}}
\newcommand{\kCEG}{\kCE^{\mathsf{Gauss}}}
\newcommand{\kCEL}{\kCE^{\mathsf{Lap}}}
\newcommand{\wCE}{\mathsf{wCE}}
\newcommand{\ekCE}{\widehat{\kCE}}
\newcommand{\ewCE}{\widehat{\wCE}}
\newcommand{\Rad}{\mathcal{R}}

\newcommand{\scerror}{\mathsf{smCE}}

\newcommand{\intCE}{\mathsf{intCE}}

\newcommand{\RintCE}{\mathsf{RintCE}}
\newcommand{\SintCE}{\mathsf{SintCE}}

\DeclareMathOperator{\Ima}{Im}

\newcommand{\poly}{\mathrm{poly}}

\newcommand{\X}{\mathcal{X}}

\newcommand{\mS}{\mathcal{S}}

\newcommand{\mF}{\mathcal{F}}
\newcommand{\mD}{\mathcal{D}}

\newcommand{\mI}{\mathcal{I}}

\newcommand{\mX}{\mathcal{X}}

\newcommand{\FX}{\mathcal{F}_\mathcal{X}}

\newcommand{\level}{\mathrm{level}}
\newcommand{\favg}{\bar{f}}

\newcommand{\wh}{\widehat}

\DeclareMathOperator{\supp}{supp}

\newcommand{\lift}{\mathrm{lift}}

\newcommand{\calD}{\mathsf{cal}(\mathcal{D})}

\newcommand{\lap}{\mathsf{Lap}}
\newcommand{\gauss}{\mathsf{Gauss}}

\newcommand*\Let[2]{\State #1 $\gets$ #2}

\renewcommand{\d}{\,\mathrm{d}}

\title{A Unifying Theory of Distance from Calibration}
\author{
Jaros\l aw B\l asiok\\
Columbia University \and
Parikshit Gopalan\\
Apple \and
Lunjia Hu\\
Stanford University \and
Preetum Nakkiran\\
Apple   
}
\date{}

\newcommand{\citet}{\cite}
\newcommand{\citep}{\cite}

\allowdisplaybreaks
\begin{document}

\maketitle

\begin{abstract}
We study the fundamental question of how to define and measure the distance from calibration for probabilistic predictors. While the notion of \emph{perfect calibration} is well-understood, there is no consensus on how to quantify the distance from perfect calibration. Numerous calibration measures have been proposed in the literature, but it is unclear how they compare to each other, and many popular measures such as Expected Calibration Error (ECE) fail to satisfy basic properties like continuity.

We present a rigorous framework for analyzing calibration measures, inspired by the literature on property testing. We propose a ground-truth notion of distance from calibration: the $\ell_1$ distance to the nearest perfectly calibrated predictor. We define a \emph{consistent calibration measure} as one that is polynomially related to this distance. Applying our framework, we identify three calibration measures that are \emph{consistent} and can be estimated efficiently: smooth calibration, interval calibration, and Laplace kernel calibration. The former two give quadratic approximations to the ground truth distance, which we show is information-theoretically optimal in a natural model for measuring calibration which we term the \emph{prediction-only access} model. Our work thus establishes fundamental lower and upper bounds on measuring the distance to calibration, and also provides theoretical justification for preferring certain metrics (like Laplace kernel calibration) in practice.

\end{abstract}
\thispagestyle{empty}
\newpage

\tableofcontents

\newpage
\section{Introduction}

Probabilistic predictions are central to many domains which involve categorical, even deterministic, outcomes. Whether it is doctor predicting a certain incidence probability of heart disease,
a meteorologist predicting a certain chance of rain,
or an autonomous vehicle system predicting a probability of road obstruction--- probabilistic prediction allows the predictor to incorporate and convey epistemic and aleatory uncertainty in their predictions.

In order for predicted probabilities to be operationally meaningful, and
not just arbitrary numbers, they must be accompanied by some form of formal probabilistic guarantee. The most basic requirement of this form is \emph{calibration}
\citep{dawid1982well}. Given a distribution $\mD$ on $\X \times \bits$ representing points with binary labels, a predictor $f:\X \to \zo$ which predicts the probability of the label being $1$ is calibrated if for every $v \in \Ima(f)$, we have $\E[y|f(x) =v] =v$. Calibration requires that, for example, among the set of patients
which are predicted to have a 10\% incidence of heart disease,
the true incidence of heart disease is exactly 10\%. Calibration is recognized as a crucial aspect of probabilistic predictions in many applications,
from their original development in meteorological forecasting \citep{degroot1983comparison,murphy1984probability,hallenbeck1920forecasting,murphy1998early}, to models of risk and diagnosis in medicine
\citep{jiang2012calibrating,li2004radiologists,doi2007computer,mealiffe2010clinical,kompa2021second,van2015calibration,crowson2016assessing},
to image classification settings in computer vision
\citep{minderer2021revisiting,mustafa2021supervised}. There is a large body of theoretical work on it in forecasting, for example \cite{dawid1985calibration, FosterV98, kakadeF08, FosterH18}. More recently, the work of \cite{hkrr2018} on multicalibration as a notion of group fairness (see also \cite{KleinbergMR17, kearns2017preventing}), and connections to indistinguishability \cite{OI} and loss minimization \cite{omni, op2} have spurred renewed interest in calibration from theoretical computer science.

In practice, it is of course rare to encounter \emph{perfectly} calibrated predictors, and thus it is important to quantify their \emph{distance from calibration}.
However, while there is consensus across domains on what it means for a predictor to be \emph{perfectly calibrated},
there is no consensus even within a domain on how to measure this distance.
This is because the
commonly-used metrics of calibration have fundamental theoretical flaws,
which manifest as practical frustrations.
Consider the Expected Calibration Error (ECE), which is
the de-facto standard metric in the machine learning community
(e.g. \citep{naeini2014binary,guo2017calibration,minderer2021revisiting,rahaman2021uncertainty}).

\begin{definition}
\label{def:ece}
    For a predictor $f: \cX \to [0, 1]$ and distribution $\cD$ over $(x,y) \in \cX \x \bits$, the expected calibration error $\ECE_\cD(f)$ is defined as
    \[ \ECE_\cD(f) = \E_\cD\left[\left|\E_\cD[y \mid f(x)]  - f(x)\right|\right] .\]
\end{definition}

The ECE has a couple of flaws: First, 
it is impossible to estimate in general from finite samples (e.g. \citep[Proposition 5.1]{lee2022t} and \citep{arrieta2022metrics}). This is partly because estimating the conditional expectation $\E[y|f(x) = v]$ requires multiple examples with the exact same prediction $f(x) = v$, which could happen with arbitrarily low probability in a fixed and finite number of examples over a large domain $\cX$.
Second,  the ECE is \emph{discontinuous} as a function of the predictor $f$, as noted by \cite{kakadeF08, FosterH18}.
That is, arbitrarily small perturbations to the predictor $f$ can cause large fluctuations in $\ECE_\cD(f)$. We illustrate this with a simple example below.

Consider the uniform distribution over a two-point space $X = \{a,b\}$, with the label for $a$ is $0$ whereas for $b$ it is $1$. The predictor $\favg$ which always predicts $1/2$ is perfectly calibrated under $\mD$, so $\ECE(\favg) =0$. 
In contrast, the related predictor
where $f(a) = (1/2 - \varepsilon)$ and $f(b) = (1/2 + \varepsilon)$, for arbitrarily small $\eps > 0$, has $\ECE(f) = 1/2 - \eps$.
Thus the infinitesimal change from $\favg$ to $f$
causes a jump of almost $1/2$ in $\ECE$. 

This discontinuity also presents a barrier to popular heuristics for estimating the $\ECE$. For example, the estimation problem is usually handled by discretizing the range of $f$, yielding an alternate quantity --- the ``binned-ECE''--- that can be estimated from samples \citep{naeini2015obtaining}. However,
the choice of discretization turns out to matter significantly
both in theory \citep[Example 3.2]{kumar2019verified} and in practice
\citep{nixon2019measuring}.
For example, both \citep[Section 5]{minderer2021revisiting} and \citep[Section 5.3]{nixon2019measuring} found
that changing the number of bins used in binned-ECE can change conclusions
about which of two models is better calibrated. In the simple two-point example above, if we choose $m$ bins of equal width, then we observe a binned-ECE of either $0$ or $\approx 1/2$, depending on whether $m$ is odd or even!

To address the shortcomings of the ECE, a long line of works
have proposed alternate metrics of miscalibration.
These metrics take a diversity of forms:
some are based on modifications to the ECE
(e.g. alternate binning schemes, debiased estimators, or smoothing)
\citep{zhang2020mix,kumar2019verified,roelofs2022mitigating,karandikar2021soft,lee2022t},
some use proper scoring rules,
some rely on distributional tests such as 
Kolmogorov–Smirnov \citep{gupta2020calibration}, Kernel MMD \citep{ksj18a}, or other nonparametric tests \citep{arrieta2022metrics}.
Yet it is not clear what to make of this smorgasbord of calibration metrics:
whether these different metrics are at all related to each other,
and whether they satisfy desirable properties (such as continuity). For a practitioner training a model, if their model is calibrated under some of these notions, but not others, what are they to make of it? Should they report the most optimistic metrics, or should they strive to be calibrated for all of them? Or is there some inherent but undiscovered reason why all these metrics should paint a similar picture? 

Underlying this confusion is a foundational question: {\em what is the  ground truth distance of a predictor from calibration?}  To our knowledge, this question has not been answered or even asked in the prior literature. Without a clearly articulated ground truth and a set of desiderata that a calibration measure must satisfy, we cannot hope to have meaningful comparisons among metrics.  

At best one can say that $\ECE$ and (certain but not all) binning based variants give an upper bound on the true distance to calibration; we prove this formally for $\ECE$ in \cref{sec:ece}. Thus if a predictor can be guaranteed to have small $\ECE$, then it is indeed close to being calibrated in a formal sense (see for instance \cite[Claim 2.10]{hkrr2018}). 
But small $\ECE$ might an unnecessarily strong (or even impossible) constraint to satisfy
in many realistic settings, especially when dealing with predictors
which are allowed to produce real-valued outputs.
For example, consider the standard setting of a deep neural network
trained from random initialization for binary classification.
The predicted value $f(x) \in [0, 1]$ is likely to be different for every individual $x$ in the population, which could result in a similar situation to our example. The $\ECE$ is likely to greatly overstate the true distance from calibration in such a setting.

This brings us to the main motivations behind this work. We aim to:
\begin{itemize}
    \item Formulate desiderata for good calibration measures,
    based on a rigorous notion of ground truth distance from calibration.
    \item Use our desiderata to compare existing calibration measures, 
    identifying measures that are good approximations to the ground truth distance. 
    \item Apply theoretical insights to inform practical measurements
    of calibration in machine learning,
    addressing known shortcomings of existing methods.
\end{itemize}

\subsection{Summary of Our Contributions}

We summarize the main contributions of our work:
\begin{itemize}
    \item {\bf Framework for measuring the distance to calibration (Section~\ref{sec:framework}).} We propose a ground truth notion of distance to calibration which is the $\ell_1$ distance to the closest perfectly calibrated predictor, inspired by the property testing literature \cite{BlumLR90, GoldreichGR98}. We define the set of \emph{consistent  calibration measures} to be those that provide polynomial upper and lower bounds on the true distance. 
    \item {\bf Consistent calibration measures.} We identify three calibration measures that are in fact consistent: two have been proposed previously \cite{kakadeF08, ksj18a} and the third is new. 
    Interestingly, the two prior measures (smooth and kernel calibration) were proposed with other motivations in mind,
    and not as standalone calibration measures.
    We consider it surprising that they turn out
    to be intimately related to the ground truth $\ell_1$ distance.
   \begin{enumerate}
   \item {\bf Interval calibration error (Section~\ref{sec:interval-calibration})}. This is a new measure which is reminiscent of the binning estimate that is popular in practice \cite{naeini2014binary,guo2017calibration,minderer2021revisiting}. We show that by randomizing both the width of each bin and using a random offset, and by adding the average bin width to the resulting calibration error, one can derive a consistent estimator that this is always an upper bound on the true distance, and it is never more than the square root of the true distance. 
    
    \item {\bf Smooth calibration error (Section~\ref{sec:smooth-ce-duality})} was proposed in the work of \cite{kakadeF08}. We show using LP duality that it is a constant factor approximation of the \emph{lower} distance to calibration, which we define to be, roughly speaking, a particular Wasserstein distance to perfect calibration. The lower distance to calibration is always at most the true distance to calibration and is always at least a constant times the true distance squared.
    
    \item {\bf Laplace-kernel calibration error (Section~\ref{sec:kernels}).} This is a calibration measure that was proposed in the work of \cite{ksj18a}. While they did not recommend a particular choice of kernel, we show that using the Laplace kernel happens to yield a consistent measure,
    while using the Gaussian kernel does not.
    \end{enumerate}
    In contrast to these measures, other commonly used heuristics ($\ECE$ and binning based) do not meet our criteria for being consistent calibration measures. 
Our work thus provides a firm theoretical foundation on which to base evaluations and comparisons of various calibration measures; such a foundation was arguably lacking in the literature.

\item {\bf Matching lower bounds.}
Smooth calibration and interva calibration
provide quadratic approximations to the true distance from calibration.
We prove that this is the best possible approximation,
by showing an information-theoretic barrier:
    for calibration measures depending only on the labels $y$ and predictions $f$, which are oblivious to the points $x$ themselves (as most calibration measures are),
    it is impossible to obtain better than a quadratic approximation to the true distance
    from calibration.
    Thus, the measures above are in fact optimal in this sense.

    \item {\bf Better efficiency in samples and run time.} We present improved algorithms and sample complexity bounds for computing some calibration measures. We present the first efficient algorithm for computing smooth calibration error using a linear program. 
    We also observe %
    that the techniques
    of \cite{rahimi2007kernel}
    yield an alternate algorithm to computing kernel calibration error
    which is (somewhat surprisingly) reminiscent of randomized binning.

    \item {\bf Insights for Practice.}
Our results point to concrete takeaways for practical measurements of calibration.
First, we recommend using either
Laplace kernel calibration or Interval calibration,
as calibration measures that are theoretically consistent, computationally efficient, and simple to implement.
Second, if Binned-ECE must be used, we recommend 
randomly shifting the bin boundaries together, and adding the average width of the bins
to the calibration estimate.
These modifications turn binning into a upper-bound on calibration distance,
and bring it closer to  interval calibration error which is a consistent calibration measure (Section~\ref{sec:summ-consistent}).
Finally, in Section~\ref{sec:experimental} we experimentally evaluate
our calibration measures on a family of synthetic data distributions,
to demonstrate their behavior in more natural settings (beyond worst-case guarantees).
    
\end{itemize}

\paragraph{Organization of this paper.} 
The rest of this paper is organized as follows. 
In \Cref{sec:overview} we present an informal overview of our main results, highlighting the definitions and key conceptual ideas. 
We discuss related works in \Cref{sec:related-work}.
\Cref{sec:framework} sets up our notion of true distance from calibration and the desiderata that we seek from calibration measures. We also explain how $\ECE$ and some other measures fail these desiderata. \Cref{sec:pa} defines the upper and lower distance to calibration. \Cref{sec:int-ce} analyzes Interval Calibration error, \Cref{sec:smooth-ce-duality} analyzes Smooth calibration error and \Cref{sec:kernels} analyzes the Laplace kernel calibration error. In \Cref{sec:sample-complexity} we give sample complexity bounds and efficient algorithms for estimating various calibration measures using random sample. 
And in \Cref{sec:experimental}, we experimentally evaluate our calibration measures on a representative family of synthetic data distributions.
Throughout, we defer technical proofs to \Cref{sec:proofs}.

For the reader primarily interested in 
using calibration measures in practice,
we suggest the following fast-track: %
read \Cref{sec:overview} followed by  Section~\ref{sec:practical-laplace}
for a practical implementation note, and Section~\ref{sec:experimental} for
example experiments. Users of Binned-ECE
may be interested in the relevant parts of Section \ref{sec:summ-consistent}, which 
describes how to ``fix'' the standard binning procedure to
be more closely related to calibration distance.

\section{Overview of Our Results}
\label{sec:overview}

We start by setting up some notation for calibration in the binary classification setting.
Let $\X$ be a discrete domain, defining the input space.
We are given samples $(x, y)$ drawn from a distribution $\mD$ on $\X \times \bits$.
A \emph{predictor} is a function $f:\X \rgta \zo$, where $f(x)$ is interpreted
as an estimate of $\Pr[ y= 1 \mid x]$. 
For a predictor $f$ and distribution $\mD$,
we often consider the induced joint distribution
of prediction-label pairs $(f(x), y) \in [0, 1] \x \bits$,
which we denote $\mD_f$.
We say a prediction-label distribution $\Gamma$ over $[0,1]\times\{0,1\}$ is \emph{perfectly calibrated} if $\E_{(v,y)\sim\Gamma}[y \mid v] = v$.
For a distribution $\mD$ over $\mX\times\{0,1\}$, we say a predictor $f:\mX\to[0,1]$ is \emph{perfectly calibrated w.r.t.\ $\mD$} if the induced distribution $\mD_f$ is perfectly calibrated.
Finally, a \emph{calibration measure} $\mu$ is a function that maps a distribution $\mD$ and a predictor $f:\X \to \zo$ to a value $\mu_\mD(f) \in [0,1]$.

\subsection{Framework for Measuring Distance from Calibration}
\label{sec:summ-framework}

The primary conceptual contribution  of this work is a formal framework in which we can reason about and compare various measures of calibration. We elaborate upon the key ingredients of this framework.

\paragraph{The true distance to calibration.}
We define the ground truth distance from calibration as the distance to the closest calibrated predictor. Measuring distance requires a metric on the space of all predictors. 
A natural metric  is the $\ell_1$ metric given by $\ell_1(f,g) = \E_{\mD}|f(x) - g(x)|$.
Accordingly we define the true distance from calibration as
\begin{equation}
\label{eqn:dce}
\dCE_\mD(f) := \inf_{g \in \calD} \E_\mD|f(x) - g(x)|,
\end{equation}
where $\calD$ denotes the set of predictors that are perfectly calibrated w.r.t.\ $\mD$.
This definition is intuitive, and natural from a property testing point of view \cite{BlumLR90, GoldreichGR98, ParnasRR06}, but has not been proposed before to our knowledge.
Note that it is not clear how to compute this distance efficiently:
the set $\calD$ is non-convex, and in fact it is discrete when the domain $\X$ is discrete. A more subtle issue is that it depends on knowing the domain $\X$, whereas traditionally calibration measures only depend on the joint distribution $\mD_f$ of predictions and labels.

\paragraph{Access model.}
Calibration measures $\mu_\mD(f)$ can depend on the entire distribution
$\mD$, as well as on the predictor $f:\X \to \zo$.
However, we would prefer measures which only depend on the
prediction-label joint distribution $\mD_f$, similar to
standard loss functions in machine learning and classic calibration measures  \cite{dawid1984present, dawid1985calibration, FosterV98}.
This distinction has important consequences for the power of calibration measures, which we describe shortly.
We delineate the two levels of access as follows:
\begin{enumerate}
    \item {\bf Sample access (SA).}
In the SA model, $\mu_\mD(f)$ is allowed to depend on the full joint distribution
$(x, f(x), y)$ for $(x,y)\sim \mD$. This terminology follows \cite{OI}.
    \item {\bf Prediction-only access (PA).}
In the PA model, $\mu_\mD(f)$ is only allowed to depend on $\mD_f$, the joint distribution
$(f(x), y)$ for $(x,y)\sim \mD$.
In particular, $\mu$ cannot depend on the input domain $\X$.
\end{enumerate}

Observe that the ground truth distance ($\dCE$) is defined in the sample access model,
since Equation~\eqref{eqn:dce} depends on the domain and distribution of $x$.
On the other hand, we often desire measures that can be computed in the prediction access model.

    \paragraph{Robust completeness and soundness.} 
    We propose two desiderata for any calibration measure $\mu$: robust completeness and robust soundness, in analogy to completeness and soundness in proof systems. 
    \begin{enumerate}
        \item 
    {\bf Robust completeness} requires $\mu_\mD(f) \leq \Oh(\dCE_\mD(f)^c)$ for some constant $c$.
    This guarantees that any predictor which is close to a perfectly calibrated predictor (in $\ell_1$) has small calibration error under $\mu$. 
    This is a more robust guarantee than standard completeness, which in this setting 
    would mean just that $\dCE_\mD(f) =0$ implies $\mu_\mD(f) = 0$, but would not give any guarantees
    when $\dCE_\mD(f)$ is non-zero but small.
    \item {\bf Robust soundness} requires $\mu_\mD(f) \geq \Omega(\dCE_\mD(f)^s)$ for some constant $s$. 
    That is, if $\dCE_\mD(f)$ is large then so is $\mu_\mD(f)$.
    \end{enumerate}

   We call a calibration measure \emph{consistent} (or more precisely, \emph{$(c, s)$-consistent})
   if it satisfies both robust completeness and robust soundness, for some parameters $c, s > 0$:
    \begin{equation}
    \label{eqn:consistency}
    \Omega(\dCE_\mD(f)^s)  \leq \mu_\mD(f) \leq \Oh(\dCE_\mD(f)^c).
    \end{equation}
   Consistent measures are exactly those that are polynomially-related to the true distance from calibration, $\dCE_\mD(f)$.
The reader might wonder if, in our definition of consistent calibration measures 
(Equation~\ref{eqn:consistency}),
we could require \emph{constant factor} approximations to $\dCE_\mD(f)$
rather than polynomial factors.
It turns out that there are information-theoretic barriers to such approximations
in the prediction-access model.
The core obstacle is that the true distance $\dCE$ is defined in the SA model, and one cannot compute it exactly in the prediction-only access model or approximate it within a constant factor. Indeed, we show that any calibration measure computable in the prediction-access must satisfy $s/c \geq 2$: information theoretically, a quadratic approximation is the best possible. 

Another nice property of our definition is that the set of all consistent measures stays the same, even if we define distances between predictors using the $\ell_p$ metric for $p > 1$ in place of $\ell_1$, since all $\ell_p$ measures are polynomially related.

 \paragraph{Desiderata for calibration measures.}
 Given the discussion so far, we can now stipulate three desiderata
 that we would like calibration measures $\mu$ to satisfy:
\begin{enumerate}
    \item {\bf Access:} $\mu$ is well-defined in the Prediction-only access model (PA).
    \item {\bf Consistency:} $\mu$ is $(c, s)$-consistent -- that is, $\mu$ is polynomially related to the true distance from calibration $\dCE$.
    Ideally we have $s/c =2$, which is optimal in the PA model (\Cref{cor:quadratic}).
    \item {\bf Efficiency:} $\mu_\mD(f)$ can be computed within accuracy $\eps$ in time $\poly(1/\eps)$ using $\poly(1/\eps)$ random samples from $\mD_f$.
\end{enumerate}
Various notions that have been proposed in the literature fail one or more of these desiderata; $\ECE$ for instance fails robust completeness since an arbitrarily small perturbation of a perfectly calibrated predictor could result in  high $\ECE$. We refer the reader to Table \ref{tab:priorworks} for a more complete treatment of such notions.

\subsection{Information-theoretic Limitations of the Prediction-access Model}

\paragraph{Upper and Lower Distances.}
We start with the following question, which formalizes how well one can approximate the true distance to calibration in the prediction-only access (PA) model. 

{\em For a given distribution $\mD$ and predictor $f$,
how large or small can $\dCE_{\mD'}(f')$ be, among all other $(\mD', f')$
which have the same prediction-label distribution ($\mD'_{f'}=\mD_f$)?}

We denote the minimum and maximum by $\ldCE_\mD(f)$ and $\udCE_\mD(f)$ respectively, which we call the lower and upper distance to calibration respectively. Hence
\begin{align}
\label{eq:chain1}
    \ldCE_{\mD}(f) \leq \dCE_{\mD}(f) \leq \udCE_{\mD}(f).
\end{align}
Both these quantities are defined in the PA model,
in which they represent the tightest lower and upper bounds respectively that one can prove on $\dCE$. 
As framed, they involve considering all possible domains and distributions $\mD'$ over them. But we can give simpler characterizations of these notions.

The upper distance $\udCE$ can be alternatively viewed as the minimum distance to calibration via post-processing: it is the distance to the closest calibrated predictor $g \in \calD$ such that $g =\kappa(f)$ can be obtained from $f$ by {\em post-processing} its predictions. For the lower distance $\ldCE$, we can abstract away the domain and ask only for a coupling between $f$ and a perfectly calibrated predictor: 

{\em Consider all joint distributions $\Pi$ of $(u, v, y)$ over $\zo \times \zo \times \bits$ where $(v, y) \sim \mD_f$ and the distribution of $(u,y)$ is perfectly calibrated. How small can $\E|u - v|$ be?}

Limiting ourselves to couplings of the form $(g(x), f(x), y) \sim \mD$ where $g \in \calD$ would recover $\dCE$. Our definition also permits couplings that may not be realizable on the domain $\X$, giving a lower bound.

An equivalent view of these distances is that the upper distance only considers those calibrated predictors whose level sets are obtained by a coarsening of the level sets of $f$. The lower distance allows calibrated predictors that are obtained by a finer partitioning of the level sets of $f$. 
Theorem \ref{thm:dce-upper-lower} proves the equivalence of these various formulations
of $\ldCE$ and $\udCE$.

\paragraph{A Quadratic Barrier.}
How tight are the lower and upper bounds in Equation \eqref{eq:chain1}?
That is, how tightly can $\dCE$ be determined in the Prediction-only access model?
In Lemma \ref{lem:pa-gap}, we show that
there can be at least a quadratic gap in between any two adjacent terms in
Equation \eqref{eq:chain1}.
We construct two distributions $\mD^1$ and $\mD^2$ and a predictor $f$ such that 
\begin{itemize}
    \item $\mD^1_f =\mD^2_f$, so the upper and the lower distance are equal for both distributions, but they are well-separated from each other;  $\ldCE_{\mD^i}(f) = \Theta(\alpha^2)$ whereas $\udCE_{\mD^i}(f) = \Theta(\alpha)$ for $i \in \{1,2\}$.
    \item $\dCE_{\mD^i}(f)$ equals either $\ldCE$ or $\udCE$ depending on whether $i = 1$ or $2$.  
\end{itemize}
This example raises the question of whether an even bigger gap can exist, which we answer next.

\subsection{Consistent Calibration Measures}
\label{sec:summ-consistent}
We describe three consistent calibration measures,
and their relation to the true distance from calibration.

\paragraph{Interval calibration (Section~\ref{sec:interval-calibration})}

Interval calibration error is a subtle modification to the heuristic of binning predictions into buckets and computing the expected calibration error.
Formally, given a partition $\mI = \{I_1, \ldots, I_m\}$ of $[0,1]$ into intervals of width bounded by $w(\mI)$, we first consider the standard quantity
\begin{equation}
\label{eqn:binnedECE}
\mathsf{binnedECE}_\mD(f, \mI) = \sum_{j \in [m]} |\E[(f - y)\one(f \in I_j)]|.
\end{equation}
This quantity, as the name suggests, is exactly the Binned-ECE for the bins defined by the partition $\mI$.
We then define our notion of \emph{Interval calibration error} ($\intCE$)
as the minimum of this Binned-ECE over all partitions $\mI$, when ``regularized'' by maximum bin width $w(\mI)$:
    \[
    \intCE_\mD(f) := \inf_{\mI: ~\textrm{Interval partition}} \left(\mathsf{binnedECE}_\mD(f, \mI) + w(\mI)\right).
    \]
In Theorem \ref{thm:intce}, we show that $\intCE$ satisfies the following bounds.
\[ \udCE_\mD(f) \leq \intCE_\mD(f) \leq 4\sqrt{\ldCE_\mD(f)}.\]
This shows that the measure $\intCE$ is $(1/2,1)$-consistent, and gives the best possible (quadratic) approximation to the true distance to calibration. The outer inequality implies that the gap between the lower and upper distance is no more than quadratic, hence the gap exhibited in Lemma \ref{lem:pa-gap} is tight. 

We now address the computational complexity. While the definition of interval calibration minimizes over all possible interval partitions,
in Section \ref{sec:surrogate}, we show that it suffices to 
consider a geometrically decreasing set of values for the width $w$, 
with a random shift, to get the desired upper bound on $\dCE$.

Our result suggests an additional practical takeaway:
if the standard binning algorithm must be used to measure calibration, then 
the bin width should be added to the binnedECE.
This yields a quantity which is at least an \emph{upper bound} on the true distance to calibration,
which is not true without adding the bin widths.
Specifically, for \emph{any} interval partition $\mI$, we have:
\begin{align}
\label{eqn:binnedECEw}
    \udCE_\cD(f) \leq \mathsf{binnedECE}(f, \mI) + w(\mI).
\end{align}
Thus, if we add the bin width, then $\mathsf{binnedECE}$ can at least be used to certify closeness to calibration.
The extreme case of width $0$ buckets corresponds to $\ECE$,
while the case when the bucket has width $1$ corresponds to the weaker condition of accuracy in expectation \cite{hkrr2018}.
It is natural to penalize larger width buckets which allow cancellations between calibration errors for widely separated values of $f$.  
The notion of using bucket width as a penalty to compare calibration error results obtained from using differing width buckets is intuitive in hindsight, but not done in prior work to our knowledge (e.g. \cite{minderer2021revisiting}).

\paragraph{Smooth Calibration (Section~\ref{sec:smooth-ce-duality})}

Smooth calibration is a calibration measure first defined by \cite{kakadeF08}, see also \cite{FosterH18, GopalanKSZ22}. 
Smooth calibration error is defined as the following maximization
over the family $L$ of all bounded $1$-Lipschitz functions $w: \zo \to [-1,1]$:
\[ \scerror_{\mD}(f) := \scerror(\mD_f) = \sup_{w \in L}\E_{(v, y) \sim \mD_f} [w(v)(y - v)].\]
Without the Lipschitz condition on functions $w$,
this definition would be equivalent to $\ECE(f)$. Adding the Lipschitz condition smooths out the contribution from each neighborhood of $v$ and results in a calibration measure that is Lipschitz in $f$ with respect to the $\ell_1$ distance. This notion has found applications in game theory and leaky forecasting \cite{kakadeF08, FosterH18}. Our main result is that the smooth calibration error captures the lower distance from calibration up to constant factors: 
\[ \frac 12 \, \ldCE_\mD(f) \leq \scerror_\mD(f) \leq 2 \, \ldCE_\mD(f). \]
We find this tight connection to be somewhat surprising, since
$\scerror$ (as a maximization over weight functions $w$)
and $\ldCE$ (as a minimization over couplings) have {\em a priori} very different definitions.
They turn out to be related via LP duality, in a way analogous
to Kantorovich-Rubinstein duality of Wasserstein distances. We present a high-level overview of the proof at the start of \Cref{sec:smooth-ce-duality}. Along the way, we give an efficient polynomial time algorithm for estimating the smooth calibration error, the first such algorithm to our knowledge. 
To summarize the relations between notions discussed so far,
we have
\begin{align}
\label{eqn:main-ineq}
\boxed{
    \scerror \approx \ldCE \leq \dCE \leq \udCE \leq \intCE
    \leq 4\sqrt{\ldCE}
    }
\end{align}
For each of the first three inequalities, we show that the gap can be quadratic. The final inequality
shows that these gaps are \emph{at most} quadratic.

\paragraph{Kernel Calibration (Section~\ref{sec:kernels})}

The notion of kernel calibration error was introduced in \cite{ksj18a} as \emph{Maximum Mean Calibration Error} (MMCE). 
Kernel calibration can be viewed as a variant of smooth calibration error, where we use as weight functions
$w: [0, 1] \to \R$ which are bounded with respect to a norm $\|\cdot\|_{K}$ on the \emph{Reproducing Kernel Hilbert Space} associated with some positive-definite kernel $K$:
\[ \kCE^K_{\mD}(f) := \sup_{w: \|w\|_K \leq 1} ~\E_{(v, y) \sim \mD_f} [w(v)(y - v)].\]

When $\mD_f$ is an empirical distribution over samples $\{(v_1, y_1), \ldots, (v_n, y_n)\}$,
this can be computed as
\[
\kCE^K_\mD(f) = \sqrt{\frac{1}{n^2}\sum_{i,j} (y_i - v_i)(y_j - v_j) K(v_i, v_j)}.
\]

The original motivation of introducing the kernel calibration error was to provide a differentiable proxy for ECE --- allowing for the calibration error to be explicitly penalized during the training of a neural network. 
However, \cite{ksj18a} does not discuss how the choice of the kernel affects the resulting measure, although they used Laplace kernel in their experiments. We prove here that this choice has strong theoretical  justification ---  the kernel calibration error with respect to Laplace kernel is a consistent calibration measure; specifically for some positive absolute constants $c_1,c_2 > 0$,
\[ c_1 \ldCE(f) \leq \kCEL(f) \leq c_2 \sqrt{\ldCE(f)}.\] 
This says that we can view kernel calibration with respect to the Laplace kernel as fundamental measure in its own right, as opposed to a proxy for (the otherwise flawed) $\ECE$.
We also show that the choice of kernel is in fact crucial:
for the Gaussian kernel, another commonly used kernel across machine learning, the resulting measure is not robustly sound anymore
(\Cref{thm:gauss-lb-body}).

\subsection{Better Algorithms and Sample Complexity \label{sec:alg-intro}}

For many of the measures discussed in the paper, we provide efficient algorithms yielding an $\varepsilon$ additive approximation to the measure in question, using samples from the distribution $\mD_f$. In most cases, those results follow a two step paradigm, we give an algorithm that approximates the measure on a finite sample, followed by a generalization bound.
Our generalization bounds follow from essentially standard bounds on Rademacher complexity of the
function families involved in defining our measures (e.g. bounding the Rademacher complexity of 1-Lipshitz functions for $\scerror$).
On the algorithmic side, we prove that the $\scerror$ on the empirical distribution over a sample of size $n$ can be computed by solving a linear program with $\Oh(n)$ variables and constraints. Similarly, the $\ldCE$ can be approximated up to an error $\varepsilon$, by linear time prepossessing followed by solving a linear program with $\Oh(\varepsilon^{-1})$ variables and constrains.

We provide an alternate algorithm
for estimating the kernel calibration error with Laplace kernel,
using the \emph{Random Features Sampling} technique from \cite{rahimi2007kernel}.
This algorithm does not improve on naive estimators 
in worst-case guarantees, but it reveals
an intriguing connection.
After unwrapping the random features abstraction, the final algorithm is similar to the popular interval binning calibration estimator, where we choose the length of the interval at random from a specific distribution, and introduce a uniformly random shift.
We find it surprising that an estimate of this type is  \emph{exactly} equal to $(\kCEL)^2$ in expectation. 

\section{Related Work}
\label{sec:related-work}

\newcommand{\cmark}{\ding{51}}%
\newcommand{\xmark}{\ding{55}}
\begin{table}[t]
\centering
\begin{tabular}{lccc} \toprule
Metric      & Continuity & Completeness & Soundness \\ \midrule 
($\ell_p$-)ECE         & \xmark & \cmark & \cmark \\
Binned-ECE  & \xmark & \cmark & \xmark \\
Proper Scoring Rules (Brier, NLL) & \cmark &  \xmark  & \cmark \\
NCE \citep{siu1997improved} & \cmark & \xmark & \cmark \\
ECCE \citep{arrieta2022metrics} & \xmark & \cmark & \cmark \\
MMCE \citep{ksj18a} & \cmark & \cmark & \cmark \\
smCE \citep{kakadeF08} & \cmark & \cmark & \cmark \\
\bottomrule
\end{tabular}
\caption{Calibration measures proposed in, or based on, prior works.}
\label{tab:priorworks}
\end{table}

We start by discussing the high-level relation between our work and other areas of theoretical computer science, and then discuss work on calibration in machine learning and forecasting.

\paragraph{Property testing.} Our framework for defining the distance to calibration is inspired by the elegant body of literature on property testing \cite{BlumLR90, RubinfeldS96, GoldreichGR98}. Indeed, the notions of ground truth distance to calibration, robust completeness and robust soundness are in direct correspondence to notions in the literature on tolerant property testing and distance estimation \cite{ParnasRR06}. Like in property testing, algorithms for estimating calibration measures operate under stringent resource constraints, although the constraints are different. In property testing, the algorithm only has a local view of the object based on a few queries. In our setting, the constraint comes from having the operate in the prediction-only access model whereas the ground truth distance is defined in the sample-access model.

\paragraph{Multicalibration.} Recent interest in calibration and its variants in theoretical computer science has been spurred by the work of \cite{hkrr2018} introducing multicalibration as a group fairness notion (see also \cite{KleinbergMR17, kearns2018}). This notion has proved to be unexpectedly rich even beyond the context of fairness, with connections to indistinguishability \cite{OI} and loss minimization \cite{omni}. Motivated by the goal of finding more computationally efficient notions of multicalibration, notions such as low-degree multicalibration \cite{GopalanKSZ22} and calibrated multiaccuracy have been analyzed in the literature \cite{op2}, some of these propose new calibration measures.

\paragraph{Level of access to the distribution.}
The sample access model is considered in the work of \cite{OI}, who relate it to the notion of multicalibration \cite{hkrr2018}. Prediction-only access is a restriction of sample access which is natural in the context of calibration, and is incomparable to the no-access model of \cite{OI} where on gets access to point label pairs. This model is not considered explicitly in \cite{OI}, and the name prediction-only access for it is new. But the model itself is well-studied in the literature on calibration \cite{dawid1985calibration, Dawid, FosterV98}, indeed all existing notions of calibration that we are aware of are defined in the PA model, as are the commonly used losses in machine learning.

\paragraph{Prior work on calibration measures.}

Several prior works have proposed alternate measures of calibration
(Table \ref{tab:priorworks} lists a few of them).
Most focus on  the {\em how}: they give a formula or procedure for computing of the calibration measure from a
finite set of samples, sometimes accompanied by a generalization guarantee that connects it to some property of the population. There is typically not much justification for {\em why} the population quantity is a good measure of calibration error, or discussion of its merits relative to other notions in the literature (notable exceptions are the works of \cite{kakadeF08, FosterH18}). The key distinction in our work is that we start from a clear and intuitive ground truth notion and desiderata for calibration measures, we analyze measures based on how well they satisfy these desiderata, and then give efficient estimators and generalization guarantees for consistent calibration measures. 

Our desiderata reveal important distinctions between measures that were proposed previously; Table~\ref{tab:priorworks} summarizes how well calibration measures suggested in prior works satisfy our desiderata. It shows that a host of calibration measures based on variants of $\ECE$, binning and proper scoring rules fail to give basic guarantees.
We present these guarantees formally in \cref{sec:ece}.
Briefly, many ECE variants suffer from the same flaws as ECE itself, and 
proper scoring rules suffer different issues we describe below.
We also discuss other notions of calibration from the literature in \cref{app:related}.

\textbf{Proper Scoring Rules.}
Proper scoring rules such as the Brier Score \citep{brier1950verification}
or Negative-Log-Loss (NLL) are popular proxies for miscalibration.
Every proper scoring rule satisfies soundness, since if the score is $0$,
the function $f$ is perfectly calibrated.
However, such rules violate completeness: there are perfectly-calibrated functions for
which the score is non-zero.
For example, if the true distribution on labels $p(y|x) = \textrm{Bernoilli}(0.5)$,
then the constant function $f(x) = 0.5$ is perfectly calibrated but has non-zero Brier score.
This is because proper scoring rules measure predictive quality, not just calibration.
The same holds for Normalized Cross Entropy (NCE), which is sound but not complete.

{\bf Smooth and Kernel Calibration.} We show that some  definitions in the literature do satisfy our desiderata--- namely the notions of {\em weak calibration} (smooth calibration in our terminology) introduced in \citep{kakadeF08}, and {\em MMCE} (calibration with a Laplace kernel in our terminology) introduced in \citep{ksj18a}.
Smooth calibration was introduced under the name ``weak calibration'' in \cite{kakadeF08}, the terminology of smooth calibration is from \cite{GopalanKSZ22}\footnote{\citep{FosterH18} introduced a notion of ``smooth calibration'' with an unrelated definition, but thankfully, they proved that their ``smooth calibration'' is in fact polynomially related to the \cite{GopalanKSZ22} notion --- therefore in our framework it is also a consistent calibration measure.}.
Interestingly, these were developed with different motivations.
MMCE was proposed by \cite{ksj18a} for practical reasons: as a differentiable proxy for ECE, to allow optimizing for calibration via backpropagation. One of the motivations behind smooth calibration, discussed in both \cite{kakadeF08, FosterH18} was to address the discontinuity of the standard binning measures of calibration and $\ECE$. But it main application was as a weakening of perfect calibration, to study the power of deterministic forecasts in the online setting and derandomize the classical result of \cite{FosterV98} on calibrated forecasters.

Our work establishes that these measures are not just good ways to measure calibration, they are  more fundamental than previously known. Smooth calibration is within constant factors of the lower distance to calibration, and yields the best possible quadratic approximation to the true distance to calibration.

\section{A Framework for Calibration Measures}
\label{sec:framework}

In this section, we will present our framework for calibration measures. We start by characterizing the set of perfectly calibrated predictors. We then propose our ground truth notion of distance from calibration,  in analogy to the distance from a code in property testing. Building on this, we formulate robust completeness and soundness guarantees that we want calibration measures to satisfy. Finally, we show information theoretic reasons why any calibration measure can only hope to give a quadratic approximation to the ground truth distance. We would like to emphasize the new definitions and the rationale behind them, hence most proofs are deferred to Appendix \ref{app:distance} to streamline the flow. We start with some notation. 

\paragraph{Notation.} 
Let $\X$ be a discrete domain.\footnote{We will assume that the domain $\X$ is discrete but possibly very large. As a consequence, $\Ima(f)$ is discrete, and events such as $f(x) =v$ for $v \in \Ima(f)$ are well defined. We can think of the finiteness assumption reflecting the fact that inputs to any model have finite precision. We do this to avoid measure-theoretic intricacies, but assuming $f: \X \to \zo$ is measurable should suffice when $\X$ is infinite.  
} Let $\mD$ be a distribution on $\X \times \bits$; we denote a sample from $\mD$ by $(x,y) \sim \mD$ where $x \in \X, y \in \bits$. A predictor is a function $f:\X \rgta \zo$, where $f(x)$ is an estimate of $\Pr[ y= 1|x]$.  We define the Bayes optimal predictor $f^*$ as $f^*(x) = \E[y|x]$. Note that $\mD$ is completely specified by the marginal distribution $\mD_{\X}$ on $\X$, and the conditional expectations $f^*$. We let $\mF_\mX$ denote the set of all predictors $f:\X \rgta \zo$.  We define the $\ell_1$ distance in $\FX$ as 
\[ \ell_1(f,g) = \E_{\mD}|f(x) - g(x)|. \]
For a distribution $\mD$ and predictor $f$, we use $\mD_f$ to denote the distribution over $\Ima(f) \times \bits$ of $(f(x), y)$ where $(x,y) \sim \mD$. Two predictors $f$ and $g$ might be far apart in $\ell_1$, yet $D_f$ and $\mD_g$ can be identical. \footnote{Consider the uniform distribution on $\X = \bits$ and let $f^*(x) = 1/2$ so labels are drawn uniformly. Consider the predictors $f(x) = x$ and $g(x) = 1 -x$. While $\ell_1(f,g) = 1$, the distributions $\mD_f$ and $\mD_g$ are identical, since $f/g$ is uniform on $\bits$, and the labels are uniform conditioned on $f/g$.} 

A calibration measure $\mu$ is a function that maps a distribution $\mD$ and a predictor $f:\X \to \zo$ to a value $\mu_\mD(f) \in [0,1]$. A crucial question is the level of access to the underlying distribution that a procedure for computing $\mu$ has. We refer to the setting where an algorithm has access to $(x, f(x), y)$ for $(x,y)\sim \mD$ as the sample-access model or SA model for short following \cite{OI}. Calibration measures are typically defined in the more restricted {\em prediction-only access model} or PA model for short, where we only get access to the joint distribution $\mD_f$ of prediction-label pairs $(f, y)$.
Such calibration measures $\mu$ can be defined as follows: we first define $\mu(\Gamma)\in [0,1]$ for every distribution $\Gamma$ over $[0,1]\times\{0,1\}$, and then for a distribution $\mD$ and a predictor $f$, we define $\mu_{\mD}(f)$ to be $\mu(\mD_f)$.

We say a distribution $\Gamma$ over $[0,1]\times\{0,1\}$ is \emph{perfectly calibrated} if $\E_{(v,y)\sim\Gamma}[y|v] = v$. For a distribution $\mD$ over $\mX\times\{0,1\}$, we say a predictor $f:\mX\to[0,1]$ is \emph{perfectly calibrated} w.r.t.\ $\mD$ if $\mD_f$ is perfectly calibrated. We use $\calD$ to denote the set of predictors $f$ that is perfectly calibrated w.r.t.\ $\mD$. 

There is an injection from $\calD$ to the set of partitions of the domain $\X$. A consequence is that when $\X$ is finite, so is $\calD$. In particular, $\calD$ is not a convex subset of $\FX$. We describe the injection below for completeness, although it is not crucial for our results. 
For a partition $\mS = \{S_i\}_{i=1}^m$, we define $g_\mS(x) = E[y|x \in S_i]$ for all $x \in S_i$. It is clear that $g_\mS \in \calD$.  For a predictor $f$, let $\level(f)$ be the partition of the domain $\X$ given by its level sets. By the definition of calibration, $f \in \calD$ iff it is equal to $g_{\level(f)}$, which establishes the injection.

\subsection{Desiderata for Calibration Measures}

A calibration measure $\mu$ is a function that for a given distribution $\mD$, maps predictors $f$ in $\FX$ to values in $[0,1]$. We denote this value as $\mu_{\mD}(f)$. At the bare minimum, we want $\mu_\mD$ to satisfy completeness and soundness, meaning that for all $\mD, f$,
\begin{align} 
    \mu_\mD(f) = 0 \ & \text{if} \ f \in \calD \ \ \tag{Completeness}\\
    \mu_\mD(f) > 0 \ & \text{if} \ f \not\in \calD \ \ \tag{Soundness}
\end{align}
Ideally, we want these guarantees to be robust: $\mu(f)$ is small if $f$ is close to calibrated, and large is $f$ is far from calibrated. Formalizing this  requires us to specify how we wish to measure the distance from calibration. A family of metrics $m$ is a collection of metrics $m_{\mD}$ on $\FX$ for every distribution $\mD$ on $\X$. For instance, the $\ell_p$ distance on $\FX$ under distribution $\mD$ for $p \geq 1$ is given by
\[ \ell_{p, \mD}(f,g) = \E_\mD[|f(x) - g(x)|^p]^{1/p} \]
We note that $m_{\mD}$ only ought to depend on the marginal $\mD_\X$ on $\X$. When the distribution $\mD$ is clear, we will sometimes suppress the dependence on the distribution and refer to $m$ as a metric rather than a family. Indeed, it is common to refer to the above distance as $\ell_p$ distance, ignoring the dependence on $\mD$.

\begin{definition}[True distance to calibration]
Given a metric family $m$ on $\FX$, we define the true $m$-distance to calibration under $\mD$ as
\[ \dCE^m_{\mD}(f) = \min_{g \in \calD} m_\mD(f, g). \]
\end{definition}

With this definition in place, we define consistent calibration measures with respect to $m$.
\begin{definition}[Consistent calibration measures]
For $c, s \geq 0$, we say that $\mu$ satisfies $c$-robust completeness w.r.t. $m$ if there exist a constant $a \geq 0$ such that for every distribution $\mD$ on $\X \times \bits$, and predictor $f \in \FX$
\begin{align} 
\label{eq:rc}
    \mu_\mD(f) \leq a(\dCE^m_\mD(f))^{c} \tag{Robust completeness}
\end{align}
and $s$-robust soundness w.r.t. $m$ if there exist $b \geq 0$ such that for every distribution $\mD$ on $\X \times \bits$, and predictor $f \in \FX$
\begin{align} 
\label{eq:rs}
    \mu_\mD(f) \geq b(\dCE_\mD^m(f))^{s}. \tag{Robust soundness}
\end{align}
We say that $\mu$ is an $(c,s)$-consistent calibration measure w.r.t $m$ if both these conditions hold, and we define its approximation degree to be $s/c$. \footnote{For metrics that can take on arbitrarily small values (such as the $\ell_p$ metrics), it follows that $s > c$.} We say $\mu$ it is an consistent calibration measure w.r.t. $m$ if there exists $c, s \geq 0$ for which $(c,s)$-consistency for $m$ holds. 
\end{definition}

To see that these names indeed make sense, observe that if $f$ is $\eps$-close to being perfectly calibrated, then robust completeness ensures that $\mu_\mD(f)$ is $O(\eps^c)$ and hence goes to $0$ with $\epsilon$. Robust soundness ensures that if $\mu_\mD(f) = \eps \to 0$, then $\dCE^m_\mD(f) = O(\eps^{1/s}) \to 0$. Conversely, when the true $m$-distance to calibration for $f$ is $\eta \gg 0$, robust soundness ensures that $\mu_\mD(f) = \Omega(\eta^{s})$ is also bounded away from $0$.  

Given a sequence of predictors $\{f_n\}$, we say that the sequence converges to $f \in\FX$, denote $F_n \to f$ if 
\[ \lim_{n \to \infty} m_{\mD}(f_n,f) = 0.\]
Robust soundness ensures that if $f_n \to g \in \calD$, then $\mu(f_n) \to 0$.  
$\dCE^m_\mD$ satisfies a stronger continuity property, namely that it is $1$-Lipshcitz with respect to $m_\mD$:
\[ |\dCE^m_\mD(f) - \dCE^m_\mD(f')| \leq m_\mD(f, f'). \]
This property is easy to verify from the definition. 
It implies that for any $f \in \FX$ not necessarily calibrated,  if $f_n \to f$, $\mu(f_n) \to \mu(f)$. 
Not every $\ell_1$-consistent calibration measure have this stronger property of convergence everywhere, although some do. 

Indeed, the following lemma implies that among all calibration measures that satisfy completeness and are $1$-Lipschitz with respect to $m_\mD$, $\dCE^m_\mD$ is the largest. Thus any consistent calibration measure that can grow as $\omega(\dCE^m)$ cannot be Lipschitz.

\begin{lemma}
\label{lem:best}
    Any calibration measure $\mu_\mD$ which satisfies completeness and is $L$-Lipschitz w.r.t $m_\mD$ must satisfy $\mu_\mD(f) \leq L\ \dCE_\mD(f)$ for all $f \in \FX$.
\end{lemma}

Metric families that are particularly important to us are the $\ell_p$ metrics. Since all $\ell_p$ measures are polynomially related, the set of $\ell_p$ consistent calibration metrics is independent of $p$ for bounded $p$. 
\begin{lemma}
\label{lem:all-lp}
    The set of $\ell_p$-consistent calibration measures is identical for all $p \in [1, \infty)$.
\end{lemma}

Given this result, we will focus on the $\ell_1$ metric and define the true distance to calibration by 
\begin{align} 
    \dCE_\cD(f):= \dCE_{\cD}^{\ell_1}(f) = \min_{g \in \calD} \ell_{1, \mD}(f, g). \tag{True distance from calibration}
\end{align}
It has good continuity properties, and  the resulting set of consistent calibration measures does not depend on the choice of the $\ell_p$ metric. Henceforth when we refer to $(c,s)$-consistent calibration metrics without making $m$ explicit, it is assumed that we mean the $\ell_1$ distance. We note that there might be settings (not considered in this work) where other metrics on $\FX$ are suitable. 

\subsection{Approximation Limits in the PA Model}
\label{sec:limits}
Given the desirable properties of $\dCE$, one might wonder: why not use $\dCE$ as a calibration measure in itself?
The main barrier to this is that $\dCE$ cannot be computed (or even defined) in the prediction access model. Indeed, if it were, there would be no need to look for alternative notions of approximate calibration. 

\begin{lemma}
\label{lem:pa-gap}
    Let $\alpha \in (0,1/2]$. There exists a domain $\X$, a predictor $f \in \FX$, and distributions $\mD^1, \mD^2$ on $\X \times \bits$ such that 
    \begin{itemize}
        \item The distributions $\mD^1_f$ and $\mD^2_f$ are identical.
        \item $\dCE_{\mD^1}(f) \leq 2\alpha^2$, while $\dCE_{\mD^2}(f) \geq \alpha$.
    \end{itemize} 
\end{lemma}
\begin{proof}
    We consider the space $\X = \{00,01,10,11\}$. 
    $\mD^1$ and $\mD^2$ will share the same marginal distribution $\mD_\X$ on $\X$ given by 
    \begin{align*} 
        \mD_\X(x) = \begin{cases} \alpha  \ \text{if} \  x \in \{00, 11\}\\
        \fr{2} - \alpha \ \ \text{if} \  x \in \{01, 10\}
        \end{cases}
    \end{align*}
    The predictor $f:\X \to \zo$ is given by 
    \begin{align*} 
        f(x) = \begin{cases} \fr{2} - \alpha  \ \text{if} \  x_1 =1\\
        \fr{2} + \alpha \ \text{if} \ x_1 =0
    \end{cases}
    \end{align*}
    
    For the distribution $\mD^1$, the conditional probabilities are given by $f^*_1(x) = (x_1 + x_2)/2$. 
    It is easy to check that the predictor $g_1$ defined as 
    \begin{align*} 
        g_1(x) = \begin{cases} \fr{2} - \alpha  \ \text{if} \  x_2 = 0\\
        \fr{2} + \alpha \ \ \text{if} \  x_2 =1
        \end{cases}
    \end{align*}
    lies in $\mathsf{cal}(\mD^1)$ and 
    \begin{align*}
    |f(x) - g_1(x)| = \begin{cases} 0  \ \text{for} \ x \in \{01,10\},\\
        2\alpha \ \text{for} \ x \in \{00,11\},
        \end{cases}
    \end{align*}    
    It follows that  $\ell_1(f, g_1) = 2\alpha^2$, hence $\dCE_{\mD^1}(f) \leq 2\alpha^2$.

    The conditional probabilities for the distribution $\mD^2$ are given by
    \begin{align*} 
        f^*_2(x) = \begin{cases} \fr{2} + \alpha  \ \text{if} \  x_1 =1\\
        \fr{2} - \alpha \ \text{if} \ x_1 =0
    \end{cases}
    \end{align*}
    One can verify that the closest calibrated predictor is the constant $1/2$ predictor, so that $\dCE_{\mD^2}(f) \geq \alpha.$
    
    To verify that $\mD^1_f = \mD^2_f$, we observe that under either distribution
    \begin{align*} 
        \E\lt[y|f = \fr{2} - \alpha\rt]  = \fr{2} +\alpha,\ \ \E\lt[y|f = \fr{2} + \alpha\rt]  = \fr{2} - \alpha. 
    \end{align*}
\end{proof}

This leads us to the quest for approximations that can be computed (efficiently) in the Prediction-access model. It implies that one can at best hope to get a degree $2$ approximation to $\dCE$. 
\begin{corollary}
\label{cor:quadratic}
Let $\mu(f)$ be a $(\ell_1, c, s)$-consistent calibration measure computable in the prediction access model. Then $s \geq 2c$. 
\end{corollary}
\begin{proof}
    Using $c$-robust completeness  for $\mD^1$ we have
    \[ \mu_{\mD^1}(f) \leq a (2\alpha^2)^c. \]
    Using $s$-robust soundness  for $\mD_2$ we have
    \[ \mu_{\mD^2}(f) \geq b (\alpha^s). \]
    Since $\mD^1_f = \mD^2_f$ and $\mu$ is computable in the PA model, $\mu_{\mD^1}(f) = \mu_{\mD^2}(f)$. Hence $b\alpha^s \leq a (2\alpha^2)^c$, which gives
    \[ \lt(\frac{b}{a}\alpha\rt)^{s/c} \leq 2\alpha^2. \]
    If $s/c <2$, this gives a contradiction for $\alpha \to 0$. 
\end{proof}

Given this setup, we can now state our desiderata for an ideal calibration measure $\mu$.
\begin{enumerate}
    \item {\bf Access:} $\mu_\mD(f) = \mu(\mD_f)$ is well defined in the Prediction-only access model.
    \item {\bf Consistency:} It is $(c, s)$-consistent. Ideally, it has degree $s/c =2$.
    \item {\bf Efficiency:} It can be computed within accuracy $\eps$ in time $\poly(1/\eps)$ using $\poly(1/\eps)$ random samples from $\mD_f$.
\end{enumerate}

\subsection{On \texorpdfstring{$\ECE$}{ECE} and Other Measures} 
\label{sec:ece}

Recall that for a predictor $f$, we define its expected calibration error $\ECE(f)$ as
    \[ \ECE(f) = \E[|\E[y|f]  - f|] .\]
Clearly, $\ECE$ is well defined in the PA model. We analyze $\ECE$ in our framework and show that it satisfies $1$-robust soundness, but not robust completeness. For the former, we present an alternate view of $\ECE$ in terms of $\ell_1$ distance. Recall that $\level(f)$ is the partition of $\X$ into the level sets of $f$, and that for a partition $\mS=\{S_i\}$, the predictor $g_\mS$ maps each $x \in S_i$ to $\E[y|S_i]$. 

\begin{lemma}
    Let $\mS = \level(f)$. We have $\ECE(f) = \ell_1(f, g_\mS) \geq \dCE_\mD(f)$.
\end{lemma}
\begin{proof}
    We claim that
    \begin{align*}
        \E[|f - \E[y|f]|] = \E[|f(x) - g_\mS(x)|]
    \end{align*}
    which uses the observation that for all $x \in f^{-1}(v)$, $f(x) =v$, $g(x) = \E[y|v]$. The LHS is $\ECE(f)$, while the RHS is $\ell_1(f, g_{\mS})$. Clearly this is larger than $\dCE_\mD(f)$ which minimizes the $\ell_1$ distance over all $g \in \calD$. 
\end{proof}

 The main drawbacks of $\ECE$ are that it does not satisfy robust completeness, and is discontinuous at $0$.

\begin{lemma}
\label{lem:ece}
    $\ECE_\mD(f)$ does not satisfy $c$-robust completeness for any $c > 0$. It can be discontinuous at $0$.   
\end{lemma}
\begin{proof}
Let $X = \{0,1\}$. $\mD$ is specified by the uniform distribution on $X$ with $f^*(0) = 0, f^*(1) = 1$. Consider the predictor $f_\eps(0) = 1/2 - \varepsilon, f_\eps(1) = 1/2 + \varepsilon$. Let $\favg = 1/2 \in \calD$.  Clearly $\dCE_\mD(f_\eps) \leq \eps$ whereas $\ECE(f_\eps) = 1/2 -\eps$. 

This also shows that $\ECE$ is discontinuous at $0$ since $\ECE(f_\eps) \to 1/2$ as $\eps \to 0$, whereas $f_0 = \favg$ has $\ECE(f_0) = 0$.
\end{proof}

Table~\ref{tab:priorworks} summarize how other calibration measures that have been studied in the literature fare under our desiderata. Further discussion of these measures can be found in \cref{app:related}.

\section{Distance Based Measures in the PA Model}
\label{sec:pa}

We start by defining upper and lower bounds to the true distance to calibration in the PA model. Our main result in this subsection is Theorem \ref{thm:dce-upper-lower} showing that these are the {\em best possible bounds} one can have on $\dCE$ in the PA model.
To define and analyze these distances, we need some auxiliary notions.

\begin{definition}
Let $\Gamma$ be a distribution over $[0,1]\times\{0,1\}$.
Define the set $\ext(\Gamma)$ to consist of all joint distributions $\Pi$ of triples $(u, v, y) \in \zo \times \zo \times \bits$, such that 
 \begin{itemize}
     \item the marginal distribution of $(v,y)$ is $\Gamma$;
     \item the marginal distribution $(u, y)$ is perfectly calibrated: $\E_\Pi[y|u] = u$.
 \end{itemize}
 We define $\lift(\Gamma)$ to be all pairs $(\mD, f)$ where 
 \begin{itemize}
     \item $\mD$ is a distribution over $\X \times \bits$ for some domain $\X$.
     \item $f: \X \to \zo$ is predictor so that $\mD_{f} = \Gamma$.
 \end{itemize}
\end{definition}

We first define the upper distance to calibration.
\begin{definition}[Upper distance to calibration]
For a distribution $\Gamma$ over $[0,1]\times\{0,1\}$,let $K(\Gamma)$ denote the set of transformations $\kappa:[0,1]\to [0,1]$ such that the distribution of $(\kappa(v),y)$ for $(v,y)\sim \Gamma$ is perfectly calibrated. We define the {\em upper distance from calibration} $\udCE(\Gamma)$ as
    \[ \udCE(\Gamma) = \inf_{\kappa \in K(\Gamma)} \E_{(v,y)\sim\Gamma}[|v-\kappa(v)|], \]
For a distribution $\mD$ over $\mX\times\{0,1\}$ and a predictor $f:\mX \to [0,1]$, we define the {\em upper distance from calibration} $\udCE_{\mD}(f)$ to be $\udCE(\mD_f)$, or equivalently,
\[
\udCE_{\mD}(f) := \inf_{\substack{\kappa:[0,1]\to [0,1] \\\kappa\circ f\in \calD}}\E_{(x,y)\sim \mD}[|f(x) - \kappa(f(x))|].
\]
\end{definition}
We call this the upper distance since we only compare $f$ with a calibrated predictor $\kappa\circ f$ that can be obtained by applying a postprocessing $\kappa$ to $f$. It follows immediately that $\udCE_\mD(f) \geq \dCE_\mD(f)$.

 We next define the lower distance to calibration. 
 \begin{definition}[Lower distance to calibration]
 \label{def:ldCE}
 We define the \emph{lower distance to calibration}  denoted $\ldCE(\Gamma)$ as
\begin{equation}
    \ldCE(\Gamma) := \inf_{\Pi \in \ext(\Gamma)} \E_{(u,v,y)\sim\Pi} |u -v| \label{eq:lupadce}.
\end{equation}
For a distribution $\mD$ and a predictor $f$, we define $\ldCE_{\mD}(f):= \ldCE(\mD_f)$.
\end{definition}

The following lemma justifies the terminology of upper and lower distance.

\begin{lemma}
\label{lem:ul-dce-bounds}
    We have $\ldCE_\mD(f) \leq \dCE_{\mD}(f) \leq \udCE_\mD(f)$
\end{lemma}
\begin{proof}
    Every calibrated predictor $g \in \calD$ gives a distribution $\Pi \in \ext(\mD_f)$ where we sample $(x, y) \sim \mD$ and return $(g(x), f(x), y)$. Note that $\E_{(u,v,y)\sim \Pi}[|u -v|] = \ell_1(f,g)$. Minimizing over $g \in \calD$ gives the first inequality. The second follows because in the definition of $\udCE$ we minimize over a subset of $\calD$, namely only those $g = \kappa\circ f$ that can be obtained from $f$ via a postprocessing $\kappa$.
\end{proof}

We now show that these are the best possible bounds one can have on $\dCE$ in the PA model.

\begin{theorem}
\label{thm:dce-upper-lower}
The following identities hold
\begin{align*}
    \ldCE(\Gamma) & = \inf_{(\mD,f) \in \lift(\Gamma)} \dCE_{\mD}(f) \\
    \udCE(\Gamma) & = \sup_{(\mD,f) \in \lift(\Gamma)} \dCE_{\mD}(f).
\end{align*}
\end{theorem}
\begin{proof}
By \Cref{lem:ul-dce-bounds}, for every $(\mD,f)\in \lift(\Gamma)$,
\[
\ldCE(\Gamma) = \ldCE_{\mD}(f) \le \dCE_{\cD}(f) \le \udCE_{\mD}(f) = \udCE(\Gamma).
\]
This implies that
\begin{align*}
\ldCE(\Gamma) \le \inf_{(\mD,f) \in \lift(\Gamma)} \dCE_{\mD}(f) \\
\udCE(\Gamma) \ge \sup_{(\mD,f) \in \lift(\Gamma)} \dCE_{\mD}(f).
\end{align*}
It remains to show the reverse inequalities
\begin{align}
\ldCE(\Gamma) & \ge \inf_{(\mD,f) \in \lift(\Gamma)} \dCE_{\mD}(f),  \label{eq:ul-dce-characterize-1}\\
\udCE(\Gamma) & \le \sup_{(\mD,f) \in \lift(\Gamma)} \dCE_{\mD}(f).\label{eq:ul-dce-characterize-2}
\end{align}

To prove Equation \eqref{eq:ul-dce-characterize-1}, we choose $\cX:= [0,1]\times [0,1]$ and define predictors $f,g:\cX\to[0,1]$ such that for any $x= (u,v)\in \cX$, it holds that $f(x) = v$ and $g(x) = u$.
For a fixed $\Pi \in \ext(\Gamma)$, we define a distribution $\mD$ over $\cX\times\{0,1\}$ such that $(x,y)\sim \mD$ is drawn by first drawing $(u,v,y)\sim \Pi$ and then setting $x:= (u,v)$. By the definition of $\Pi \in \ext(\Gamma)$, it holds that $\cD_{f} = \Gamma$ and that $g\in \calD$. Therefore, 
\[
\dCE_{\mD}(f) \le \ell_1(f,g)= \E_{(u,v,y)\sim \Pi}|u - v|.
\]
The fact that $\cD_f = \Gamma$ implies that $(\cD,f)\in \lift(\Gamma)$, and thus
\[
\inf_{(\mD,f) \in \lift(\Gamma)} \dCE_{\mD}(f) \le \E_{(u,v,y)\sim \Pi}|u - v|. 
\]
Choosing $\Pi \in \ext(\Gamma)$ for which the RHS is minimized and equals $\ldCE_\mD(f)$ proves Equation \eqref{eq:ul-dce-characterize-1}.

To prove \eqref{eq:ul-dce-characterize-2}, we choose $\cX:= [0,1]$ and define predictor $f:\cX \to [0,1]$ such that $f(x) = x$ for every $x\in \cX$. We define the distribution $\cD$ over $\cX\times\{0,1\}$ to be the distribution $\Gamma$. It is clear that $\mD_{f} = \Gamma$, so $(\mD,f)\in \lift(\Gamma)$. For any perfectly calibrated predictor $g\in \calD$, the distribution $(g(x),y)$ for $(x,y)\sim\mD$ is perfectly calibrated, which implies that the same distribution $(g(v),y)$ for $(v,y)\sim \Gamma$ is also perfectly calibrated. This implies that 
\[
\udCE(\Gamma) \le \E_{(v,y)\sim \Gamma}|v - g(v)| = \E_{(x,y)\sim \mD}|f(x) - g(x)| =  \ell_1(f,g).
\]
Taking infimum over $g\in \calD$ proves that $\udCE(\Gamma) \le \dCE_{\cD}(f)$, which gives Equation \eqref{eq:ul-dce-characterize-2}.
\end{proof}

The gap between each of these quantities can be at least quadratic, as the distributions $\mD^1$ and $\mD^2$ in Lemma \ref{lem:pa-gap} shows. Under both distributions $\mD^1$ and $\mD^2$, we have $\ldCE(f) \leq 2\alpha^2$, $\udCE(f) = \alpha$. But $\dCE_{\mD^1}(f) = 2\alpha^2$ while $\dCE_{\mD^2}(f) = \alpha$. We will show that this gap is indeed tight in the next section using the notion of Interval calibration.

\section{Interval Calibration \label{sec:interval-calibration}}
 \label{sec:int-ce}

 In this section, we introduce the notion of interval calibration. Our main result is Theorem \ref{thm:intce} which shows that it is quadratically related to the true distance from calibration. Since $\intCE$ is defined in the PA model, this implies in particular, that there might be at most quadratic gap between $\ldCE$ and $\udCE$. We exhibit a gap instance showing that this is tight (Lemma \ref{lem:tight-gap}). As defined, it is unclear if Interval calibration can be efficiently estimated. We propose a surrogate version of interval calibration which gives similar bounds and can be efficiently estimated from samples in Section \ref{sec:surrogate}.

 An interval partition $\mI$ of $\zo$ is a partition of $\zo$ into disjoint intervals $\{I_j\}_{j \in [m]}$. Let $w(I)$ denote the width of interval $I$.

 \begin{definition}[Interval Calibration Error]
 \label{def:intCE}
For a distribution $\Gamma$ over $[0,1]\times\{0,1\}$ and interval partition $\mI$ define
    \[  \mathsf{binnedECE}(\Gamma, \mI) := \sum_{j \in [m]}|\E_{(v,y)\sim\Gamma}[(v - y)\one(v \in I_j)]|. \]
We define the \emph{average interval width}
 \[
 w_\Gamma(\mI) := \sum_{j \in [m]}\E_{(v,y)\sim\Gamma}[\one(v \in I_j) w(I_j)].
 \]
    The {\em interval calibration error} $\intCE(\Gamma)$ is then the minimum of $\mathsf{binnedECE}(\Gamma, \mI) +  w_\Gamma(\mI)$ over all interval partitions $\mI$:
    \[
    \intCE(\Gamma) := \min_{\mI: ~\textrm{Interval partition}} \left( \mathsf{binnedECE}(\Gamma, \mI) + w_\Gamma(\mI)\right).
    \]
For a distribution $\mD$ over $\mX\times\{0,1\}$ and a predictor $f:\mX\to[0,1]$, we define $\intCE_\mD(f,\mI):= \intCE(\mD_f,\mI)$ and $\intCE_\mD(f):= \intCE(\mD_f)$.
 \end{definition}
 
 Our main theorem about interval calibration is the following.
 \begin{theorem}
\label{thm:intce}
We have $\udCE(\Gamma) \leq \intCE(\Gamma) \leq 4\sqrt{\ldCE(\Gamma)}$.
\end{theorem}

Combining this with Lemma \ref{lem:ul-dce-bounds}, we conclude that $\intCE$ is indeed a quadratic approximation to the true distance from calibration, which is the best achievable in the PA model by Corollary \ref{cor:quadratic}.
\begin{corollary}
\label{cor:intce1}
$\intCE$ is a $(1/2, 1)$-consistent calibration measure. We have 
\begin{align*}
    \dCE_\mD(f) \leq \intCE_\mD(f) \leq 4\sqrt{\dCE_\mD(f)}.
\end{align*}
\end{corollary}

Another corollary is the following bounds for distance measures, which shows that the gaps presented in Lemma \ref{lem:pa-gap} are the largest possible.
\begin{corollary}
\label{cor:intce2}
 We have 
\begin{align*}
    \ldCE_\mD(f) & \leq \dCE_\mD(f)  \leq 4\sqrt{\ldCE_\mD(f)},\\
    \fr{16}\udCE_\mD(f)^2 & \leq  \dCE_\mD(f) \leq \udCE_\mD(f).
\end{align*}
\end{corollary}

We first show the lower bound on $\intCE$, which is easier to show. 
\begin{proof}[Proof of Theorem \ref{thm:intce} (Part 1)]
 It suffices to prove the following statement: for any interval partition $\mI= \{I_j\}_{j \in [m]}$, $\udCE(\Gamma) \leq \mathsf{binnedECE}(\Gamma, \mI)$.
For $v \in I_j$, define $\kappa(v) = \E[y|v \in I_j]$. The distribution of $(\kappa(v),y)$ is perfectly calibrated. Hence
\begin{align*}
    \udCE(\Gamma) &\leq \E[|v -\kappa(v)|]\\
    & \leq \sum_{j \in [m]}\E[\one(v \in I_j)|v - \E[y|v \in I_j]|]\\
    & \leq \sum_{j \in [m]}\E[\one(v \in I_j)(|v - \E[v|v \in I_j]| + |[\E[y|v \in I_j] - \E[v|v \in I_j]|)]\\
    & \leq \sum_{j \in [m]}\E[\one(v \in I_j) w(I_j)] + \sum_{j \in [m]}\E[\one(v \in I_j)|[\E[y|v \in I_j] - \E[v|v \in I_j]|]
\end{align*}
where we use the fact that conditioned on $v \in I_j$, $v$ differs from its expectation by at most $w(I_j)$. The first sum is exactly the average width $w_\Gamma(\mI)$. For the second, by the definition of conditional expectations
\begin{align*}
    \E[\one(v \in I_j)\E[y|v \in I_j]] &= \E[\one(v \in I_j)y]\\
    \E[\one(v \in I_j)\E[v|v \in I_j]] &= \E[\one(v \in I_j)v]
\end{align*}
Hence
\begin{align*}
    \E[\one(v \in I_j)|[\E[y|v \in I_j] - \E[v|v \in I_j]|] &= |\E[\one(v \in I_j)\E[y|v \in I_j] - \one(v \in I_j)\E[v|v \in I_j]]|\\
    &= |\E[\one(v \in I_j)(y -v)]|
\end{align*}
We conclude that
\begin{align*}
    \udCE(\Gamma) \leq w_\Gamma(\mI) + \sum_{j \in [m]} |\E[\one(v \in I_j)(y -v)]| = w_\Gamma(\mI) + \mathsf{binnedECE}(\Gamma, \mI).
\end{align*}
The claimed lower bound follows by minimizing over all interval partitions.
\end{proof}

To prove the upper bound on $\intCE$ in \cref{thm:intce}, we consider a variant of the interval calibration error where we focus on intervals with a fixed width $\varepsilon$ and take expectation after shifting the intervals randomly:

\begin{definition}[Random Interval Calibration Error]
For a distribution $\Gamma$ over $[0,1]\times\{0,1\}$ and an interval width parameter $\varepsilon > 0$, we define
\[
\RintCE(\Gamma,\varepsilon) = \E_{r}\Big[\sum_{j\in \Z}|\E_{(v,y)\sim\Gamma}[(y - v)\one(v\in I_{r,j}^\varepsilon)]|\Big],
\]
where the outer expectation is over $r$ drawn uniformly from $[0,\varepsilon)$ and $I_{r,j}^\varepsilon$ is the interval $[r + j\varepsilon, r + (j+1)\varepsilon)$. 
\end{definition}
Note that although the summation is over $j\in \Z$, there are only finitely many $j$'s that can contribute to the sum (which are the $j$'s that satisfy $I_{r,j}^\varepsilon\cap [0,1]\ne \emptyset$).
The following claim follows by averaging argument from the definitions of $\intCE$ and $\RintCE$:
\begin{claim}
\label{claim:intCE-RintCE}
$\intCE(\Gamma) \le \RintCE(\Gamma,\varepsilon) + \varepsilon$.
\end{claim}

The key to proving the upper bound on $\intCE(\Gamma)$ in \Cref{thm:intce} is the following lemma.
\begin{lemma}
\label{lm:RintCE-ldCE}
$\RintCE(\Gamma,\varepsilon) \le (1 + \frac 2\varepsilon)\,\ldCE(\Gamma)$.
\end{lemma}
We first complete the proof of \Cref{thm:intce} using \Cref{lm:RintCE-ldCE}. 
\begin{proof}[Proof of Theorem \ref{thm:intce} (Part 2)]
We prove the upper bound on $\intCE(\Gamma)$ in \Cref{thm:intce}. Combining \Cref{claim:intCE-RintCE} and \Cref{lm:RintCE-ldCE}, for every $\varepsilon > 0$, we have
\[
\intCE(\Gamma) \le \left(1 + \frac 2\varepsilon\right)\, \ldCE(\Gamma) + \varepsilon.
\]
Choosing $\varepsilon = \sqrt{2\, \ldCE(\Gamma)}$ to minimize the right-hand side completes the proof.
\end{proof}
\begin{proof}[Proof of \Cref{lm:RintCE-ldCE}]
Let $\Pi$ be a distribution over $[0,1]\times [0,1]\times\{0,1\}$ in $\ext(\Gamma)$.
Our goal is to prove that 
\[ \RintCE(\Gamma,\varepsilon) \le (1 + 2/\varepsilon)\E_{(u,v,y)\sim \Pi}|u - v|.\] 
By the definition of $\Pi\in \ext(\Gamma)$, for $(u,v,y)\sim \Pi$, the distribution of $(v,y)$ is $\Gamma$, and the distribution of $(u,y)$ is perfectly calibrated. The latter implies that
\begin{equation}
\label{eq:RintCE-ldCE-1}
\E[(u - y)w(u)] = 0 \quad \text{for any function }w:[0,1]\to [-1,1].
\end{equation}

For every $r\in [0,\varepsilon)$, we have the following inequality, where all expectations and probabilities are over $(u,v,y)\sim \Pi$:
\begin{equation}
\label{eq:RintCE-ldCE-2}
\sum_{j\in \Z}|\E[(v - y)\one(v\in I_{r,j}^\varepsilon)]| \le \sum_{j\in \Z}|\E[(v - u)\one(v\in I_{r,j}^\varepsilon)]| + \sum_{j\in \Z}|\E[(u - y)\one(v\in I_{r,j}^\varepsilon)]|.
\end{equation}
The following inequality bounds the first term in the RHS of \eqref{eq:RintCE-ldCE-2}:
\begin{equation}
\label{eq:RintCE-ldCE-3}
\sum_{j\in \Z}|\E[(v - u)\one(v\in I_{r,j}^\varepsilon)]| \le \sum_{j\in \Z}\E[|v - u|\one(v\in I_{r,j}^\varepsilon)] = \E|v - u|.
\end{equation}
The following inequality bounds the second term in the RHS of \eqref{eq:RintCE-ldCE-2}:
\begin{align}
& \sum_{j\in \Z}|\E[(u - y)\one(v\in I_{r,j}^\varepsilon)]| \notag \\
\le {} & \sum_{j\in \Z}|\E[(u - y)\one(u\in I_{r,j}^\varepsilon)]| + \sum_{j\in \Z}|\E[(u - y)(\one(v\in I_{r,j}^\varepsilon) - \one(u\in I_{r,j}^\varepsilon))]|\notag \\
= {} & \sum_{j\in \Z}|\E[(u - y)(\one(v\in I_{r,j}^\varepsilon) - \one(u\in I_{r,j}^\varepsilon))]| \tag{by \eqref{eq:RintCE-ldCE-1}}\\
\le {} & \sum_{j\in \Z}\E|\one(v\in I_{r,j}^\varepsilon) - \one(u\in I_{r,j}^\varepsilon)|\tag{by $|u - y|\le 1$}\\
= {} & 2\Pr[j_r(u)\ne j_r(v)], \label{eq:RintCE-ldCE-4}
\end{align}
where $j_r(u)\in \Z$ is the $j\in \Z$ such that $u\in I_{r,j}^\varepsilon$ and $j_r(v)$ is defined similarly.
Plugging \eqref{eq:RintCE-ldCE-3} and \eqref{eq:RintCE-ldCE-4} into \eqref{eq:RintCE-ldCE-2} and
taking expectation over $r$ drawn uniformly from $[0,\varepsilon)$, we have
\begin{align*}
\RintCE(\Gamma,\varepsilon) & \le \E|v - u| + 2\E_r\Pr_{(u,v,y)\sim\Pi}[j_r(u)\ne j_r(v)]\\
& = \E|v - u| + 2\E_{(u,v,y)\sim\Pi}\Pr_{r}[j_r(u)\ne j_r(v)].
\end{align*}
It is easy to check that for fixed $u,v\in [0,1]$, we have 
\[ \Pr_r[j_r(u) \ne j_r(v)] \le \frac 1\varepsilon |u - v|.\] 
Plugging this into the inequality above, we get
\[
\RintCE(\Gamma,\varepsilon) \le \E|v - u| + \frac 2\varepsilon\E_{(u,v,y)\sim\Pi}|u - v|.
\qedhere
\]
\end{proof}

\subsection{Quadratic Gap between Interval Calibration and Upper Calibration Distance}
For a distribution $\mD$ and a predictor $f$,
our results in previous subsections imply the following chain of inequalities (omitting $\mD$ in the subscript for brevity):
\begin{equation}
\label{eq:int-dce-chain}
\ldCE(f) \le \udCE(f) \le \intCE(f) \le 4\sqrt{\ldCE(f)}.
\end{equation}
These inequalities completely characterize the relationship between $\ldCE(f)$ and $\udCE(f)$ and also the relationship between $\ldCE(f)$ and $\intCE(f)$ for the following reason. By \Cref{lem:pa-gap}, we know that $\udCE(f)$ can be as large as $\Omega(\sqrt{\ldCE(f)})$, which implies that $\intCE(f)$ can be as large as $\Omega(\sqrt{\ldCE(f)})$. Also, it is easy to show that $\intCE(f)$ can be as small as $O(\ldCE(f))$ by choosing $f$ to be a constant function, which implies that $\udCE(f)$ can be as small as $O(\ldCE(f))$.

The remaining question is whether \eqref{eq:int-dce-chain} completely characterizes the relationship between $\udCE(f)$ and $\intCE(f)$. We show that the answer is yes by the following lemma (\Cref{lem:tight-gap}) which gives examples where $\intCE(f) = \Omega((\udCE(f))^{1/2})$.
We also show that $\intCE(f)$ can be discontinuous as a function of $f$ in \Cref{lm:int-discontinuous}.

\begin{lemma}
\label{lem:tight-gap}
For any $\alpha\in (0,1/4)$, there exist distribution $\cD$ and predictor $f$ such that
\[
\udCE_\cD(f) \le 5\alpha^2 \quad \text{and} \quad \intCE_\cD(f) \ge \alpha.
\]
\end{lemma}
\begin{proof}
We use an example similar to the one in the proof of \Cref{lem:pa-gap}. Specifically, we choose $\mX = \{00,01,10,11\}$, and choose the distribution $\mD$ over $\cX\times\{0,1\}$ such that the marginal distribution $\mD_\mX$ over $\mX$ is given by the following probability mass function:
\[
\mD_\mX(x) = \begin{cases}
\alpha, & \text{if }x\in \{00,11\};\\
1/2 - \alpha, & \text{if }x\in \{01,10\}.
\end{cases}
\]
We choose the conditional distribution of $y$ given $x$ such that $\E_{(x,y)\sim\mD}[y|x] = (x_1 + x_2)/2$, where $x_1$ and $x_2$ are the two coordinates of $x$. Note that this distribution $\mD$ is the distribution $\mD^1$ in the proof of \Cref{lem:pa-gap}.

Defining $\beta = \alpha/2$, we choose the predictor $f$ slightly differently from the proof of \Cref{lem:pa-gap} as follows:
\begin{align*}
f(00) & = 1/2 + \alpha + \beta\\
f(01) & = 1/2 + \alpha\\
f(10) & = 1/2 - \alpha\\
f(11) & = 1/2 - \alpha - \beta.
\end{align*}
Note that the function $f$ in the proof of \Cref{lem:pa-gap} corresponds to choosing $\beta = 0$ instead.
We define the perfectly calibrated predictor $g\in\calD$ as in \Cref{lem:pa-gap}. That is,
\[
g(x) = \begin{cases}
1/2 - \alpha, & \text{if }x_2 = 0;\\
1/2 + \alpha, & \text{if }x_2 = 1.
\end{cases}
\]

It is easy to check that $\ell_1(f,g) = 2\alpha(2\alpha + \beta) \le 5\alpha^2$, which implies that $\udCE_\cD(f)\le  5\alpha^2$.

Now we show a lower bound for $\intCE_\cD(f)$. Consider an interval partition $\mI$ of $[0,1]$. If $1/2 - \alpha$ and $1/2 + \alpha$ are in different intervals, we have
\[
\intCE_\cD(f,\mI) \ge |\E[(y - f(x))\one(f(x)\le 1/2)]| + |\E[(y - f(x))\one(f(x) > 1/2)]| \ge 2\alpha.
\]
If $1/2 - \alpha$ and $1/2 + \alpha$ are in the same interval, we have
\[
w_{\mD_f}(\mI) \ge \Pr[x\in \{10,01\}]\cdot 2\alpha \ge \alpha.
\]
This implies that $\intCE_\cD(f) \ge \alpha$.
\end{proof}

\begin{lemma}
\label{lm:int-discontinuous}
There exist a distribution $\mD$ over $\cX\times\{0,1\}$ and a sequence of predictors $f_n:\cX\to [0,1]$ converging uniformly to a predictor $f:\cX\to[0,1]$ as $n \to \infty$ such that $\lim_{n\to \infty}\intCE_\mD(f_n)\ne \intCE_\mD(f)$.
\end{lemma}
We defer the proof of \Cref{lm:int-discontinuous} to \Cref{sec:proof-pa}.

\subsection{Efficient Estimation of Surrogate Interval Calibration}
\label{sec:surrogate}

From the definition of interval calibration error ($\intCE(\Gamma)$), it is unclear how one can efficiently estimate the notion given examples $(v,y)$ drawn from $\Gamma$ partly because the definition involves a minimization over interval partitions. In this subsection, we give an efficient algorithm for estimating a surrogate version of interval calibration which we define below. We show in \Cref{thm:SintCE} that this surrogate notion enjoys the same quadratic relationship with $\udCE(\Gamma)$ and $\ldCE(\Gamma)$ as in \Cref{thm:intce}.

\begin{definition}[Surrogate Interval Calibration Error] For a distribution $\Gamma$ over $[0,1]\times\{0,1\}$, we define
\[
\SintCE(\Gamma) = \inf_{k\in \Z_{\ge 0}}\left(\RintCE(\Gamma,2^{-k}) + 2^{-k}\right).
\]
\end{definition}
\begin{theorem}
\label{thm:SintCE}
\[
\udCE(\Gamma) \le \intCE(\Gamma) \le \SintCE(\Gamma) \le 6\sqrt{\ldCE(\Gamma)}
\]
\end{theorem}
\begin{proof}
The inequality $\udCE(\Gamma) \le \intCE(\Gamma)$ is already shown in \Cref{thm:intce}. The inequality $\intCE(\Gamma) \le \SintCE(\Gamma)$ is an immediate consequence of \Cref{claim:intCE-RintCE}. It remains to show that $\SintCE(\Gamma) \le 6\sqrt{\ldCE(\Gamma)}$. For every $k\in \Z_{\ge 0}$, by \Cref{lm:RintCE-ldCE},
\[
\SintCE(\Gamma) \le \RintCE(\Gamma,2^{-k}) + 2^{-k} \le (1 + 2 \times 2^k)\, \ldCE(\Gamma) + 2^{-k}.
\]
Choosing $k$ such that $\sqrt{\frac 12\, \ldCE(\Gamma)} \le 2^{-k} \le \sqrt{2\, \ldCE(\Gamma)}$ completes the proof.
\end{proof}
Now we give an efficient estimator for $\SintCE(\Gamma)$ up to error $\varepsilon$. Let $k^*\in\Z_{\ge 0}$ be the integer satisfying $\varepsilon/4 < 2^{-k^*} \le \varepsilon/2$. We collect examples $(v_1,y_1),\ldots,(v_n,y_n)$ drawn i.i.d.\ from $\Gamma$. For $k = 0,\ldots,k^*$, we draw $r_1,\ldots,r_m$ independently and uniformly from $[0,2^{-k})$. We compute
\[
\wh\RintCE(\Gamma,2^{-k}):= \frac{1}{m}\sum_{s = 1}^m\sum_{j\in \Z}\left|\frac 1n\sum_{\ell = 1}^n(y_\ell - f_\ell)\one(f_\ell \in I_{r_s,j}^{2^{-k}})\right|
\]
and then output
\[
\wh\SintCE(\Gamma) := \min_{k = 0,\ldots,k^*}\left(\wh\RintCE(\Gamma,2^{-k}) + 2^{-k}\right).
\]
The following theorem ensures that $\wh\SintCE(\Gamma)$ is an accurate and efficient estimator for $\SintCE(\Gamma)$:
\begin{theorem}
\label{thm:estimate-intce}
There exists an absolute constant $C>0$ with the following property. For $\varepsilon,\delta\in (0,1/4)$, define $k^*\in \Z_{\ge 0}$ be the integer satisfying $\varepsilon/4 < 2^{-k^*} \le \varepsilon/2$. Assume that $n \ge C\varepsilon^{-3} + C\varepsilon^{-2}\log(1/\delta)$, $m \ge C\varepsilon^{-2}\log (k^*/\delta)$ and we compute $\wh\SintCE(\Gamma)$ as above for a distribution $\Gamma$ over $[0,1]\times\{0,1\}$. Then with probability at least $1-\delta$, 
\[
|\wh\SintCE(\Gamma) - \SintCE(\Gamma)| \le \varepsilon.
\]
\end{theorem}
In \Cref{sec:proof-pa}, we prove \Cref{thm:estimate-intce} using standard uniform convergence bounds.
\section{Smooth Calibration and the Lower Distance to Calibration} \label{sec:smooth-ce-duality}

In this section, we define and analyze the notion of smooth calibration.  The main result of this section is that the smooth calibration error $\scerror$ is equivalent, up to a constant factor, to $\ldCE$. We also give algorithms that can compute both these quantities to within an additive $\eps$ in time $\poly(1/\eps)$ on an empirical distribution. 

At a high level, the proof that $\ldCE$ and $\scerror$ are related  proceeds as follows.
\begin{enumerate}
\item 
For a distribution $\Gamma$ over $[0,1]\times \{0,1\}$, our definition of $\ldCE(\Gamma)$ (\Cref{def:ldCE}) is based on couplings $\Pi\in \ext(\Gamma)$ that connect $\Gamma$ to calibrated distributions $\Gamma'$.
For a given distribution $\mD$, the space of predictors $f:\X \to \zo$ which are calibrated is non-convex (for finite $\X$, it is a finite set). But when we move to the space of distributions $\Gamma'$ over $\zo \times \bits$, then the space of perfectly calibrated distributions is convex. This is because for $(v, b) \in \zo \times \bits$ if $\Gamma'(v,b)$ denotes the probability assigned to it, then the calibration constraint states that for every $v$,
\[ \frac{\Gamma'(v,1)}{\Gamma'(v, 0) + \Gamma'(v,1)} = v\]
which is a linear constraint for every $v$. This allows us to view the problem of computing $\ldCE$ as optimization over couplings $\Pi$ connecting $\Gamma$ to some $\Gamma'$ satisfying such linear constraints.
\item We show that by suitably discretizing $[0,1]$, we can write the problem of computing $\ldCE$ as a linear program. The dual of this program (after some manipulation) asks for a $2$-Lipschitz function $w:\zo \to [-1,1]$ which witnesses the lack of calibration of $f$, by showing that $\E[w(v)(y - v)]$ is large. Rescaling gives a $1$-Lipschitz function which proves that $\scerror_\mD(f) \ge \ldCE_\mD(f)/2$. The other direction which corresponds to weak duality is easy to show. 
\end{enumerate}

We now proceed with the formal definitions and proof.
We start by defining a general family of calibration measures called \emph{weighted calibration error} from \cite{GopalanKSZ22}. 
\begin{definition}[Weighted calibration]\cite{GopalanKSZ22}
\label{def:wce}
Let $W$ be a family of functions $w:[0,1]\to \R$. The \emph{weighted calibration error} of a distribution $\Gamma$ over $[0,1]\times \{0,1\}$ is defined as
\[
\wCE^W(\Gamma):= \sup_{w\in W}\left|\E_{(v,y)\sim \Gamma}[(y - v)w(v)]\right|.
\]
Given a distribution $\mD$ over $\mX\times\{0,1\}$ and  predictor $f:\mX\to[0,1]$, we denote the weighted calibration error of $f$ under $\mD$ as
\[
\wCE^W_\mD(f):= \wCE^W(\cD_f) = \sup_{w\in W}\left|\E_{(x,y)\sim \mD}[(y - f(x))w(f(x))]\right|.
\]
\end{definition}
Clearly, any weighted calibration error notion is well defined in the PA model. Moreover, all of those at the very least satisfy completeness: if $\Gamma$ is perfectly calibrated, than for any $w$, we have 
\[ \E_{\Gamma}[(y-v) w(v)] = \E[ \E[ (y-v) w(v) | v ]],\] 
and since $\E[y | v] = v$ for a perfectly calibrated predictor, this latter quantity is zero.

A particularly important calibration measure among those is the \emph{smooth calibration} where $W$ is the family of all $1$-Lipschitz, bounded functions. This was introduced in the work of \cite{kakadeF08} who termed it {\em weak calibration}, the terminology of smooth calibration is from \cite{GopalanKSZ22}. 
\begin{definition}[Smooth calibration]
\label{def:sm}
Let $L$ be the family of all $1$-Lipschitz functions $w : [0, 1] \to [-1, 1]$. 
The \emph{smooth calibration error} of a distribution $\Gamma$ over $[0,1]\times\{0,1\}$ is defined as weighted calibration error with respect to the family of all $1$-Lipschitz functions
\begin{equation}
    \label{eq:S-C-error}
    \scerror(\Gamma) := \wCE^L(\Gamma).
\end{equation}
Accordingly, for a distribution $\mD$ and a predictor $f$, we define 
\[ \scerror_\mD(f):=\scerror(\mD_f) = \wCE^L(\mD_f) = \wCE^L_\mD(f).\]
\end{definition}

Our main result on the smooth calibration error is the following.
\begin{theorem}
\label{thm:smce-equals-ldce}
For any distribution $\Gamma$ over $[0,1]\times\{0,1\}$, we have
\begin{equation*}
    \frac{1}{2} \ldCE(\Gamma) \leq \scerror(\Gamma) \leq 2 \ldCE(\Gamma).
\end{equation*}
\end{theorem}

Combining this with \cref{cor:intce2}, we conclude that $\scerror$ is a $(1,2)$-consitent calibration measure, and yields an optimal degree-$2$ approximation to $\dCE$.
Along the way, we will find an efficient algorithm for computing $\ldCE(\Gamma)$ (see \Cref{remark:dual-2}).

As it is often the case, the inequality $\scerror(\Gamma) \leq 2 \ldCE(\Gamma)$ is significantly easier to prove (corresponding to the weak duality). We will start by proving this easier direction.
The following lemma is a strengthening of the fact that $\scerror$ is Lipschitz continuous in the predictor $f$ (i.e. for two predictors $f,g$ defined on the same set $\mathcal{X}$, we have $|\scerror_\mD(f) - \scerror_\mD(g)| \leq 2\ell_1(f,g)$).
\begin{lemma}
\label{clm:smooth-calibration-continuity}
Let $\Pi$ be a distribution over $[0,1]\times[0,1]\times\{0,1\}$.
For any $1$-Lipschitz function $w : [0,1] \to [-1, 1]$, we have 
\[
\lt|\E_{(u,v,y)\sim\Pi}[(y-u) w(u)] - \E_{(u,v,y)\sim\Pi}[(y - v) w(v)]\rt| \le  2\E_{(u,v,y)\sim\Pi}|u -v|.
\]
\end{lemma}
\begin{proof}
 We have
\begin{align*}
    \E [(y-u) w(u) - (y-v) w(v)]  \le {} & \E |(y - u)(w(u) - w(v))| + \E |(u - v) w(v)| \\
    \leq {} & 2\E|u-v|.  \tag{since $|y - u| \le 1$ and $|w(u) - w(v)| \leq |u-v|$}
    \end{align*}
\end{proof}
We use this to prove the upper bound on $\scerror(\Gamma)$ in \Cref{thm:smce-equals-ldce}.

\begin{proof}[Proof of Theorem \ref{thm:smce-equals-ldce} (upper bound)]
By the definition of $\ext(\Gamma)$, for any distribution $\Pi\in \ext(\Gamma)$, the distribution of $(u,y)$ for $(u,v,y)\sim \Pi$ is perfectly calibrated. Therefore,
\[
\E_{(u,v,y)\sim\Pi}[(y - u)w(u)] = 0.
\]
By \Cref{clm:smooth-calibration-continuity},
\begin{align*}
\lt|\E_{(u,v,y)\sim\Pi}[(y - v)w(v)]\rt| & \le 2\E_{(u,v,y)\sim\Pi}[|u - v|] + \lt|\E_{(u,v,y)\sim\Pi}[(y - u)w(u)]\rt|\\
& \leq 2\E_{(u,v,y)\sim\Pi}[|u - v|].
\end{align*}
Taking infimum over $\Pi\in \ext(\Gamma)$ completes the proof.
\end{proof}

To prove \Cref{thm:smce-equals-ldce}, it remains to prove the lower bound on $\scerror(\Gamma)$. We prove that in the rest of the section.

\subsection{Linear Program Formulation of Lower Calibration Distance}
\label{sec:ldce-lp}
For a distribution $\Gamma$ over $[0,1]\times \{0,1\}$,
we show that a discretized version of $\ldCE(\Gamma)$ can be formulated as the optimal value of a linear program, and the error caused by the discretization can be made arbitrarily small. We then use the strong duality theorem of linear programming to prove the lower bound of $\scerror(\Gamma)$ in \Cref{thm:smce-equals-ldce}. The linear programming formulation also allows us to give an alternative proof of the upper bound in \Cref{thm:smce-equals-ldce} using the weak duality theorem. Moreover, the linear program formulation gives us an efficient algorithm for estimating $\ldCE(\Gamma)$.

Our first step is to assume that $\Gamma$ is a distribution over $V\times \{0,1\}$ for some \emph{finite} set $V\subseteq[0,1]$. This is mostly without loss of generality because for $\varepsilon > 0$ we can round every value $v\in [0,1]$ in $(v,y)\sim \Gamma$ to the closest value in $\{0, \varepsilon,2\varepsilon,\ldots\}\cap [0,1]$ without changing $\ldCE(\Gamma)$ by more than $\varepsilon$. The following definition allows us to further define discretized versions of $\ext(\Gamma)$ and $\ldCE(\Gamma)$:
\begin{definition}
\label{def:ldCE-U}
Let $U,V\subseteq [0,1]$ be finite sets.
Let $\Gamma$ be a distribution over $V \times\{0,1\}$.
Define the set $\ext^U(\Gamma)$ to consist of all joint distributions $\Pi$ of triples $(u, v, y) \in U \times V \times \bits$, such that 
 \begin{itemize}
     \item the marginal distribution of $(v,y)$ is $\Gamma$;
     \item the marginal distribution $(u, y)$ is perfectly calibrated: $\E_\Pi[y|u] = u$.
 \end{itemize}
 We define $\ldCE^U(\Gamma)$ to be
\begin{equation}
    \ldCE^U(\Gamma) := \inf_{\Pi \in \ext^U(\Gamma)} \E_{(u,v,y)\sim\Pi} |u -v| \label{eq:ldCE-U}.
\end{equation}
\end{definition}
Later in \Cref{lm:ldCE-discretization} we will show that $\ldCE^U(\Gamma)$ is close to $\ldCE(\Gamma)$ as long as $U$ is a suitably rich class. For now, we show how to formulate $\ldCE^U(\Gamma)$ as the optimal value of a linear program: 
\begin{lemma}
\label{cor:ldCE-linear}
Let $U,V,\Gamma$ be defined as in \Cref{def:ldCE-U} and assume $\{0,1\}\subseteq U$. By a slight abuse of notation, we define $\Gamma(v,y)$ to be the probability mass of $\Gamma$ on $(v,y)\in V\times \{0,1\}$. Then the following linear program with variables $\Pi(u,v,y)$ for $(u,v,y)\in U\times V\times \{0,1\}$ is feasible and its optimal value equals $\ldCE^U(\Gamma)$:
\begin{align}
\mathrm{minimize}\quad & \sum_{(u,v,y)\in U\times V\times \{0,1\}}|u - v|\,\Pi(u,v,y)\label{eq:primal}\\
\mathrm{s.t.}\quad & \sum_{u\in U}\Pi(u,v,y) = \Gamma(v,y), && \text{for every }(v,y)\in V\times \{0,1\}; \tag{$r(v,y)$}\\
& (1 - u)\sum_{v\in V}\Pi(u,v,1) = u\sum_{v\in V}\Pi(u,v,0), && \text{for every }u\in U; \tag{$s(u)$}\\
& \Pi(u,v,y) \ge 0, && \text{for every }(u,v,y)\in U\times V\times \{0,1\}.\notag
\end{align}
Moreover, the dual of the linear program \eqref{eq:primal} is the following linear program \eqref{eq:dual} with variables $r(v,y)$ and $s(u)$ for $u\in U,v\in V$ and $y\in \{0,1\}$. By the duality theorem, the optimal value of \eqref{eq:dual} is also $\ldCE^U(\Gamma)$.
\begin{align}
\mathrm{maximize} \quad & \sum_{(v,y)\in V\times \{0,1\}}r(v,y)\Gamma(v,y)\label{eq:dual}\\
\mathrm{s.t.}\quad 
& r(v,y) \le |u - v| +(y - u) s(u), \quad \text{for every }(u,v,y)\in U\times V\times \{0,1\}. \tag{$\Pi(u,v,y)$}
\end{align}
\end{lemma}
\begin{proof}
Any distribution $\Pi$ over $U\times V\times \{0,1\}$ corresponds to a function $\Pi:U\times V\times \{0,1\}\to \R$ where $\Pi(u,v,y)$ is the probability mass on $(u,v,y)\in U\times V\times \{0,1\}$. It is easy to check that if the distribution $\Pi$ belongs to $\ext(\Gamma)$, then the function $\Pi$ satisfies the constraints of \eqref{eq:primal}, and conversely, any function $\Pi$ satisfying the constraints also corresponds to a distribution $\Pi\in \ext(\Gamma)$. In particular, for $(u,v,y)\sim\Pi$, the first constraint ensures that the marginal distribution of $(v,y)$ is $\Gamma$, and the second constraint ensures that the marginal distribution of $(u,y)$ is calibrated. Moreover, the objective of \eqref{eq:primal} corresponds to the expectation $\E_{(u,v,y)\sim \Pi}|u - v|$ in \eqref{eq:ldCE-U}. This proves that $\ldCE^U(\Gamma)$ is equal to the optimal value of the linear program \eqref{eq:primal}. To show that the linear program \eqref{eq:primal} is feasible, consider setting $\Pi(u,v,y) = \Gamma(v,y)$ if $u = y$, and setting $\Pi(u,v,y) = 0$ if $u\ne y$. It is easy to check that this choice of $\Pi$ satisfies the constraints of \eqref{eq:primal} using our assumption that $\{0,1\}\subseteq U$.
\end{proof}
\begin{claim}
\label{claim:bounded-s}
Let $U,V\subseteq[0,1]$ be finite sets and assume $\{0,1\}\subseteq U$. The optimal value of the dual linear program \eqref{eq:dual} does not change even if we add the additional constraints $-1 \le s(u) \le 1$ for every $u\in U$.
\end{claim}
\begin{proof}
Let $r,s$ be a feasible solution to \eqref{eq:dual}.
Setting $u = y$ in the constraint of \eqref{eq:dual}, we have
\begin{equation}
r(v,y) \le |v - y|\quad \text{for every }(v,y)\in V\times \{0,1\}.\label{eq:r1v}
\end{equation}
Consider any $u\in U$ for which $s(u) > 1$. For every $v\in V$,
\begin{align*}
r(v,0) & \le |u - v| + (0 - u)s(u) \le |u - v| + (0 - u), \quad \text{and}\\
r(v,1) & \le |v - 1| \le |u - v| + |1 - u| = |u - v| + (1 - u). \tag{by \eqref{eq:r1v}}
\end{align*}
Therefore, changing $s(u)$ to $1$ does not violate the constraint of \eqref{eq:dual}. Similarly when $s(u) < -1$, we can change $s(u)$ to $-1$ without violating any constraint.
\end{proof}
\begin{remark}
\label{remark:dual-1}
Since $\Gamma(v,y)$ in the objective of \eqref{eq:dual} is nonnegative, it is always without loss of generality to assume that $r(v,y)$ is as large as possible, i.e., 
\begin{equation}
\label{eq:max-r}
r(v,y) = \min_{u\in U}\big(|u - v| + (y - u)s(u)\big).
\end{equation}
Assuming \eqref{eq:max-r}, it is easy to check that $r(v,y)$ is $1$-Lipschitz in $v$, i.e., $|r(v_1,y) - r(v_2,y)|\le |v_1 - v_2|$ for every $v_1,v_2\in V$ and $y\in \{0,1\}$. 
When $\{0,1\}\subseteq U$,
\Cref{claim:bounded-s} allows us to assume that $-1 \le s(u)\le 1$ without loss of generality. When this assumption and \eqref{eq:max-r} are both satisfied, it is easy to verify that $r(v,y)\in [-|v - y|, |v - y|]\subseteq [-1,1]$ and $r(v,1) - r(v,0)\in [-1,1]$ for every $v\in V$ and $y\in \{0,1\}$. Indeed, in \eqref{eq:max-r} we have $|u - v| + (y - u)s(u)\ge |u - v| - |y - u| \ge -|v - y|$ and thus $r(v,y) \ge -|v - y|$. The upper bound $r(v,y)\le |v - y|$ has been shown in \eqref{eq:r1v}. 
\end{remark}
\begin{remark}
\label{remark:dual-1-1}
When $U,V$ are finite sets satisfying $\{0,1\}\subseteq U = V \subseteq [0,1]$, using \Cref{remark:dual-1} one can verify that the dual linear program \eqref{eq:dual} has the same optimal value as the following linear program:
\begin{align}
\mathrm{maximize} \quad & \sum_{(v,y)\in V\times \{0,1\}}r(v,y)\Gamma(v,y)\label{eq:dual-2}\\
\mathrm{s.t.}\quad 
& |r(v_1,y) - r(v_2,y)| \le |v_1 - v_2|, && \text{for every }(v_1,v_2,y)\in V\times V\times \{0,1\};\label{eq:dual-2-1}\\
& r(v,y) \le (y - v) s(v), && \text{for every }(v,y)\in V\times \{0,1\}.\nonumber
\end{align}
The constraints \eqref{eq:dual-2-1} can be enforced simply by checking neighboring pairs $(v_1,v_2)$ when the values in $V$ are sorted. Thus the effective number of constraints in \eqref{eq:dual-2-1} is $O(|V|)$.
\end{remark}
\begin{remark}
\label{remark:dual-2}
Let $U\subseteq[0,1]$ be a finite set satisfying $\{0,1\}\subseteq U$.
Given a distribution $\Gamma$ over $V\times \{0,1\}$ for a finite $V\subseteq[0,1]$, \Cref{cor:ldCE-linear} allows us to efficiently compute $\ldCE^U(\Gamma)$ by solving either the primal linear program \eqref{eq:primal} or the dual linear program \eqref{eq:dual}. When $U = V$, it may be more efficient to solve the equivalent linear program \eqref{eq:dual-2} which effectively has only $O(|V|)$ constraints as we mention in \Cref{remark:dual-1-1}. Moreover, given two distributions $\Gamma$ and $\Gamma'$ that are close in a certain Wasserstein distance, using the dual linear program \eqref{eq:dual} we can show that $\ldCE^U(\Gamma')$ and $\ldCE^U(\Gamma)$ are close (we make this formal in \Cref{lem:sample-ldce}). This allows us to estimate $\ldCE^U(\Gamma)$ only using examples drawn from $\Gamma$ (see \Cref{sec:est-ldce}). In \Cref{lm:ldCE-discretization} below we show that choosing $|U| = O(1/\varepsilon)$ suffices to ensure that $\ldCE^U(\Gamma)$ approximates $\ldCE(\Gamma)$ up to an additive error $\varepsilon$.
\end{remark}
The following lemma relates $\ldCE^U(\Gamma)$ and $\ldCE(\Gamma)$:
\begin{lemma}
\label{lm:ldCE-discretization}
Let $\Gamma$ be a distribution over $V\times \{0,1\}$ for a finite $V\subseteq[0,1]$. Let $U$ be a finite $\varepsilon$ covering of $[0,1]$ satisfying $\{0,1\}\subseteq U$. That is, there exists $\sigma:[0,1]\to U$ such that $|u - \sigma(u)| \le \varepsilon$ for every $u\in [0,1]$. Then we have
\[
\ldCE(\Gamma) \le \ldCE^U(\Gamma) \le \ldCE(\Gamma) + 2\varepsilon.
\]
\end{lemma}
\begin{proof}
It is clear from the definitions that $\ldCE(\Gamma) \le \ldCE^U(\Gamma)$. It remains to prove $\ldCE^U(\Gamma) \le \ldCE(\Gamma) + 2\varepsilon$.
It suffices to prove the following for any arbitrary distribution $\Pi\in \ext(\Gamma)$:
\begin{equation}
\label{eq:ldCE-discretization-0}
\ldCE^U(\Gamma) \le \E_{(u,v,y)\sim \Pi}|u - v| + 2\varepsilon.
\end{equation}
For a positive integer $m$, partition the interval $[0,1]$ into $I_1,\ldots,I_m$, where $I_1 = [0,1/m]$ and $I_j = ((j-1)/m, j/m]$ for $j = 2,\ldots,m$. For every $j = 1,\ldots,m$, define $u_j:= \E_{(u,v,y)\sim \Pi}[u|u\in I_j]$. For $(u,v,y)\sim \Pi$, we define random variable $u':= u_j$ where $j$ is chosen such that $u\in I_j$. Define $\Pi'$ to be the joint distribution of $(u',v,y)$. It is easy to check that $\Pi'\in \ext^{U'}(\Gamma)$, where $U' = \{u_1,\ldots,u_m\}$. Therefore,
\begin{equation}
\label{eq:ldCE-discretization-1}
\ldCE^{U'}(\Gamma) \le \E_{(u',v,y)\sim \Pi'}|u' - v| \le \E_{(u,v,y)\sim \Pi}|u - v| + O(1/m).
\end{equation}
Now we show that
\begin{equation}
\label{eq:ldCE-discretization-2}
\ldCE^U(\Gamma) \le \ldCE^{U'}(\Gamma) + 2\varepsilon.
\end{equation}
Consider any feasible solution $r:V\times\{0,1\}\to \R$ and $s:U\to \R$ to the dual linear program \eqref{eq:dual}. To prove \eqref{eq:ldCE-discretization-2}, it suffices to construct $r':V\times \{0,1\}\to \R$ and $s':U'\to \R$ such that $r'$ and $s'$ are a feasible solution to \eqref{eq:dual} after we replace $U$ with $U'$ in \eqref{eq:dual}, and
\begin{equation}
\label{eq:r's'}
\sum_{(v,y)\in V\times \{0,1\}} r'(v,y)\Gamma(v,y) \ge \sum_{(v,y)\in V\times \{0,1\}} r(v,y)\Gamma(v,y) - 2\varepsilon.
\end{equation}
By \Cref{remark:dual-1}, it is without loss of generality to assume that $s(u)\in [-1,1]$. We choose $r'(v,y) = r(v,y) - 2\varepsilon$ and $s'(u) = s(\sigma(u))$.
This immediately guarantees \eqref{eq:r's'}. The following calculation verifies that $r'$ and $s'$ are feasible solutions: for $(u,v,y)\in U'\times V\times \{0,1\}$,
\begin{align*}
& |u - v| + (y - u)s'(u)\\
= {} & |u - v| + (y - u)s(\sigma(u))\\
\ge {} & |\sigma(u) - v| + (y - \sigma(u)) s(\sigma(u)) - 2\varepsilon \tag{by $|u - \sigma(u)| \le \varepsilon$ and $|s(\sigma(u))|\le 1$}\\
\ge {} & r(v,y) - 2\varepsilon\\
= {} & r'(v,y). 
\end{align*}
Combining \eqref{eq:ldCE-discretization-1} and \eqref{eq:ldCE-discretization-2},
\[
\ldCE^U(\Gamma) \le \E_{(u,v,y)\sim \Pi}|u - v| + 2\varepsilon + O(1/m).
\]
Taking $m$ to infinity proves \eqref{eq:ldCE-discretization-0}.
\end{proof}
We prove lower and upper bounds for $\scerror(\Gamma)$ using $\ldCE^U(\Gamma)$ in the two lemmas below.
\begin{lemma}
\label{lm:hard-direction}
Let $\Gamma$ be a distribution over $V\times \{0,1\}$ for a finite $V\subseteq[0,1]$. Define $U = V\cup \{0,1\}$. Then $\ldCE^U(\Gamma) \le 2\scerror(\Gamma)$.
\end{lemma}
\begin{proof}
It suffices to prove that for any feasible solution $r$ and $s$ to the dual linear program \eqref{eq:dual}, it holds that
\[
\sum_{(v,y)\in V\times \{0,1\}} r(v,y)\Gamma(v,y) \le 2\scerror(\Gamma).
\]
By \Cref{remark:dual-1}, we can assume without loss of generality that $|r(v_1,y) - r(v_2,y)|\le |v_1 - v_2|$ and $r(v,1) - r(v,0)\in [-1,1]$.
Define $w(v):= r(v,1) - r(v,0)$. Then $w$ is $2$-Lipschitz and $w(v)\in [-1,1]$, which implies that
\[
2\scerror(\Gamma) \ge \E_{(v,y)\sim \Gamma}[(y - v)w(v)].
\]
Moreover, setting $u = v$ in the constraint of \eqref{eq:dual}, we have
\[
(1 - v)r(v,0) + v r(v,1) \le -(1 - v)vs(v) + v(1 - v)s(v) = 0,
\]
which implies that
\[
-vw(v) \ge r(v,0).
\]
Therefore,
\[
(y - v)w(v) \ge yw(v) - vw(v) \ge y(r(v,1) - r(v,0)) + r(v,0) = r(v,y).
\]
This implies that
\[
2\scerror(\Gamma) \ge \E_{(v,y)\sim \Gamma}[(y - v)w(v)] \ge \E_{(v,y)\sim \Gamma}r(v,y) = \sum_{(v,y)\in V\times \{0,1\}} r(v,y)\Gamma(v,y).\qedhere
\]
\end{proof}
\begin{lemma}
\label{lm:easy-direction}
Let $\Gamma$ be a distribution over $V\times \{0,1\}$ for a finite $V\subseteq[0,1]$.
For any finite $U\subseteq [0,1]$, we have $\scerror(\Gamma) \le 2 \ldCE^U(\Gamma)$.
\end{lemma}
\begin{proof}
Let $w:[0,1]\to [-1,1]$ be a $1$-Lipschitz function.
We choose $r(v,y) = (y-v)w(v)/2$ and $s(u) = w(u)/2$.
The following calculation verifies that this choice of $r$ and $s$ satisfies the constraints in \eqref{eq:dual}:
\begin{align*}
r(v,y) - (y - u)s(u) & = \frac 12((y - v)w(v) - (y - u)w(u))\\
& = \frac 12((u - v)w(v) + (y - u)(w(v) - w(u))) \tag{by $|w(v)|\le 1$, $|y -u|\le 1$, and $|w(v) - w(u)| \le |u - v|$}\\
& \le |u - v|.
\end{align*}
Therefore,
\[
\ldCE^U(\Gamma) \ge \sum_{(v,y)\in V\times\{0,1\}} r(v,y)\Gamma(v,y) = \E_{(v,y)
\sim \Gamma}[r(v,y)] = \E_{(v,y)\sim \Gamma}[(y - v)w(v)/2].
\]
Taking supremum over $w$ completes the proof.
\end{proof}
In the proofs of \Cref{lm:hard-direction,lm:easy-direction} above, we use the fact that $\ldCE^U(\Gamma)$ is equal to the optimal value of the dual linear program \eqref{eq:dual}. However, for \Cref{lm:hard-direction} we only need the fact that $\ldCE^U(\Gamma)$ is \emph{at most} the optimal value, whereas for \Cref{lm:easy-direction} we only need the fact that $\ldCE^U(\Gamma)$ is \emph{at least} the optimal value. That is, our proof of \Cref{lm:hard-direction} is based on the strong duality theorem, whereas the proof of \Cref{lm:easy-direction} is based on the weak duality theorem. Below we apply \Cref{lm:hard-direction} and \Cref{lm:easy-direction} to prove the lower and upper bounds of $\scerror(\Gamma)$ in \Cref{thm:smce-equals-ldce}, respectively.
\begin{proof}[Proof of \Cref{thm:smce-equals-ldce}]
For $\varepsilon_1 > 0$, we round the value $v\in [0,1]$ in $(v,y)\sim \Gamma$ to the closest value $v'\in \{0,\varepsilon_1,2\varepsilon_1,\ldots,\}\cap [0,1]$. Let $\Gamma'$ be the distribution of $(v',y)$. It is clear that $|\ldCE(\Gamma') - \ldCE(\Gamma)| \le \varepsilon_1$, and by \Cref{clm:smooth-calibration-continuity} we have $|\scerror(\Gamma') - \scerror(\Gamma)| \le 2\varepsilon_1$.

By \Cref{lm:ldCE-discretization}, for any $\varepsilon_2 > 0$, there exists a finite set $U\subseteq[0,1]$ such that $\ldCE(\Gamma')\le \ldCE^U(\Gamma') \le \ldCE(\Gamma') + \varepsilon_2$. Moreover, we can always choose $U$ so that $\{0,1\}\cup V\subseteq U$. Now by \Cref{lm:hard-direction},
\begin{align*}
\ldCE(\Gamma) -\varepsilon_1 \le \ldCE(\Gamma') \le \ldCE^U(\Gamma') \le 2\scerror(\Gamma') \le 2\scerror(\Gamma) + 4\varepsilon_1.
\end{align*}
By \Cref{lm:easy-direction},
\[
\scerror(\Gamma) - 2\varepsilon_1 \le \scerror(\Gamma') \le 2 \ldCE^U(\Gamma') \le 2\ldCE(\Gamma') + 2\varepsilon_2 \le 2\ldCE(\Gamma) + 2\varepsilon_1 + 2\varepsilon_2.
\]
Taking $\varepsilon_1,\varepsilon_2\to 0$ completes the proof.
\end{proof}
We conclude with an efficient algorithm for smooth calibration error. The generalization bound to accompany it will be proved in \Cref{cor:smCE-generalization} in \Cref{sec:sample-complexity}. 
\begin{theorem}
\label{thm:smooth-lp}
For the empirical distribution $\Gamma$ over a sample $S = ( (v_1, y_1), \ldots (v_n, y_n) )\in ([0,1]\times \{0,1\})^n$ we can calculate 
\[ \scerror(\Gamma) := \sup_{w \in L} \frac{1}{n}\sum_i (y_i - v_i) w(v_i)\] 
in time $\mathrm{poly}(n)$, where $L$ is the family of all $1$-Lipschitz functions $w : [0, 1] \to [-1, 1]$. 
\end{theorem}
\begin{proof}
For a sample $S$, the supremum above can be computed using the following linear maximization problem in variables $z_i$ (that are intended to be equal to $w(v_i)$).
\begin{align*}
    & \max \frac{1}{n}\sum_i (y_i - v_i) z_i \\
     \mathrm{s.t.} \quad
     & -1 \leq z_i \leq 1, \quad\forall i; \\
    & |z_i - z_j| \leq |v_i - v_j|, \quad \forall i,j.
\end{align*}
Indeed, on one hand any Lipschitz function $w$ yields a feasible solution to this linear program, by setting $z_i := w(v_i)$. On the other hand, for any feasible solution to this program, we can find a $1$-Lipschitz function $w$ satisfying $w(v_i) = z_i$ for all $i$, using a piecewise linear extension.
\end{proof}
\section{Kernel Calibration Error}
\label{sec:kernels}

We now consider kernel calibration ($\kCE^K$),
which is a special case of weighted calibration (Definition~\ref{def:wce})  
where the family of weight functions lies in a Reproducing Kernel Hilbert Space $\cH$.
This notion was previously defined in~\cite{ksj18a} (called ``MMCE''), 
motivated as a differentiable proxy for ECE.

We advance the theory of kernel calibration in several ways.
First, we show that the kernel calibration error for the \emph{Laplace} kernel is in fact a consistent calibration measure.
This provides strong theoretical justification for measuring kernel calibration,
and also gives a reason to use the \emph{Laplace} kernel specifically, among other choices of kernel.
Indeed, we complement the Laplace kernel with a negative result:
using the Gaussian kernel does not yield a consistent calibration measure.

Finally, as a curiosity, we observe that the techniques of \cite{rahimi2007kernel} yield an alternate estimator for
Laplace kernel calibration error, which bears similarity to
the randomized-binning estimator of interval calibration error.

\subsection{Preliminaries}
We consider a \emph{Reproducing Kernel Hilbert Space} of functions on a real line $\R$, i.e. a Hilbert space $\cH$ of functions $h:\R\to\R$, with the associated norm $\|\cdot\|_{\cH}$. This space is equipped with the feature map $\phi : \R \to \cH$, satisfying $\inprod{h, \phi(v)}_{\cH}= h(v)$. The associated kernel $K : \R \times \R \to \R$ is now defined as $K(u,v) = \inprod{\phi(u), \phi(v)}_{\cH}$.

\begin{definition}[Kernel Calibration Error~\citep{ksj18a}]
Given a RKHS $\cH$ with the norm $\|\cdot\|_{\cH}$, we can consider a class of functions bounded by $1$ with respect to this norm $B_{\cH} := \{ h \in \cH : \|h\|_{\cH} \leq 1\}$,
and we can study the associated weighted calibration error $\wCE^{B_{\cH}}$ (as in \Cref{def:wce}). 

The \emph{kernel calibration error} of a distribution $\Gamma$ over $[0,1]\times\{0,1\}$ associated with the kernel $K$ is defined as weighted calibration error with respect to the family of weight functions $B_{\cH}$
\begin{equation}
    \kCE^K(\Gamma) := \wCE^{B_{\cH}}(\Gamma).
\end{equation}
Accordingly, for a distribution $\mD$ and a predictor $f$, we define $\kCE^{K,\mD}(f):= \kCE_K(\mD_f)$.
\end{definition}

The following results are standard, from \citep{ksj18a}.
First, $\kCE_K$ can be written as the $K$-norm of a certain function,
without explicitly maximizing over weight functions $h \in \cH$.

\begin{lemma}[\cite{ksj18a}]
\label{lem:kernel-dual-definition}
For any kernel $K$ and the associated RKHS $\cH$, and any distribution $\Gamma$ over $[0,1]\times\{0,1\}$,
\begin{equation*}
\kCE^{K}(\Gamma) = \|\E_{(v,y)\sim\Gamma}[ (y - v) \phi(v) ]\|_{\cH}.
\end{equation*}
\end{lemma}

We reproduce the proof of this for completeness
in Appendix~\ref{sec:app-ker-proofs}.
This expression can be efficiently evaluated for an empirical distribution on a samples $S = \{(y_1, v_1), \ldots (y_k, v_k)\}$.
\begin{claim}[\cite{ksj18a}]
\label{clm:kernel-efficient-calculation}
Let $\Gamma$ be the empirical distribution over a given sample $\{(v_1, y_1), \ldots, (v_n, y_n)\}$. We can compute $\kCE^K(\Gamma)$ in time $\Oh(n^2)$ using $\Oh(n^2)$ evaluations of the kernel function:
\begin{equation}
\label{eqn:kernel-emperical}
\kCE^K(\Gamma)^2 = \frac{1}{n^2}\sum_{i,j} (y_i - v_i)(y_j - v_j) K(v_i, v_j).
\end{equation}
\end{claim}

In \Cref{sec:sample-complexity} we discuss the convergence of the kernel calibration error for the empirical distribution over the sample, to the kernel calibration error of the entire distribution --- this convergence, together with \Cref{clm:kernel-efficient-calculation} gives an efficient way to estimate the kernel calibration error of a given predictor from a boudned number of samples from the underlying distribution.

\paragraph{The Laplace Kernel.}
We recall standard facts about the Laplace kernel $K_\lap(u,v) := \exp(-|u-v|)$, and its associated RKHS.
It turns out that the norm induced by functions in the associated RKHS has simple explicit expression --- the corresponding space is a Sobolev space.
\begin{fact}[\cite{berlinet2011reproducing}]
For the Laplace kernel $K_\lap(u,v) = \exp(-|u-v|)$, we have associated RKHS $\cH_\lap = \{ h : \R \to \R  : \int \hat{h}(\omega)^2 (1  + \omega^2) \d \omega < \infty \}$, where $\hat{h}$ denotes the Fourier transform of $u$. 
The associated inner product is given by
\begin{equation*}
    \inprod{h_1, h_2}_{K_\lap} = \int_{-\infty}^{\infty} \hat h_1(\omega) \hat h_2(\omega)(1 + \omega^2) \d \omega, 
\end{equation*}
in particular, for function $h: \R \to \R$,
\begin{equation*}
    \|h\|_{K_\lap}^2 = \int_{-\infty}^{\infty} \hat{h}(\omega)^2 (1 + \omega^2) \d \omega = \|h\|_2^2 +  \|h'\|_2^2.
\end{equation*}
\end{fact}

\subsection{Laplace Kernel Calibration Error is a Consistent Calibration Measure}

We now ask whether there is a kernel $K$ for which $\kCE_K$ is a consistent calibration measure.
The main result in this section is to show that this is the case for the \emph{Laplace} kernel.
Specifically, we prove that:
\begin{theorem}
\label{thm:laplace-kernel-bounds}
The Laplace kernel calibration error $\kCEL:=\kCE^{K_\lap}$ satisfies the following inequalities
\begin{equation*}
\frac{1}{3} \scerror(\Gamma) \leq \kCEL(\Gamma) \leq \sqrt{\ldCE(\Gamma)}.
\end{equation*}
\end{theorem}
By \Cref{cor:intce2} and \Cref{thm:smce-equals-ldce} it follows that $\kCEL$ is a $(1/2, 2)$-consistent calibration measure.
Interestingly, the choice of kernel is crucial: we show
that for the Gaussian kernel, the resulting measure does not satisfy robust soundness anymore. Specifically in Appendix~\ref{sec:gaussian-kernel}, we prove the following theorem.
\begin{theorem}
\label{thm:gauss-lb-body}
For every $\varepsilon$, there is a distribution $\Gamma_{\varepsilon}$ over $[0,1] \times \{0, 1\}$, such that $\scerror(\Gamma_\varepsilon) \geq \Omega(\varepsilon^{\Oh(1)})$, and $\kCEG(\Gamma_\varepsilon) \leq \Oh(\exp(-1/\varepsilon))$, where $\kCEG:=\kCE^{K_\gauss}$ is the Gaussian kernel calibration error with $K_\gauss(u,v) = \exp(-(u - v)^2)$.
\end{theorem}

We will start by proving continuity of Laplace kernel calibration error. This following lemma is strengthening of the upper bound in \cref{thm:laplace-kernel-bounds}.
\begin{lemma}
\label{lem:laplace-kernel-continous}
Let $\Pi$ be a distribution over $[0,1]\times[0,1]\times\{0,1\}$. For $(u,v,y)\sim\Pi$, let $\Gamma_1$ be the distribution of $(u,y)$ and $\Gamma_2$ be the distribution of $(v,y)$.
Assume $\E|u - v| \le \varepsilon$.
Then 
\[ |\kCEL(\Gamma_1) - \kCEL(\Gamma_2)| \leq 2 \sqrt{2 \varepsilon}.\]
\end{lemma}
\begin{proof}
Using~\Cref{lem:kernel-dual-definition} and a triangle inequality, we need to show that 
\[ \| \E [ (y-u) \phi(u) - (y-v) \phi(v) ]\|_{\cH} \leq \Oh(\sqrt{\varepsilon}).\] 
By convexity of a norm, we have 
\[  \| \E [ (y-u) \phi(u) - (y-v) \phi(v) ]\|_{\cH}  \leq \E \|u \phi(u) - v \phi(v)\| + \E \|\phi(u) - \phi(v)\|_{\cH}.\]

We will bound both of those terms separately. By Jensen inequality, we have 
\[ \E \|\phi(u) - \phi(v)\|_\cH \leq \sqrt{\E \|\phi(u) - \phi(v)\|_\cH^2},\]
and similarly for $\|u \phi(u) - v \phi(v)\|_{\cH}$. Now
\begin{equation*}
    \E \|\phi(u) - \phi(v)\|_\cH^2 = 2 - \E 2K(u,v) = 2 - \E 2\exp(-|u-v|) \leq 2 \varepsilon,
\end{equation*}
where the inequality follows from $\exp(-x) \geq 1-x$. Similarly,
\begin{equation*}
    \E \|u \phi(u) - v \phi(v)\|_\cH^2 = u^2 + v^2 - 2 uv \E K(u,v) \leq u^2 + v^2 - 2 uv (1 - \varepsilon) \leq 2 \varepsilon.
\end{equation*}
Adding those two together, we get the final bound 
\[ \| \E (y-u) \phi(u) - (y-v)\phi(v)\|_\cH \leq 2 \sqrt{2 \varepsilon}.\qedhere\]
\end{proof}

\begin{lemma}
\label{lem:lipschitz-are-laplace-bounded}
Let $B_\lap$ be a set of functions $w : \R \to \R$ bounded by one with respect to the norm induced by the Laplace kernel $\exp(-|u-v|)$ on the associated RKHS. Then for any $1$-Lipschitz function $w : [0,1] \to [-1, 1]$ and any $\varepsilon$, there is $\|\tilde{w}_{\varepsilon}\|_{\cH_\lap} \leq 3$, such that $|w - \tilde{w}_{\varepsilon}| < \varepsilon$ for all $x \in [0, 1]$.
\end{lemma}
\begin{proof}
For a $1$-Lipschitz function $w : [0,1] \to [-1,1]$, we will first construct the $1$-Lipschitz extension of $w$, say $\tilde{w} : \R \to [-1,1]$, by taking 
\begin{equation*}
    \tilde{w}(t) = \left\{\begin{array}{ll}0 & \textrm{ for } w \leq -1 \\
    w(0)(1 + t) & \textrm{ for } w \in [-1, 0] \\
    w(t) & \textrm{ for } w \in [0,1] \\
    w(1)(-t+2) & \textrm{ for } w \in [1, 2] \\
    0 & \textrm{otherwise.}\end{array} \right.
\end{equation*}

It is standard fact that by taking a convolution of $\tilde{w}_\varepsilon := \tilde{w} * g_\varepsilon$, where $g_\varepsilon$ is a smooth non-negative function, supported on $[-\varepsilon, \varepsilon]$ with $\int_\R g_\varepsilon = 1$, then $\tilde{w}_\varepsilon$ will satisfy the following conditions \cite{stackexchange}.
\begin{enumerate}
    \item $\forall v,\, |\tilde{w}(v) - \tilde{w}_{\varepsilon}(v)| \leq \varepsilon$,
    \item $\supp(\tilde{w}_{\varepsilon}) \subset [-1-\varepsilon, 2 + \varepsilon]$,
    \item $\tilde{w}_{\varepsilon}$ is smooth, and moreover $\forall v,\, |w'(v)| \leq 1$.
\end{enumerate}

Combining properties those, we get 
\begin{equation*}
    \|\tilde{w}_{\varepsilon}\|_{\cH_\lap}^2 = \|f\|_2^2 + \|f'\|_2^2 \leq (3 + \varepsilon)(1+\varepsilon)^2 + (3+\varepsilon) \leq 7.
\end{equation*}

\end{proof}
We are now ready to prove \Cref{thm:laplace-kernel-bounds}

\begin{proof}[Proof of \Cref{thm:laplace-kernel-bounds}]
By \Cref{lem:lipschitz-are-laplace-bounded}, we have $\scerror(\Gamma) \leq \kCEL(\Gamma)$. Indeed, let us take a Lipschitz weight function $w$, such that $\scerror(\Gamma) = \E_{(v,y)\sim\Gamma} (y-v)w(v).$ Now we can take $\tilde{w} \in 3B_\lap$ as in \Cref{lem:lipschitz-are-laplace-bounded}, such that $\E (y-v) \tilde{w}(v) \geq \scerror(\Gamma) - \varepsilon$, and therefore $\kCEL(\Gamma) \geq \frac{1}{3}(\scerror(\Gamma) - \varepsilon)$. Taking $\varepsilon \to 0$ proves the desired fact.

The other inequality $\kCEL(\Gamma) \leq \ldCE(\Gamma)^{1/2}$ follows directly from \Cref{lem:laplace-kernel-continous}: by definition of $\ldCE$ we can find distribution $\Pi\in \ext(\Gamma)$ of $(u,v,y)$ s.t. $(v, y)$ is distributed according to $\Gamma$, and the distribution $\Gamma'$ of  $(u, y)$ is perfectly calibrated, and $\E |u - v| = \ldCE(\Gamma)$. Now 
\begin{equation*}
    \kCEL(\Gamma) \leq \kCEL(\Gamma') + \Oh(\sqrt{\ldCE(\Gamma')}) = \Oh(\sqrt{\ldCE(\Gamma')}).\qedhere
\end{equation*}
\end{proof}

\subsection{Alternate Estimation Algorithms}
\label{sec:alternate}

Computing the exact kernel calibration error from $n$ samples requires $O(n^2)$ time, by the computation in \Cref{clm:kernel-efficient-calculation}.
However, we can approximate this quantity efficiently,
by simply sub-sampling terms independently from the $n^2$ terms in Equation~\eqref{eqn:kernel-emperical}. 
By standard arguments, sub-sampling $\Oh(\eps^{-2} \log(\delta^{-1}))$
terms yields an estimator accurate within $\pm \eps$, with probability $1-\delta$.

In this section we will describe two alternate algorithms for estimating
Laplace kernel calibration specifically. These algorithms do not
improve over the naive sub-sampling estimator in worst-case guarantees,
and are not algorithmically novel--- the algorithms are corollaries of \citep{rahimi2007kernel}.
Nevertheless, we include them to
highlight an intriguing connection:
one of these algorithms involves a
randomized binning-like estimator, which is suggestive of
(though not formally equivalent to) our notion of $\intCE$.
We consider it somewhat surprising that the formal connection between
$\intCE$ and $\kCE$ extends to an \emph{algorithmic} similarity
between their respective estimators.

We present the algorithms as constructing an unbiased and bounded estimator for the following quantity.

\begin{definition}[Empirical Kernel Calibration Error]
For a sample $S=((v_1, y_1), (v_2, y_2), \ldots, (v_k, y_k))$ we define the empirical kernel calibration error as
\begin{equation*}
    \ekCE^K(S)^2 = \frac{1}{k^2} \sum_{i,j} (y_i - v_i)(y_j - v_j) K(v_i, v_j).
\end{equation*}
\end{definition}

In the abstract, Rahimi-Recht provides an efficient way of finding a low-rank approximation of the Kernel matrix $M \in \R^{k\times k}$ given by $M_{ij} = K(v_i, v_j)$. Specifically, a version of the Claim~1 in their paper can be stated as follows.

\begin{theorem}{\cite{rahimi2007kernel}}
\label{thm:random-fourier-features}
Let $v_1, \ldots v_k \in [0,1]$, and let us consider the Kernel matrix $M \in \R^{k \times k}$, $M_{i,j} = K(v_i, v_j)$ where $K(u,v) = K(|u-v|)$ is a shift invariant positive definite kernel. There is a random $z \in \bC^{n}$, such that if we take $\tilde{M} := z z^* \in \bC^{n\times n}$, then $\E[ \tilde{M}] = M$, and moreover $\|z\|_{\infty} \leq 1$. 
\end{theorem}

For the specific case of the Laplace kernel, the random vector $u \in \bC^{n}$ guaranteed by Theorem~\ref{thm:random-fourier-features} can be sampled as follows.
The proof of \Cref{lem:laplace-rff} is standard, and included for completeness in Appendix~\ref{sec:app-ker-proofs}.
\begin{claim}
\label{lem:laplace-rff}
For any sequence of points $v_1, v_2, \ldots, v_k \in \R$,
if we chose $\omega \sim \mathrm{Cauchy}(1)$, and a vector $z \in \bC^n$ given by $z_j := \exp(-i\omega v_j)$, then $\|z\|_\infty \leq 1$ and 
$\E[z z^*] = M \in \R^{k \x k}$, where $M_{i, j} = \exp(-|v_i - v_j|)$.
\end{claim}

\begin{remark}
Rahimi-Recht \citep{rahimi2007kernel} actually prove a significantly stronger version of \Cref{thm:random-fourier-features}, showing that under certain additional conditions on the kernel $K$ the average of independent rank one matrices $M' := \frac{1}{s} \sum_i \tilde{M}_i$ as above uniformly converges to $M$. That is, for $s = \Oh(\varepsilon^{-2} \log(\delta^{-1}))$, with probability $1-\delta$ they provide a bound $\forall i,j \, |M_{ij} - M'_{ij}| \leq \varepsilon$. These additional conditions are not satisfied by the Laplace kernel, but we will not need this stronger, uniform bound.
\end{remark}

\begin{algorithm}[t]
\caption{Random Fourier features algorithm for Laplace kernel calibration error estimation \label{alg:random-fourier-features}}
\begin{algorithmic}[1]
\Function{LaplaceFourierEstimate}{$(v_1, y_1), \ldots (v_n, y_k)$}
\Let{$\omega$}{$\mathrm{Cauchy}(1)$}
\Let{$R$}{$\sum_j (v_j - y_j)\exp(- i \omega v_i)$}
\State \Return{$\frac{|R|^2}{k^2}$}
\EndFunction
\end{algorithmic}
\end{algorithm}

With \Cref{thm:random-fourier-features} we can find in linear time an unbiased and bounded estimate of the quantity $\ekCE(S)^2$. Indeed, we have 
\[\ekCE(S)^2 = \frac{1}{k^2} \sum_{i,j} (y_i - v_i)(y_j - v_j) K(v_i, v_j) = \frac{1}{k^2} r^T M r\] where $r_i := (v_i - y_i)$ 
and $M$ is the kernel matrix defined above. Now, on one hand for a random vector $z$ as in~\Cref{thm:random-fourier-features} we have 
\[\frac{1}{k^2} \E[ r^T z z^* r] = \frac{1}{n^2} r^T (\E [z z^*]) r = \frac{r^T M r}{k^2} = \ekCE(S)^2.\] On the other hand 
\[\frac{r^T z z^* R}{k^2} = \frac{|\inprod{z,r}|^2}{k^2} \leq 1,\] 
since both $z$ and $r$ have entries bounded by $1$. Therefore the quantity $\frac{|\inprod{z,r}|^2}{n^2}$ is an unbiased and bounded by one estimate of $\ekCE(S)^2$.

\Cref{alg:random-fourier-features} describes an implementation of the above estimator for the Laplace kernel, which runs in linear time. The argument described above
(directly applying the results of \citep{rahimi2007kernel}) yields the following guarantee.
\begin{theorem}
\label{thm:lapace-rff-algo}
The algorithm \textsc{LaplaceFourierEstimate} provides an unbiased and bounded by $1$ estimate for $\ekCE(S)^2$
and runs on $n$ samples in time $\Oh(n)$.
\end{theorem}

If we want to translate this unbiased estimator into 
a confidence interval,
it is enough to run \textsc{LaplaceFourierEstimate} independently $\Oh(\varepsilon^{-2} \log \delta^{-1})$ times, and average the results, 
to get an estimate of $\ekCE_{\mathsf\lap}(S)^2$ that is accurate within $\pm \varepsilon$ with probability $1-\delta$, via standard Chernoff bound arguments.

Interestingly, Claim~2 in the same Rahimi-Recht work \cite{rahimi2007kernel} authors provide a different randomized low-rank approximation to the kernel matrix $M$, the \emph{Random Binning Featues} algorithm.
It is insightful to write down explicitly our kernel calibration estimator with this
different low-rank approximation.
As it turns out, by unrolling the abstraction, we can uncover a new and ``natural'' unbiased and bounded estimate of $\ekCE(S)^2$, described in \Cref{alg:binning-laplace-kernel}.
The final algorithm is similar in spirit to the binning algorithms estimating $\intCE$ (see Section~\ref{sec:interval-calibration}),
but here we are using both random bin size (from a carefully chosen distribution) and random shifts.

\begin{algorithm}[t]
\caption{Binning algorithm for Laplace kernel calibration error estimation \label{alg:binning-laplace-kernel}}
\begin{algorithmic}[1]
\Function{LaplaceBinningEstimate}{$(v_1, y_1), \ldots (v_k, y_k)$}
\Let{$\delta$}{$\mathrm{Gamma}(k=2, \theta=1)$}
\Let{$\tau$}{$\mathrm{Unif}(0, \delta)$}
\Let{$B[0, \ldots, \lfloor 1/\delta \rfloor + 2]$}{0}
\For{$i \gets 1 \textrm{ to } k$}
    \Let{$t$}{$\lfloor{\frac{v_i + \tau}{\delta}}\rfloor$}
    \Let{$B[t]$}{$B[t] + (v_i - y_i)$}
\EndFor
\State \Return{$\sum_i B[i]^2$}
\EndFunction
\end{algorithmic}
\end{algorithm}

The following theorem is essentially implicit in \cite{rahimi2007kernel} and follows directly from their proof method. Since this exact statement is not present in their paper, we repeat the proof in Appendix~\ref{sec:app-ker-proofs}.
\begin{theorem}
\label{thm:laplace-binning}
The algorithm \textsc{LaplaceBinningEstimate} provides an unbiased and bounded by $1$ estimate for $\ekCE(S)^2$
and runs on $n$ samples in time $\Oh(n)$.
\end{theorem}

\subsection{A Practical Note}
\label{sec:practical-laplace}
For practitioners measuring Laplace kernel calibration,
we emphasize that while it requires $\Oh(n^2)$
time to compute \emph{exactly} from $n$ samples,
it can be computed \emph{approximately} much more efficiently.
Specifically, for a sample of prediction-label pairs
$S=((v_1, y_1), (v_2, y_2), \ldots, (v_k, y_k))$,
the exact kernel error computation is given by Equation~\eqref{eqn:kernel-emperical}, reproduced below
for the Laplace kernel:
\begin{equation}
\label{eqn:laplace-practical}
\kCE^\lap(\Gamma)^2 = \frac{1}{n^2}\sum_{i,j} (y_i - v_i)(y_j - v_j) \exp(-|v_i - v_j|)
\end{equation}
This is an average of $n^2$ terms.
To approximate this, we can instead average $M$ of these terms,
chosen independently at random with replacement. This will yield an estimate
of $(\kCE_K)^2$ accurate to within
$\pm \widetilde{\Oh}(1/\sqrt{M})$.
For example, a reasonable setting of parameters is
$M=10n$, which yields a linear-time estimator
that is accurate to within $\pm \widetilde{\Oh}(1/\sqrt{n})$.
Note, finally, that the above computation computes the \emph{square}
of the kernel calibration error $\kCE$.

\section{Estimating Calibration Measures Using Random Sample}
\label{sec:sample-complexity}
When the distribution $\Gamma$ is the uniform distribution over data points $(v_1,y_1),\ldots,(v_n,y_n)\in [0,1] \times \{0,1\}$, we can compute $\ldCE(\Gamma)$, $\scerror(\Gamma)$, and $\kCE^K(\Gamma)$ efficiently by \Cref{remark:dual-2}, \Cref{thm:smooth-lp}, and \Cref{sec:alternate,sec:practical-laplace}. In this section, we show that we can efficiently estimate these quantities for general $\Gamma$ using i.i.d.\ data drawn from $\Gamma$.
\subsection{Estimating $\wCE,\kCE$ and $\scerror$}
In this subsection we prove bounds on the number of samples from the distribution $\Gamma$ that can be used in order to estimate the $\wCE^W(\Gamma)$ in terms of the \emph{Rademacher complexity} of a function class $W$. Using known bounds on the Rademacher complexity of the class of all $1$-Lipschitz functions, we prove that $\scerror(f)$ can be efficiently estimated using $\Oh(\varepsilon^{-2} \log \delta^{-1})$ samples, in polynomial time.
Together with known bounds on the Rademacher complexity of a unit balls in RKHS associated with the kernel $K$, we also give an alternative prove of the result in~\cite{ksj18a} that the $\kCE$ can be estimated efficiently up to an error $\varepsilon$ using $\Oh(\varepsilon^{-2} \log \delta^{-1})$ samples. 

\begin{definition}[Empirical weighted calibration error]
For a sample $S = \{(v_1, y_1), \ldots (v_k, y_k)\} \subset [0, 1]\times \{0,1\}$ we define the \emph{empirical weighted calibration error} of $S$ with respect to the family $W$ as
\begin{equation*}
    \ewCE^W(S) :=  \sup_{w \in W} \frac{1}{k} \sum_i (y_i - v_i) w(v_i).
\end{equation*}

Note that this definition coincides with $\wCE$ applied to the empirical distribution over the sample $S$.
\end{definition}

We will show that for function classes with small Rademacher complexity, the empirical weighted calibration error is with high probability close to the weighted calibration error.

Let us first briefly reintroduce the relevant notions in the theory of Rademacher complexity. We refer interested reader to~\cite{mohri2018foundations} for more detailed exposition.

\begin{definition}[Rademacher complexity]
For a set $A \subset \R^k$ we define its Rademacher complexity $\Rad(A)$ as
\begin{equation*}
    \Rad(A) := \E_\sigma \lt[\sup_{a \in A} \sum_{i =1}^n \sigma_i a_i\rt],
\end{equation*}
where the expectation is taken over $\sigma_i \sim \{\pm 1\}$ independent Rademacher random variables.

For a function class $\mathcal{F}$ from $\mathcal{X} \to \R$ we define the Rademacher complexity of $\mathcal{F}$ with respect to the sample $S \in \mathcal{X}^n$, $S = (s_1, s_2, \ldots s_n)$ as
\begin{equation*}
\Rad_S(\mathcal{F}) := \Rad(\{ (f(s_1), f(s_2), \ldots, f(s_n)) : f \in \mathcal{F}\}).
\end{equation*}

Finally, for a function class $\mathcal{F}$ and a distribution $\mathcal{D}$ over $\mathcal{X}$ we define the Rademacher complexity of the function class $\mathcal{F}$ with respect to distribution $\mathcal{D}$ with sample size $n$ as
\begin{equation*}
    \Rad_{\mathcal{D}, n}(\mathcal{F}) = \E_{S \sim \mathcal{D}^n} \Rad_S(\mathcal{F}).
\end{equation*}
\end{definition}
In what follows we will skip the subscript $\mathcal{D}$ to simplify the notation.

The main application of the notion of Rademacher complexity is the following bound on the \emph{generalization error} --- this theorem will be crucial in our proof that for function classes $W$ with small Rademacher complexity, we can estimate $\wCE_W(\Gamma)$ using small number of samples.
\begin{theorem}\cite{mohri2018foundations}
\label{thm:rademacher-generalization}
For a random sample $S \sim \mathcal{D}^n$, with probability at least $1-\delta$ we have
\begin{equation*}
    \sup_{w \in W} \left|\E_{x \sim \mathcal{D}} w(x) - \frac{1}{n} \sum_i w(s_i)\right| \leq \Rad_{\mathcal{D}, n}(W) + \Oh\left(\sqrt{\frac{\log \delta^{-1}}{n}}\right).
\end{equation*}
\end{theorem}

We need also the following technical statement about Rademacher complexity.
\begin{theorem}[\cite{meir2003generalization}]]
\label{thm:rademacher-contraction}
Let $T : \R^n \to \R^n$ be a contraction (i.e. $\|T(x) - T(y)\|_2 \leq \|x-y\|_2$ for all $x,y \in \R^n$). Then
\begin{equation*}
    \Rad(T(A)) \leq \Rad(A).
\end{equation*}
\end{theorem}

With these results in hand we will prove the followin theorem.
\begin{theorem}
\label{thm:sample-complexity}
Let $W$ be a family of functions from $[0,1] \to \R$ with bounded Rademacher complexity.
Let $\Gamma$ be a distribution over $[0,1]\times\{0,1\}$.
Then with probability $1-\delta$ over a random sample $S \sim \Gamma^n$, we have
\begin{equation*}
    |\wCE^W(\Gamma) - \ewCE^W(S)| \leq \Rad_n(W) + \Oh\lt(\sqrt{\frac{\log \delta^{-1}}{n}}\rt), 
\end{equation*}
where the Rademacher complexity $\Rad_n(W)$ is  with respect to the marginal distribution of $v$, with $(v, y) \sim \Gamma$.
\end{theorem}

In particular if $\Rad_n(W) \leq \sqrt{n^{-1} R_0}$, it is possible to estimate $\wCE(\Gamma)$ up to additive error $\varepsilon$ with probability $1-\delta$, using
$\Oh(\varepsilon^{-2}(R_0 + \log \delta^{-1}))$ samples.

This theorem follows almost immediately from the standard generalizations bounds using Rademacher complexity (\Cref{thm:rademacher-generalization}), the only minor technical step is showing that if family of functions $W$ has small Rademacher complexity, the same is true for a family $\hat{W}$ of functions of form $\hat{w}(v,y) := w(v)(y-v)$, used in the definition of $\wCE$.

\begin{lemma}
\label{lem:rademacher-complexity-transfer}
Let $W$ be a family of functions from $[0,1]$ to $\R$, and let $\hat{W}$ be a family of functions $[0,1] \times \{0,1\} \to \R$, consisting of functions $\hat{w}(v, y) = w(v)(y-v)$ for each $w \in W$.

Then for a distribution $\Gamma$ over $[0,1] \times \{0,1\}$, such that the marginal distribution of the projection onto the first coordinate is $\Gamma_1$, we have
\begin{equation}
    \Rad_{\Gamma, n}(\hat{W}) \leq \Rad_{\Gamma_1, n}(W).
\end{equation}
\end{lemma}
\begin{proof}
It is enough to prove the corresponding inequality for a fixed sample $S = \{(v_1, y_1), \ldots, (v_n, y_n)\}$.  We wish to show that
\begin{equation}
    \sup_{w} \E_{\sigma} \sum_i \sigma_i (v_i - y_i) w(v_i) \leq \sup_w \E_{\sigma} \sum_i \sigma_i w(v_i). \label{eq:rad-A}
\end{equation}

Indeed, since the sample $S$ is fixed, we can consider a set $A \subset \R^n$ given by $A = \{ (w(v_1), \ldots, w(v_n)) : w \in W\}$, and note that a map $T : \R^n \to \R^n$ which maps a vector $(w_1, \ldots w_n) \mapsto ((v_1  - y_1) w_1, \ldots (v_n - y_n) w_n)$ is a linear contraction (i.e. for any $w, w'$ $\|T(w) - T(w')\| \leq \|w - w'\|$), since $\forall_i |v_i - y_i| \leq 1$.

By \Cref{thm:rademacher-contraction}, the Rademacher complexity of any set $A$ does not increase under contraction, i.e. $\Rad(T(A)) \leq \Rad(A)$, which, by definition of $\Rad(A)$, $T$ and $A$ is exactly~\eqref{eq:rad-A}.
\end{proof}

With this lemma in hand the proof of \Cref{thm:sample-complexity} follows directly from the generalization bound (\Cref{thm:rademacher-generalization})
\begin{proof}[Proof of \Cref{thm:sample-complexity}]
Applying \Cref{thm:rademacher-generalization}, to the family $\hat{W}$ defined in \Cref{lem:rademacher-complexity-transfer}, we see that
\begin{equation*}
    |\wCE^W(f) - \ewCE^W(S)| = \sup_{\hat{w}} |\frac{1}{n} \sum \hat{w}(v_i,y_i) - \E \hat{w}(v,y)|  \leq \Rad_n(\hat{W}) + \Oh(\sqrt{\frac{\log \delta^{-1}}{n}}), 
\end{equation*}
since by \Cref{lem:rademacher-complexity-transfer}, we have $\Rad_n(\hat{W}) \leq \Rad_n(W)$, the statement of the theorem follows.
\end{proof}

Finally, for classes of functions given by the unit ball in some RKHS associated with the kernel $K$, the Rademacher complexity has been explicitly bounded.
\begin{theorem}\cite{mendelson2002geometric}]
\label{thm:kernel-rademacher-complexity}
    Let $K : \R \times \R \to \R$ be a kernel associated with RKHS $\cH$. Let $B_K = \{ h : \|h\|_{\cH} \leq 1\}$. Then for any distribution
    $\Gamma$. we have 
    \[ \Rad_n(B_K) \leq \Oh(\frac{B}{\sqrt{n}}) \  \text{where}  \ B = \sup_{v \in \R} \sqrt{K(v,v)}.\]
\end{theorem}

In particular this together with Theorem~\ref{thm:sample-complexity} implies the following bound on the number of samples needed to estimate the $\kCE^K(\Gamma)$ for any kernel $K$ with bounded $\sup \sqrt{K(v,v)} \leq B$ --- this Theorem was already proven directly in~\cite{ksj18a}, we provide an alternative proof  by composing \Cref{thm:sample-complexity} and \Cref{thm:kernel-rademacher-complexity}.
\begin{theorem}
    Let $K : \R \times \R \to \R$ be a kernel associated with RKHS $\cH$. Let $B_K = \{ h : \|h\|_{\cH} \leq 1\}$. Then for any distribution $\Gamma$ we can estimate the $\kCE^K(\Gamma)$ with probability at least $1-\delta$, with additive error at most $\varepsilon$ using $n = \Oh((B^2 + \log(1/\delta))\varepsilon^{-2})$ samples.
\end{theorem}

Since the class of $1$-Lipschitz functions also has bounded Rademacher complexity, it follows that we can estimate $\scerror(\Gamma)$ using $n = \Oh(\ln(1/\delta)/\varepsilon^{2})$ samples in time $\mathrm{poly}(n)$.

\begin{corollary}
    \label{cor:smCE-generalization}
    We can estimate $\scerror(\Gamma)$ using $n = \Oh((\ln (1/\delta))/\varepsilon^2)$ samples from the distribution $D_f$ in time $\mathrm{poly}(n)$.
\end{corollary}
\begin{proof}
The class of $[-1,1]$ bounded $1$-Lipschitz functions $L_1$ has bounded Rademacher complexity $\Rad_n(L_1) \leq \Oh(1/\sqrt{n})$ \cite{ambroladze2004complexity}. Therefore, by \Cref{thm:sample-complexity}, with probability $1-\delta$, for a random sample of size $n = C \log(1/\delta)/{\varepsilon^2}$, we have $\wCE^{L_1}(\Gamma) = \ewCE^{L_1}(S) \pm \varepsilon$.

The claim about efficient computation follows from~\Cref{thm:smooth-lp}.
\end{proof}

\subsection{Estimating $\ldCE$}
\label{sec:est-ldce}
We will prove that using $\Oh(\varepsilon^{-2})$ samples from the distribution $\Gamma$, we can estimate $\ldCE(\Gamma)$ up to an error $\varepsilon$ with probability $2/3$. The efficient algorithm for estimating $\ldCE$ of an empirical distribution on the sample was described in \Cref{sec:smooth-ce-duality}, here we will show that $\ldCE(f)$ on the empirical distribution over a sample $S$ drawn independently from distribution $\cD_f$ approximates the $\ldCE$ of the distribution $\Gamma$.

\begin{theorem}
    \label{thm:ldce-sample-complexity}
    Let $S = ((v_1, y_1), \ldots (v_n, y_n))$ be a sample drawn from $\Gamma^n$, where $n = \Omega(\varepsilon^{-2})$. Then with probability $2/3$ we have
    \begin{equation*}
        \ldCE(S) = \ldCE(\Gamma) \pm \varepsilon,
    \end{equation*}
    where $\ldCE(S)$ is the $\ldCE$ for the empirical distribution over the sample $S$.
\end{theorem}

In order to prove this, we will need the following lemma, where we use $W_1(P_1,P_2)$ to denote the Wasserstein $1$-distance of two distributions $P_1,P_2$ over $[0,1]$:
\[
W_1(P_1,P_2) = \sup_{f}\left(\E_{v\sim P_1}f(v) - \E_{v\sim P_2}f(v)\right),
\]
where the supremum is over all $1$-Lipschitz functions $f:[0,1]\to \R$.
\begin{lemma}
\label{lem:sample-ldce}
For two distributions $\Gamma^1, \Gamma^2$ over $[0,1] \times \{0,1\}$, let us denote $\lambda_b := \Pr_{\Gamma^2}(y=b)$. If the following conditions are satisfied, 
\begin{enumerate}
    \item\label{item:sample-ldce-1} $W_1( (v | y = b)_{\Gamma^1}, (v | y=b)_{\Gamma^2}) \leq \frac{\varepsilon}{\lambda_b}$ for $b\in\{0,1\}$,
    \item\label{item:sample-ldce-2} $|\Pr_{\Gamma^1}(y = 1) - \lambda_1| \leq \varepsilon$,
\end{enumerate}
then we have
\begin{equation*}
    \ldCE(\Gamma^1) = \ldCE(\Gamma^2) \pm \Oh(\varepsilon).
\end{equation*}
\end{lemma}
\begin{proof}
Similarly to the proof of \Cref{thm:smce-equals-ldce} in \Cref{sec:ldce-lp}, by discretization it suffices to consider the case where $\Gamma^1$ and $\Gamma^2$ are both distributions over $V\times \{0,1\}$ for a finite $V\subseteq[0,1]$ and show that for every finite $U\subseteq[0,1]$ satisfying $\{0,1\}\subseteq U$, it holds that
\[
\ldCE^U(\Gamma^1) = \ldCE^U(\Gamma^2) \pm \Oh(\varepsilon).
\]
By \Cref{remark:dual-1}, it suffices to show that for every
function $r:V\times \{0,1\}\to [-1,1]$ satisfying $|r(v_1,y) - r(v_2,y)| \le |v_1 - v_2|$ for every $v_1,v_2\in V$ and $y\in \{0,1\}$, it holds that
\begin{equation}
\label{eq:sample-ldce-goal}
|\E_{(v,y)\sim \Gamma_1}[r(v,y)] - \E_{(v,y)\sim \Gamma_2}[r(v,y)]| \le \Oh(\varepsilon).
\end{equation}
By our assumption in \Cref{item:sample-ldce-1} and the definition of $W_1$, we have
\begin{equation}
\label{eq:sample-ldce-1}
|\E_{(v,y)\sim \Gamma_1}[r(v,y)|y = b] - \E_{(v,y)\sim \Gamma_2}[r(v,y)|y = b)]| \le \varepsilon / \lambda_b.
\end{equation}
Therefore,
\begin{align*}
& |\E_{(v,y)\sim \Gamma_1}[r(v,y),y = b] - \E_{(v,y)\sim \Gamma_2}[r(v,y),y = b]|\\
= {} & |\E_{\Gamma_1}[r(v,y)|y = b]\Pr_{\Gamma_1}[y = b] - \E_{\Gamma_2}[r(v,y)|y = b]\Pr_{\Gamma_2}[y = b]|\\
\le {} & \left|\E_{\Gamma_1}[r(v,y)|y = b] - \E_{\Gamma_2}[r(v,y)|y = b]\right|\Pr_{\Gamma_2}[y = b]\\
& + \E_{\Gamma_1}[r(v,y)|y = b]\left|\Pr_{\Gamma_1}[y = b] - \Pr_{\Gamma_2}[y = b]\right|\\
\le {} & (\varepsilon/\lambda_b)\lambda_b + \varepsilon\tag{By \eqref{eq:sample-ldce-1} and \Cref{item:sample-ldce-2}}\\
= {} & \Oh(\varepsilon).
\end{align*}
Summing up over $b = 0,1$ proves \eqref{eq:sample-ldce-goal}.
\end{proof}

In order to prove that the conditions of the lemma above are indeed satisfied, where $\Gamma^2$ is an empirical distribution over a sample from $\Gamma^1$, we will need to following concentration theorem.
\begin{theorem}[\cite{fournier2015rate}]
\label{thm:wasserstein-concentration}
Let $\mu$ be a distribution over $[0,1]$. If we sample $n = \Omega(\varepsilon^{-2} \log \delta^{-1})$ points from the distribution $(x_1, x_2, \ldots x_n) \sim \mu^n$, then the empirical distribution $\hat{\mu}_S$ on the sample $S$ satisfies
\begin{equation*}
    \Pr(W_1(\hat{\mu}_S, \mu) > \varepsilon) \leq \delta.
\end{equation*}
\end{theorem}

We are now ready to prove the following lemma.
\begin{lemma}
\label{lem:wasserstein-sample}
If we draw $n = \Omega(\varepsilon^{-2})$ samples from a distribution $\Gamma^1$, and denote by $\Gamma^2$ the empirical distribution on the sample $S = ((v_1, y_1), \ldots, (v_n, y_n))$, then with probability at least $2/3$ the pair of distributions $(\Gamma^1, \Gamma^2)$ satisfies the conditions of \Cref{lem:sample-ldce}.
\end{lemma}
\begin{proof}
    First of all, except with probability $1/10$ we have $|\Pr_{\Gamma^1}(y=1) - \Pr_{\Gamma^2}(y=1)| \le \varepsilon$ by the standard Chernoff bound argument.
    
    If this is the case, for $b \in \{0, 1\}$ we get $n_b$ independent samples from the conditional distribution $(v | y=b)_{\Gamma^1}$, where
    \begin{equation*}
        n_b = \Pr_{\Gamma^2}(y = b) n = \Omega\left(\Pr_{\Gamma^2}(y=b) \varepsilon^{-2}\right).
    \end{equation*}
    Hence, by \Cref{thm:wasserstein-concentration}, with sufficiently high probability, for $b\in \{0,1\}$,
    \begin{equation*}
    W_1( (v | y= b)_{\Gamma^1}, (v | y=b)_{\Gamma^2}) \leq \Oh\left(\frac{1}{\sqrt{n_b}}\right) \leq \Oh\left(\frac{\varepsilon}{\sqrt{\Pr_{\Gamma^2}(y=b)}}\right)\le \Oh\left(\frac{\varepsilon}{\Pr_{\Gamma^2}(y = b)}\right).\qedhere
    \end{equation*}
\end{proof}

\begin{proof}[Proof of \Cref{thm:ldce-sample-complexity}]
Follows by composition of \Cref{lem:sample-ldce} and \Cref{lem:wasserstein-sample}.    
\end{proof}

\section{Experiments}
\label{sec:experimental}

We now conduct several experiments evaluating the calibration measures discussed in this work
on synthetic datasets.
The purpose of this is twofold:
First, to explore how our measures behave (and compare to each other)
on reasonable families of distributions, beyond our worst-case theoretical bounds.
And second, to demonstrate that these measures can be computed efficiently, some even in linear time.

\subsection{Synthetic Datasets}
\begin{figure}[t]
\centering
\includegraphics[width=\linewidth]{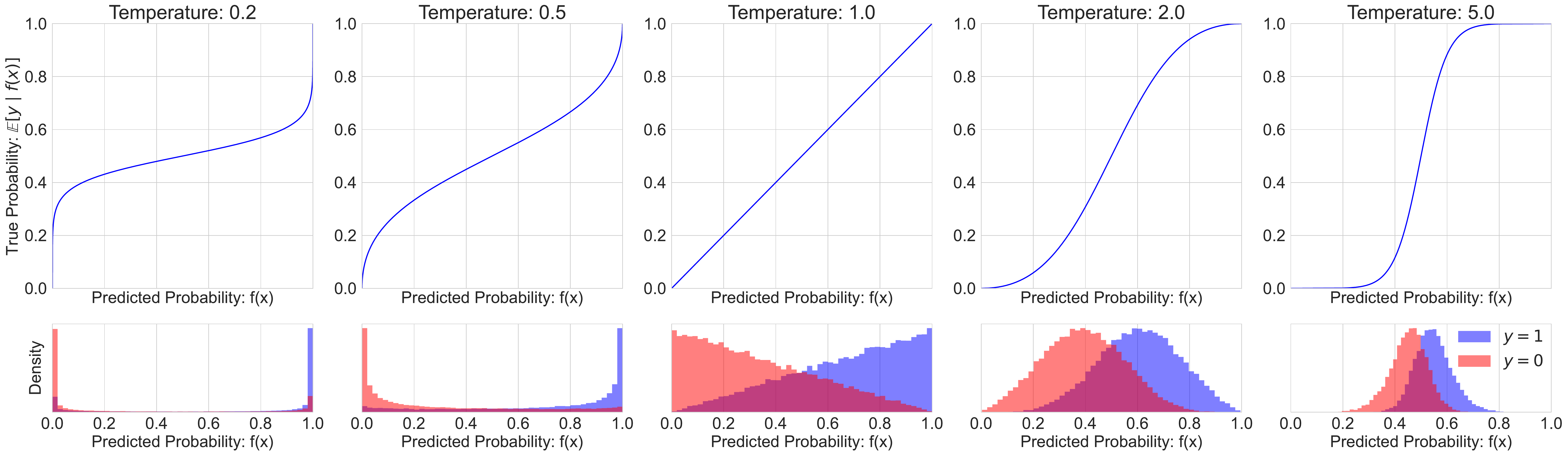}
\caption{{\bf Synthetic Dataset Family.} Reliability diagrams for the family of synthetic distributions
$\mD_\beta$, as the inverse-temperature $\beta$ is varied.
The center plot $\mD_1$, at unit temperature, is perfectly calibrated.
Temperatures greater than $1$ are underconfident, while temperatures less than $1$ are overconfident.
The bottom row shows density plots for the conditional distributions
$p(f \mid y = 0)$ and $p(f \mid y = 1)$.
}
\label{fig:temp-scaling}
\end{figure}

We consider evaluating calibration on the following family of distributions.
Define the ``baseline'' distribution $\mD_1$ as the following
joint distribution $(f, y) \sim \mD_1$:
\begin{align*}
&f \sim \textrm{Unif}\lbrack 0, 1\rbrack\\
&y \sim \textrm{Bernoulli}(f)
\end{align*}
This is a perfectly calibrated distribution,
where predictions $f$ are uniform over $\lbrack 0, 1 \rbrack$.
Now consider modifying the predicted outputs, by changing their ``temperature.''
Formally, for any inverse-temperature parameter $\beta = 1/T \in \R_{> 0}$,
define the distribution $\mD_\beta$ as:
\begin{equation}
\mD_\beta := \left\{\left(\frac{f^\beta}{f^\beta + (1-f)^\beta}~,~ y \right)\right\}_{(f, y) \sim \mD_1}
\end{equation}
If the original predictions in $\mD_1$ were the output of a softmax, the distribution $\mD_\beta$
corresponds exactly to changing the softmax temperature to $(1/\beta)$.
For small temperatures, this pushes predictions towards the endpoints of the interval $\lbrack 0, 1 \rbrack$, yielding an overconfident classifier.
And for large temperatures, this pushes predictions towards $0.5$, yielding an underconfident classifier.
This family of distributions $\{\mD_\beta\}$ is illustrated in Figure~\ref{fig:temp-scaling}.

\subsection{Implementation Details}
We will evaluate calibration measures on the family of distributions $\{\mD_\beta\}$.
We draw $N=10000$ samples from each distribution $\mD_\beta$, and evaluate
the calibration measures as described below.
All of the measures here are computed in linear time $\Oh(N)$, except for smooth calibration which requires
solving a linear program.

\paragraph{binnedECE.} This is the standard measurement of binnedECE (Equation~\ref{eqn:binnedECE}),
using 20 equal-width bins.

\paragraph{binnedECEw.} This is simply binnedECE plus the average bin width (which is $1/20$ in our case).
As noted in Equation~\eqref{eqn:binnedECEw}, adding the bin width penalty turns binnedECE into a provable upper-bound on $\dCE$.

\paragraph{kCE.} Kernel calibration error using the Laplace kernel. We compute the approximate kernel calibration using $M=10N$ terms,
as described in Section~\ref{sec:practical-laplace}. This runs in $\Oh(N)$ time.

\paragraph{intCE.} Interval calibration error. We approximate this via the
surrogate described in Section~\ref{sec:surrogate}.
Specifically, we compute $\wh\SintCE$ with choice of precision $\eps=0.01$.

\paragraph{smCE.} Smooth calibration error. We compute this from samples via the
linear program described in \Cref{thm:smooth-lp}.

\subsection{Evaluation and Discussion}

\begin{figure}[p]
\centering
\includegraphics[width=\linewidth]{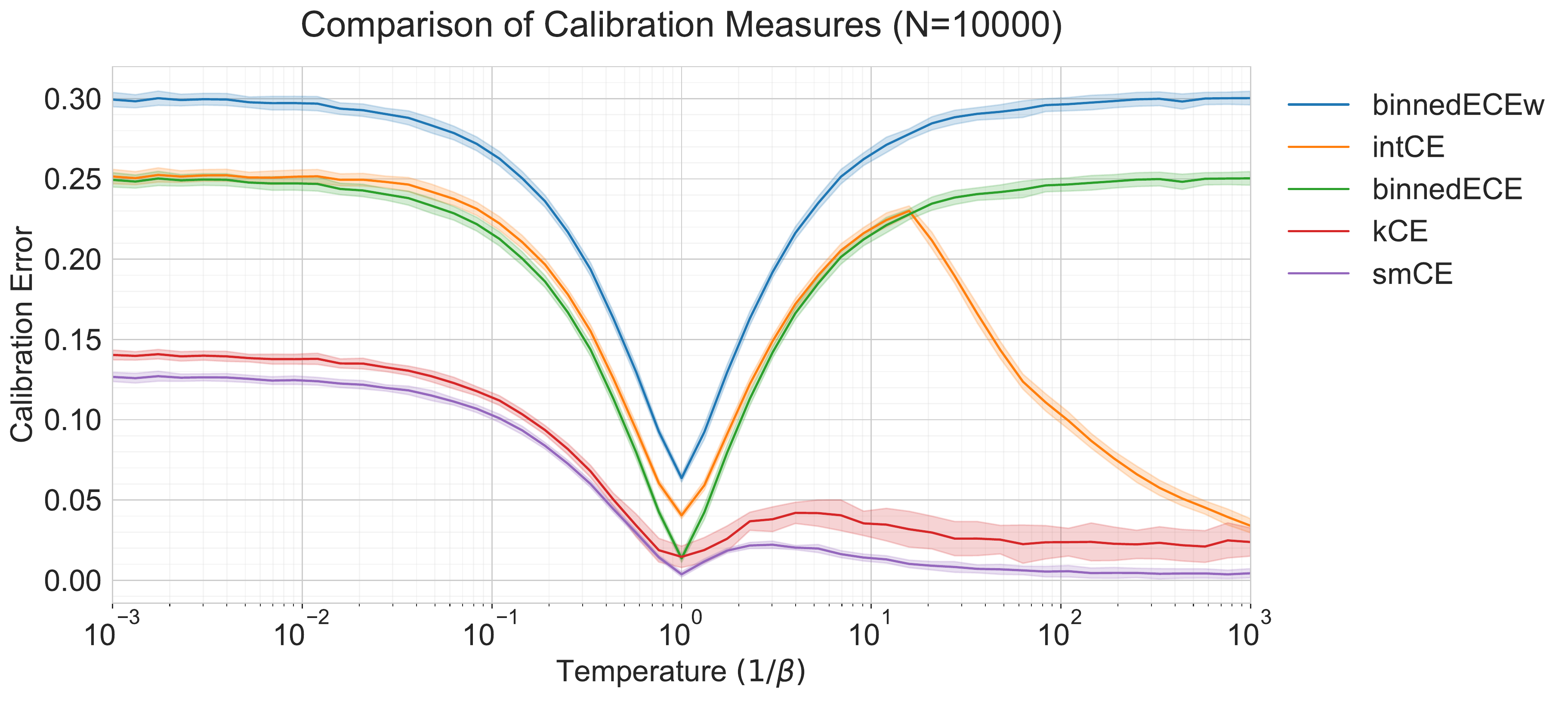}\\
\vspace{2cm}
\includegraphics[width=\linewidth]{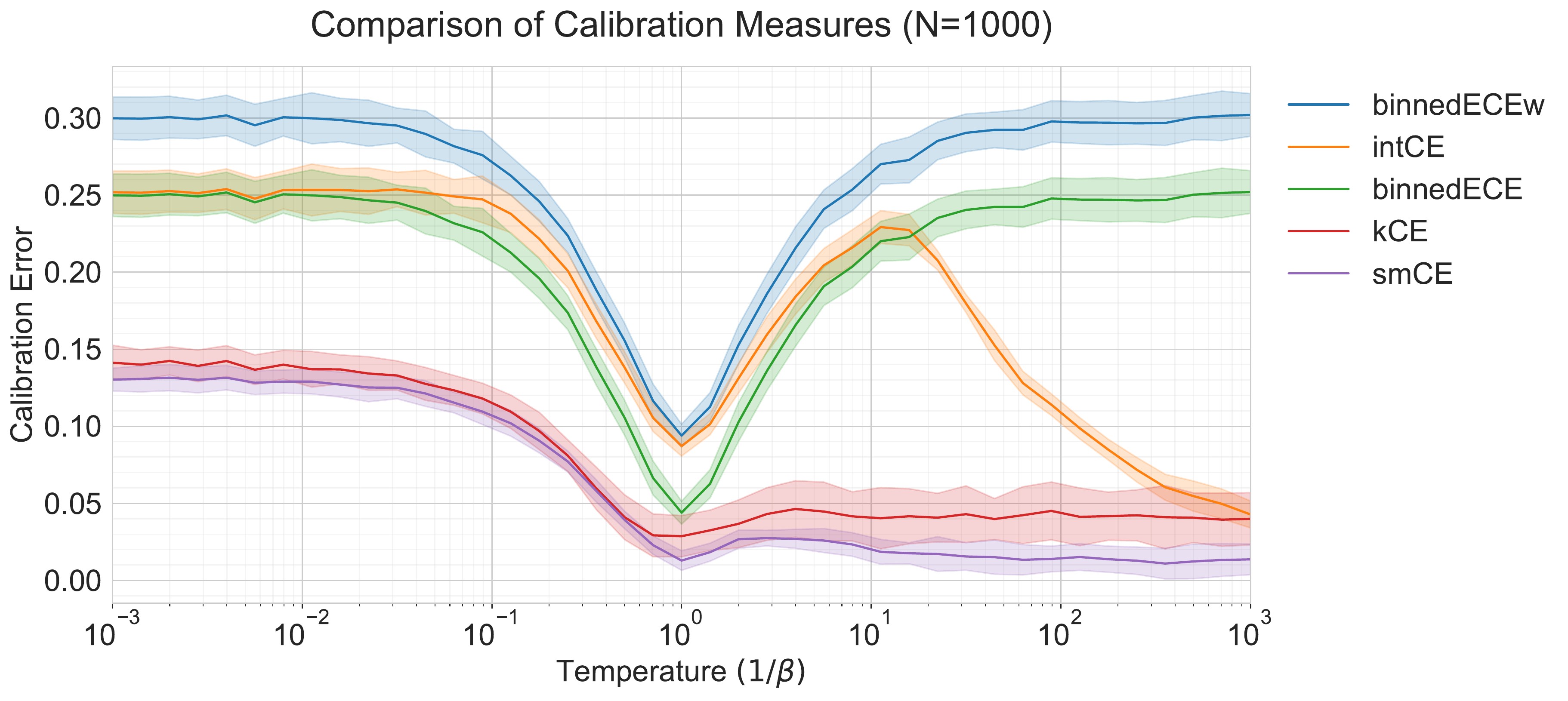}
\caption{{\bf Experimental Evaluation.} 
We experimentally evaluate calibration measures on the family of distributions $\{\mD_\beta\}$
of varying temperature.
Measures are computed over the empirical distribution of
$N=10000$ independent samples (top figure), and $N=1000$ samples (bottom figure).
We plot mean and standard deviations over $50$ independent trials.
Note that these measures concentrate well around their mean, even for $N=1000$ samples.
}
\label{fig:comparison}
\end{figure}

In Figure~\ref{fig:comparison}, we numerically evaluate calibration measures as we vary the temperature
of the distribution $\mD_\beta$.
The true calibration distance $\dCE$ lies between $\scerror$ and $\intCE$, as per Equation~\eqref{eqn:main-ineq}, but its true value
cannot be determined in the prediction-access model.

We observe that while all metrics are correlated at temperatures $T \leq 1$ (when the classifier is overconfident), the behavior at $T \gg 1$ is more subtle.
At moderately large temperatures ($T \in [1,10]$),
$\{\intCE, \mathsf{binnedECE}\}$ 
are much higher than $\{\kCE,\scerror\}$.
This is because as predictions concentrate
around $0.5$, the distribution $\mD_\beta$ becomes
similar to the construction of \Cref{lem:tight-gap},
which exhibits a quadratic gap between $\intCE$ and $\scerror$.
As the temperature $T\to \infty$, the true calibration distance $\dCE \to 0$,
and so all consistent measures ($\kCE, \scerror, \intCE$) will also go to $0$.
However $\mathsf{binnedECE}$, which is not consistent, does not go to $0$.

Finally, we note that $\scerror$ is numerically close to $\kCE$
for the range of distributions tested here.

\subsubsection*{Acknowledgements}
Part of this work was performed while LH was interning at Apple. LH is also supported by Moses Charikar’s and Omer Reingold’s Simons Investigators awards, Omer Reingold’s NSF Award IIS-1908774, and the Simons Foundation Collaboration on the Theory of Algorithmic Fairness.

JB is supported by a Junior Fellowship from the Simons Society of Fellows.

PN and JB acknowledge the city of New Orleans,
for providing an environment conducive to both
research and recreation
in the early stages of this project.
PN also acknowledges his partner
Shreya Shankar for their invaluable support.

\bibliographystyle{alpha}
\bibliography{refs}
\addtocontents{toc}{\protect\setcounter{tocdepth}{1}}
\appendix

\section{Other Related Works}
\label{app:related}

\paragraph{Boolean function complexity.} A number of low-level complexity measures for Boolean functions have been studied in the literature \cite{BuhrmanW02}. Starting from the seminal work of Nisan and Szegedy \cite{NisanS94}, it has been shown that several of these measures are in fact polynomially related; the latest addition being sensitivity\cite{Huang19}. One can view polynomial degree or decision tree depth as the central notion in this web of reductions. Our work suggests  a similar structure in the world of calibration measures, where all consistent calibration measures for a particular notion of ground truth distance are polynomially related. We show that the $\ell_p$ metrics give rise to one family of consistent measures, and our framework allows for other families which are consistent with  other metrics. 

\paragraph{ECE Variants.}
Many metrics are essentially variants of the $\ECE$, and also inherit flaws of the $\ECE$.
We use the term $\ECE$ to refer to Definition~\ref{def:ece}, which does not satisfy continuity.
\emph{Binned-ECE} refers to the family of metrics which
discretize the range of $f$ into a constant number of bins,
and then compute ECE of this rounded function. Binned-ECE, and in general most ``binned'' estimators,
do not satisfy soundness: some functions which are not perfectly calibrated will have Binned-ECE of 0.
This is intuitively because binning allows calibration error to cancel-out within bins.
Moreover, Binned-ECE is not continuous, due to discontinuities at bin boundaries.
Among proposed modifications to the Binned-ECE, ``debiasing'' adjustments
(such as \citep{roelofs2022mitigating,kumar2019verified}) unfortunately remain discontinuous and unsound.
To enforce continuity, a smoothed version of Binned-ECE was proposed in \citep{karandikar2021soft},
but it is still unsound due to the finite number of (smoothed-)bins.
Finally, the recently-introduced T-Cal \citep{lee2022t} is an estimator of the population ECE
that applies to distributions satisfying certain smoothness conditions.
Such smoothness assumptions are necessary in T-Cal, since without them, the ECE is both discontinuous
and impossible to estimate from finite samples.

\paragraph{Continuous Calibration}
The notion of \emph{Continuous Calibration} was introduced in \cite{foster2021forecast} --- the continuous calibration is not directly a calibration measure in our sense, instead it is a way of assigning for a sequence of probability distributions $\Gamma_t$ over $[0,1] \times \{0, 1\}$, whether $\Gamma_t$ is asymptotically calibrated. Their notion of continuous calibration of a sequence $\Gamma_t$ turns out to be equivalent to saying for any consistent calibration measure $\mu$ that as $t\to \infty$ the $\mu(\Gamma_t) \to 0$. This is equivalent to saying that $\mu(\Gamma_t) \to 0$ for \emph{all} consistent calibration measures $\mu$, since they are all polynomially related).

On the other hand, the notion of \emph{Continuous Calibration} cannot be used to say anything quantitative about the miscalibration of one given classifier.

\section{Proofs \label{sec:proofs}}

\subsection{Proofs from Section \ref{sec:framework}}
\label{app:distance}

\begin{proof}[Proof of Lemma \ref{lem:best}]
    Consider any calibration measure $\mu$ as in the statement. For any $g \in \calD$, 
    \begin{align*}
        \mu_\mD(f) = \mu_\mD(f) - \mu_\mD(g) \leq L\ m_{\mD}(f, g)
    \end{align*}
    where the equality is because $\mu_\mD(g) = 0$ by Soundness, and the inequality by Lipschitz continuity. Hence 
    \begin{align}
        \label{eq:L-lip}
     \mu_\mD(f) \leq \min_{g \in \calD}L \ m_\mD(f,g) = L\ \dCE_\mD(f).
     \end{align}
\end{proof}

\begin{proof}[Proof of Lemma \ref{lem:all-lp}]
Since for every $\mD$ we have the relation 
\[ (\ell_p^\cD(f,g))^p \le \ell_1^{\mD}(f,g) \le \ell_{p}^\mD(f,g) \]
it follows that
\[ (\dCE^{\ell_p}_{\mD}(f))^p \le \dCE^{\ell_1}_{\mD}(f) \le \dCE^{\ell_p}_\mD(f,g) \]
Assume that $\mu$ is $(\ell_1, c, s)$-robust. This implies 
\[b(\dCE^{\ell_p}_\mD(f))^{ps} \leq  b(\dCE^{\ell_1}_\mD(f))^s \leq \mu(f) \leq a(\dCE^{\ell_1}_\mD(f))^c \leq a(\dCE^{\ell_p}_{\mD}(f))^c\]
hence $\mu$ is $(\ell_p, c, ps)$-robust. Conversely, if $\mu$ is $(\ell_p, c', s')$ robust then 
\[b'(\dCE^{\ell_1}_\mD(f))^{s'} \leq  b'(\dCE^{\ell_p}_\mD(f))^{s'} \leq \mu(f) \leq a'(\dCE^{\ell_p}_\mD(f))^{c'} \leq a'(\dCE^{\ell_p}_{\mD}(f))^{'c/p}\]
hence $\mu$ is $(\ell_1, c'/p, s')$-robust. Note that in either case, the translation loses a factor of $p$ in the approximation degree.
\end{proof}

\subsection{Proofs from \Cref{sec:int-ce}}
\label{sec:proof-pa}
We first prove \Cref{thm:estimate-intce} using \Cref{claim:int-monotone} and \Cref{lm:intCE-uniform-conv} below. We prove \Cref{lm:int-discontinuous} after that.
\begin{claim}
\label{claim:int-monotone}
$\RintCE(\Gamma,2\varepsilon) \le \RintCE (\Gamma,\varepsilon)$.
\end{claim}
\begin{proof}
The claim is proved by the following chain, where the distribution of $r$ is the uniform distribution over $[0,2\varepsilon)$.
\begin{align*}
\RintCE(\Gamma,\varepsilon) & = \E_{r}\Big[\sum_{j\in \Z}|\E_{(v,y)\sim\Gamma}[(y - v)\one(v\in I_{r,j}^\varepsilon)]|\Big]\\
& = \E_{r}\Big[\sum_{j\in \Z}|\E_{(v,y)\sim\Gamma}[(y - v)\one(v\in I_{r,2j}^\varepsilon)]| + |\E_{(v,y)\sim\Gamma}[(y - v)\one(v\in I_{r,2j + 1}^\varepsilon)]|\Big]\\
& \ge \E_{r}\Big[\sum_{j\in \Z}|\E_{(v,y)\sim\Gamma}[(y - v)\one(v\in I_{r,2j}^\varepsilon\cup  I_{r,2j + 1}^\varepsilon)]|\Big]\\
& = \E_{r}\Big[\sum_{j\in \Z}|\E_{(v,y)\sim\Gamma}[(y - v)\one(v\in I_{r,j}^{2\varepsilon})]|\Big]\\
& = \RintCE(\Gamma,2\varepsilon).\qedhere
\end{align*}
\end{proof}
\begin{lemma}
\label{lm:intCE-uniform-conv}
There exists an absolute constant $C > 0$ with the following property. For $\varepsilon_0,\varepsilon_1,\delta\in (0,1/2)$, let $n\in \Z_{>0}$ satisfy $n \ge C\varepsilon_0^{-2}(\log(1/\delta) + \varepsilon_1^{-1})$. 
Let $(v_1,y_1),\ldots,(v_n,y_n)$ be drawn i.i.d.\ from a distribution $\Gamma$ over $[0,1]\times\{0,1\}$.
Then with probability at least $1-\delta$, for every $\varepsilon \ge \varepsilon_1$ and every $r\in [0,\varepsilon)$, we have
\begin{equation}
\label{eq:intCE-uniform-conv-1}
\left|\sum_{j\in \Z}\left|\frac 1n\sum_{\ell = 1}^n(y_\ell - v_\ell)\one(v_\ell \in I_{r,j}^{\varepsilon})\right| - \sum_{j\in \Z}\left|\E_{(v,y)\sim \Gamma}[(y - v)\one(v\in I_{r,j}^{\varepsilon})]\right|\right| \le \varepsilon_0.
\end{equation}
\end{lemma}
\begin{proof}[Proof of \Cref{lm:intCE-uniform-conv}]
For every $\varepsilon \ge \varepsilon_1$ and $r\in [0,\varepsilon)$, define a class $B_r^{\varepsilon}$ consisting of functions $b:[0,1]\to\{-1,1\}$ such that for every $j\in \Z$, there exists $b_j\in \{-1,1\}$ such that $b(v) = b_j$ for every $v\in [0,1]\cap I_{r,j}^{\varepsilon}$. It is now clear that
\begin{align}
\sum_{j\in \Z}\left|\frac 1n\sum_{\ell = 1}^n(y_\ell - v_\ell)\one(v_\ell \in I_{r,j})\right| & = \sum_{j\in \Z}\sup_{b_j\in \{-1,1\}}\frac 1n\sum_{\ell = 1}^n(y_\ell - v_\ell)\one(v_\ell \in I_{r,j})b_j\notag \\
& = \sup_{b\in B_r^\varepsilon}\frac 1n\sum_{\ell = 1}^n(y_\ell - v_\ell)b(v_\ell),\label{eq:intCE-uniform-conv-2}
\end{align}
and
\begin{align}
\sum_{j\in \Z}\left|\E[(y - v)\one(v\in I_{r,j})]\right| & = \sum_{j\in \Z}\sup_{b_j\in \{-1,1\}}\E[(y - v)\one(v\in I_{r,j})b_j]\notag \\
& = \sup_{b\in B_r^\varepsilon}\E[(y - v)b(v)]. \label{eq:intCE-uniform-conv-3}
\end{align}
It is clear that the functions class $B:=\bigcup_{\varepsilon \ge \varepsilon_1}\bigcup_{r\in [0,\varepsilon)}B_r^\varepsilon$ has VC dimension $O(\varepsilon_1^{-1})$. Let $G$ be the function class consisting of functions $g:\{0,1\}\times[0,1] \to [-1,1]$ such that there exists $b\in B$ satisfying
\[
g(y,v) = (y - v)b(v) \quad \text{for every }(y,v)\in \{0,1\}\times [0,1].
\]
The fact that $B$ has VC dimension $O(\varepsilon_1^{-1})$ implies that $G$ has pseudo-dimension $O(\varepsilon_1^{-1})$. By standard uniform convergence results \cite{Li2000ImprovedBO}, with probability at least $1-\delta$, 
\begin{equation}
\label{eq:intCE-uniform-conv-4}
\left|\frac 1n\sum_{\ell = 1}^n(y_\ell - v_\ell)b(v_\ell) - \E[(y - v)b(v)]\right| \le \varepsilon_0 \quad \text{for every }b\in B.
\end{equation}
By \eqref{eq:intCE-uniform-conv-2} and \eqref{eq:intCE-uniform-conv-3}, the inequality \eqref{eq:intCE-uniform-conv-4} implies that \eqref{eq:intCE-uniform-conv-1} holds for every $\varepsilon \ge \varepsilon_1$ and $r\in [0,\varepsilon)$.
\end{proof}
\begin{proof}[Proof of \Cref{thm:estimate-intce}]
By \Cref{lm:intCE-uniform-conv}, with probability at least $1-\delta/2$, for every $k = 0, \ldots,k^*$,
\[
\left|\wh\RintCE(\Gamma,2^{-k})  - \frac 1m \sum_{s = 1}^m \sum_{j\in \Z}|\E_{(v,y)\sim\Gamma}[(y - v)\one(v\in I_{r_s,j}^\varepsilon)]|\right| \le \varepsilon/4.
\]
By the Chernoff bound and the union bound, with probability at least $1-\delta/2$, for every $k = 0,\ldots,k^*$,
\[
\left|\frac 1m \sum_{s = 1}^m \sum_{j\in \Z}|\E_{(v,y)\sim\Gamma}[(y - v)\one(v\in I_{r_s,j}^\varepsilon)]| - \RintCE(\Gamma,2^{-k})\right| \le \varepsilon/4.
\]
Combining the above inequalities using the union bound, with probability at least $1 - \delta$, for every $k = 0,\ldots,k^*$,
\[
|\wh\RintCE(\Gamma,2^{-k}) - \RintCE(\Gamma,2^{-k})| \le \varepsilon/2.
\]
This implies that
\begin{align*}
\SintCE(\Gamma) & \le \min_{k = 0,\ldots,k^*}\left(\RintCE(\Gamma,2^{-k}) + 2^{-k}\right)\\
& \le \min_{k = 0,\ldots,k^*}\left(\wh\RintCE(\Gamma,2^{-k}) + 2^{-k}\right) + \varepsilon/2\\
& = \wh\SintCE(\Gamma) + \varepsilon/2.
\end{align*}
It remains to prove that $\SintCE(\Gamma) \ge \wh\SintCE(\Gamma) - \varepsilon$.
For every $k\in \Z_{\ge 0}$, if $k \le k^*$, we have
\[
\RintCE(\Gamma,2^{-k}) + 2^{-k} \ge \wh\RintCE(\Gamma,2^{-k}) + 2^{-k} - \varepsilon/2 \ge \wh\SintCE(\Gamma) - \varepsilon/2.
\]
If $k \ge k^*$, we have
\begin{align*}
\RintCE(\Gamma,2^{-k}) + 2^{-k} & \ge \RintCE(\Gamma,2^{-k^*})\tag{by \Cref{claim:int-monotone}}\\
 & \ge \wh\RintCE(\Gamma,2^{-k^*}) - \varepsilon/2\\
& \ge \wh\SintCE(\Gamma) - 2^{-k^*} - \varepsilon/2\\
& \ge \wh \SintCE(\Gamma) - \varepsilon.
\end{align*}
Therefore, we have $\RintCE(\Gamma,2^{-k}) + 2^{-k} \ge \wh\SintCE(\Gamma) - \varepsilon$ for every $k\in \Z_{\ge 0}$, which implies that $\SintCE(\Gamma) \ge \wh\SintCE(\Gamma) - \varepsilon$, as desired.
\end{proof}

We prove \Cref{lm:int-discontinuous} using the following example.
We choose $\mX = \{a,b,c,d\}$. Define $\alpha = 1/6$ and $\beta = 1/48$ and let $\varepsilon \in (0,\beta)$ be a parameter. We choose the distribution $\mD$ over $\mX\times\{0,1\}$ such that the marginal distribution of $x$ in $(x,y)\sim\mD$ is the uniform distribution over $\mX$, and the conditional distribution of $y$ given $x$ is defined according to the $\E[y|x]$ values in \Cref{table:discontinuous}.
We consider two predictors $f_1,f_2:\cX\to[0,1]$ also defined in \Cref{table:discontinuous}.

\begin{table}[ht]
\centering
\begin{tabular}[t]{lllll}
\hline
$x$ & $a$ & $b$ & $c$ & $d$\\
\hline
$f_1(x)$ & $1/2 - \beta$ & $1/2$ & $1/2$ & $1/2 + \beta$\\
$f_2(x)$ & $1/2 - \beta$ & $1/2 - \varepsilon$ & $1/2 + \varepsilon$ & $1/2 + \beta$\\
$\E_{\mD}[y|x]$ & $1/2 - \beta + \alpha $ & $1/2 - \varepsilon - \alpha$ & $1/2 + \varepsilon + 2\alpha$ & $1/2 + \beta - 2\alpha$\\
\hline
\end{tabular}
\caption{Discontinuity Example for $\intCE$}
\label{table:discontinuous}
\end{table}
The following claim follows immediately from the definitions of $f_1$ and $f_2$:
\begin{claim}
As $\varepsilon \to 0$, $f_2$ converges to $f_1$ uniformly.
\end{claim}
\begin{claim}
$\intCE_\mD(f_2) \le \beta$.
\end{claim}
\begin{proof}
Any interval partition $\mI$ that contains the two intervals $[1/2 - \beta,1/2)$ and $[1/2, 1/2 + \beta)$ satisfy $\intCE_\mD(f_2,\mI) = 0$ and $w_{\mD_{f_2}}(\mI) = \beta$.
\end{proof}
\begin{claim}
$\intCE_\mD(f_1) \ge 2\beta$.
\end{claim}
\begin{proof}
Consider any interval partition $\mI$. If $1/2 - \beta$ and $1/2 + \beta$ are not in the same interval, $\intCE_\mD(f_1,\mI) \ge \alpha/4 = 2\beta$. If $1/2 - \beta$ and $1/2 + \beta$ are in the same interval, then $w_{\mD_{f_1}}(\mI) \ge 2\beta$.
\end{proof}
\begin{proof}[Proof of \Cref{lm:int-discontinuous}]
    The lemma follows immediately from combining the three claims above.
\end{proof}

\subsection{Proofs from \Cref{sec:kernels}}
\label{sec:app-ker-proofs}

\begin{proof}[Proof of \Cref{lem:kernel-dual-definition}]
By duality:
\begin{align*}
\| \E (y - v) \phi(v) \|_{\cH} & = \sup_{w \in B_{\cH}} \inprod{ \E (y - v) \phi(v), w}_{\cH} \\
& = \sup_{w \in B_{\cH}} \E (y-v) \inprod{\phi(v), w}_{\cH} \\
& = \sup_{w \in B_{\cH}} \E (y-v) w(v)\\
& = \wCE^{B_{\cH}}(\Gamma).
\qedhere
\end{align*}
\end{proof}

\begin{proof}[Proof of~\Cref{lem:laplace-rff}]
The Fourier transform (or characterstic function, using probabilistic terminology) of $h(v) = \exp(-|v|)$ is exactly the probability density function of the Cauchy distribution $\hat{h}(\omega) = \frac{1}{1 + \omega^2}$.

Using inverse Fourier transform formula, we get 
\begin{equation*}
    \exp(-|v|) = \int \exp(-i \omega v) \frac{1}{1 + \omega^2} \, \d \omega = \E_{\omega \sim \mathrm{Cauchy}(1)} \exp(- i \omega v).
\end{equation*}
We can now calculate $\E z z^*$ explicitly. The expectation of the $(i,j)$ entry is
\begin{equation*}
    \E (z z^*)_{ij} = \E z_{i} \overline{z_j} = \E_{\omega \sim \mathrm{Cauchy}} \exp(-i \omega (v_i - v_j)) = \exp(-|v_i - v_j|) = M_{ij}.\qedhere
\end{equation*}
\end{proof}

\begin{proof}[Proof of \Cref{thm:laplace-binning}]
For fixed $\delta \in (0, 1), \tau \in [0, \delta)$, let $t_{\tau,\delta}(i) := \lfloor \frac{v_i + \tau}{\delta} \rfloor$ be the index of the bin into which the sample $(v_i, y_i)$ is assigned (with bin sizes $\delta$, and random shift $\tau$).

The \textsc{LaplaceBinningEstimate} algorithm takes random $\delta$ sampled from the distribution $\mathrm{Gamma}(k=2, \theta=1)$, random $\tau \sim \mathrm{Unif}(0, \delta)$ and returns 
\[ T_{\delta, \tau} := \frac{1}{k^2} \sum_{k} \left(\sum_{i \in t^{-1}_{\tau, \delta}(k)} (v_i - y_i)\right)^2.\] 
We can expand the square and write it as
\begin{equation*}
    T_{\delta, \tau} = \frac{1}{n^2} \sum_{i \leq n, j \leq n} (v_i - y_i)(v_j - y_j) \mathbf{1}[t_{\tau, \delta}(i) = t_{\tau,\delta}(j)].
\end{equation*}

First, let us prove that in fact $T_{\delta, \tau} \leq 1$. Indeed, clearly $|v_i - y_i| \leq 1$, and $\mathbf{1}[t_{\tau, \delta}(i) = t_{\tau, \delta}(j)] \leq 1$, and therefore $T_{\delta, \tau} \leq \frac{1}{n^2} \sum_{i\leq k, j\leq k} 1 \leq 1$.

We will now prove that the quantity $T_{\delta, \tau}$ is in fact an unbiased estimator of $\ekCE(S)^2$.

The expression $\mathbf{E}_{\delta, \tau} \mathbf{1}[t_{\tau, \delta}(i) = t_{\tau, \delta}(j)]$ is just the probability that $v_i$ and $v_j$ will land in the same bin, when we sample random bin-size $\delta$ and random shift $\tau$. We wish to show that this probability satisfies \[\E_{\delta, \tau}[ \mathbf{1}{t_{\tau, \delta}(i) = t_{\tau, \delta}(j)}] = \exp(-|v_i - v_j|).\]   Indeed, if this is the case, we will have 
\begin{equation*}
    \E_{\tau, \delta} T_{\delta, \tau} = \sum_{i,j} (v_i - y_i)(v_j - y_i) \E \mathbf{1}[t_{\tau,\delta}(i) = t_{\tau,\delta}(j)] = \sum_{i,j} (v_i - y_i)(v_j - y_j) \exp(-|v_i - v_j|) = \ekCE_{Lap}(S)^2.
\end{equation*}

For a fixed size of the bin $\delta$, we have $\E_{\tau \sim [0, \delta]} \mathbf{1}[t(i) = t(j)] = \max(1 - \frac{|v_i - v_j|}{\delta}, 0).$ 

The probability density function of the Gamma distribution with the shape parameter $k=2$ and scale parameter $\theta = 1$ is $p(v) = v \exp(-v) \mathbf{1}[v\geq 0]$ . We now have
\begin{align*}
    \E_{\delta} \max(1 - \frac{|u - v|}{\delta}, 0) & =
    \int_{0}^{\infty} t \exp(-t) \max(1 - \frac{|u - v|}{t}, 0) \, \mathrm{d} t \\
    & = \int_{|u - v|}^{\infty} t \exp(-t) (1 - \frac{|u - v|}{t}) \, \d t \\
    & = \int_{|u - v|}^{\infty} t \exp(-t)\, \d t - \int_{|u - v|}^{\infty} |u - v| \exp(-t) \, \d t \tag{integrating first term by parts}\\
    & = [- t \exp(-t)]_{|u - v|}^{\infty} - \int_{|u - v|}^{\infty} \exp(-t) \, \d t - \int_{|u - v|}^{\infty} |u - v| \exp(-t) \, \d t \\
    & = \exp(-|u - v|).\qedhere
\end{align*}
\end{proof}
\section{Gaussian Kernel Appendix \label{sec:gaussian-kernel}}

In this appendix we will discuss the $\kCE$ measure with respect to the Gaussian kernel $K(u,v) = \exp(-(u-v)^2).$ The associated RKHS has the following explicit description.

\begin{fact}[\cite{berlinet2011reproducing}]
For the Gaussian kernel $K_\gauss(u,v) := \exp(-(u-v)^2)$, we have associated RKHS $\cH_\gauss = \{ h : \R \to \R  : \int \hat{h}(\omega)^2 \exp(\omega^2) \d \omega < \infty \}$, where $\hat{h}$ denotes the Fourier transform of $h$. 
The associated inner product is given by
\begin{equation*}
    \inprod{h_1, h_2}_{K_\gauss} = \int_{-\infty}^{\infty} \hat{h}_1(\omega) \hat{h}_2(\omega)\exp(\omega^2) \d \omega, 
\end{equation*}
in particular, for function $h: \R \to \R$,
\begin{equation*}
    \|h\|_{K_\gauss}^2 = \int_{-\infty}^{\infty} \hat{h}(\omega)^2 \exp(\omega^2) \d \omega = \sum_k \frac{\|h^{(k)}\|_2^2}{k!}.
\end{equation*}
\end{fact}

We will show that if we chose Gaussian kernel, the associated calibration measure $\kCEG$ does not satisfy robust soundness.

\begin{theorem}
\label{thm:gauss-lb}
For every $\varepsilon$, there is a distribution $\Gamma_{\varepsilon}$ over $[0,1] \times \{0, 1\}$, such that $\scerror(\Gamma_\varepsilon) \geq \Omega(\varepsilon^{5/2})$, and $\kCEG(\Gamma_\varepsilon) \leq \Oh(\exp(-\varepsilon^{-1}))$.
\end{theorem}

On the other hand it satisfies significantly weaker continuity at $0$ property.
\begin{theorem}
\label{thm:gauss-continuity}
    For any $f$, if $\scerror(f) = \delta$, then
    \begin{equation*}
        \exp(-\Oh(\delta^{-4})) \leq \kCEG(f) \leq \Oh(\sqrt{\delta}).
    \end{equation*}
\end{theorem}

\subsection{Lower Bound}
In order to prove~\cref{thm:gauss-lb}, we will use the following lemma.
\begin{lemma}
\label{lem:gauss-lb-helper}
For any $\varepsilon$, there exist $\Oh(\varepsilon^{-1})$-Lipschitz function $h_{\varepsilon} : \R \to [-1, 1]$, satisfying the following three properties
\begin{align}
    \int \hat{h}_{\varepsilon}(\omega)^2 \exp(-\omega^2) & \leq \exp(-\varepsilon^{-1}) \label{eq:property-1} \\
    \|h_{\varepsilon} - h_{\varepsilon} \mathbf{1}_{[-1/4, 1/4]}\|_2 & \leq  \exp(-\Omega(\varepsilon^{-1})) \label{eq:property-2} \\
    \int_{-\sqrt{\varepsilon}}^{\sqrt{\varepsilon}} h_{\varepsilon}^2 & \geq \Omega(\sqrt{\varepsilon}). \label{eq:property-3}
\end{align}
\end{lemma}
\begin{proof}
Let us denote $\phi_{\gamma}(t) := \exp(-t^2/\gamma^2)$. We define 
\begin{equation*}
    h_{\varepsilon}(t) := \cos(t/\varepsilon) \phi_{\sqrt{\varepsilon}}(t).
\end{equation*}

This function is indeed $\Oh(\varepsilon^{-1})$-Lipschitz, since
\begin{equation*}
    |h_\varepsilon'(t)| = \left|\frac{1}{\varepsilon} (\sin (t/\varepsilon)) \phi_{\sqrt{\varepsilon}} - 2 \frac{t}{\varepsilon} \cos(t/\varepsilon) \phi_{\sqrt{\varepsilon}}(t)\right| \leq \frac{3}{\varepsilon}. 
\end{equation*}

Let us denote $\omega_0 := \frac{1}{\varepsilon}$. By standard properties of Fourier transform (i.e. $\widehat{fg} = \widehat{f} * \widehat{g}$, $\widehat{\cos(\omega_0 t)} = \delta_{-\omega_0} + \delta_{\omega_0}$ and $\widehat{\phi_{\varepsilon}} = \phi_{\varepsilon^{-1}}$), we have
\begin{equation*}
    \widehat{q_{\varepsilon}}(\omega) = \exp(-\frac{(\omega - \omega_0)^2}{\omega_0}) + \exp(-\frac{(\omega + \omega_0)^2}{\omega_0}).
\end{equation*}

We will show that
\begin{equation*}
    \int \exp(-2 \frac{(\omega - \omega_0)^2}{\omega_0} - \omega^2) \, \d \omega \leq 2 \exp(-\omega_0).
\end{equation*}

Indeed, we can rewrite 
\begin{equation*}
    2\frac{(\omega - \omega_0)^2}{\omega_0} + \omega^2 = (\omega - \alpha)^2 - \beta + 2\omega_0, 
\end{equation*}
where $\alpha \in [0,2], \beta\in [0, 1]$ depend on $\omega_0$. Therefore
\begin{equation*}
    \int \exp(-2\frac{(\omega - \omega_0)^2}{\omega_0} - \omega^2)\,\d \omega
    = \exp(-2 \omega_0 + 1) \int \exp(-(\omega - \alpha)^2) \leq \Oh(\exp(-2 \omega_0)). 
\end{equation*}

This implies
\begin{equation*}
    \int \widehat{q_{\varepsilon}}(\omega)^2 \exp(-\omega^2)\,\d \omega \leq \Oh(\exp(-\Omega(\omega_0))),
\end{equation*}
proving~\eqref{eq:property-1}.

In order to prove~\eqref{eq:property-2}, we have $|h_{\varepsilon}(t)| \leq \exp(-\omega_0 t^2)$, implying
\begin{equation*}
    \int_{1/4}^{\infty} h_{\varepsilon}(t)^2 \leq \exp(-\omega_0/8),
\end{equation*}
and similarly for the other tail.

Finally, for~\eqref{eq:property-3}, since $\phi_{\omega_0}(t) \geq \Omega(1)$ for $t \in [-\sqrt{\varepsilon}, \sqrt{\varepsilon}]$ we have
\begin{align*}
    \int_{-{\sqrt{\varepsilon}}}^{\sqrt{\varepsilon}} h_{\varepsilon}(t)^2 &\gtrsim \int_{-\sqrt{\varepsilon}}^{\sqrt{\varepsilon}} \sin^2(t/\varepsilon) \\
    & \gtrsim \Omega(\sqrt{\varepsilon}).\qedhere
\end{align*}
\end{proof}

\begin{proof}[Proof of \Cref{thm:gauss-lb}]
Distribution $\Gamma_{\varepsilon}$ of $(v,y)\in [0,1]\times\{0,1\}$ will be given in the following way: we will sample $v$ uniformly from $[1/4, 3/4]$, and sample $y | v$ such that $\E[y - v|v] = h_{\varepsilon}(v-1/2)/4$, where $h_{\varepsilon}$ is a function as in~\Cref{lem:gauss-lb-helper}.

Let us define $r(v) := \E[y - v | v]$. First, we need to check, if it provides a well-specified distribution, i.e. if $\E[y|v] \in [0,1]$. Indeed, since $v \in [1/4, 3/4]$, and $h_\varepsilon \in [-1, 1]$, we have $v + h_{\varepsilon}/4 \in [0, 1]$.

Now, we wish to upper bound $\kCEG(\Gamma_\varepsilon).$ For any $w \in B_{\cH_\gauss}$, we have
\begin{align*}
    \int_{-1/4}^{1/4} r(v) w(v) \, \mathrm{d} v & \leq 
    \int_{-\infty}^{\infty} r(v) w(v) \, \mathrm{d} v + \exp(-\Omega(\omega_0)),
\end{align*}
using~\eqref{eq:property-2}.

We will focus on bounding this latter integral.
\begin{align*}
    \int_{-\infty}^{\infty} r(v) w(v) & = \int \hat{r}(\omega) \hat{w}(\omega) \, \mathrm{d} v \\
    & = \int \hat{r}(\omega) \exp(-\omega^2/2) \hat{w}(\omega)\exp(\omega^2/2) \, \mathrm{d} \omega\\
    & \leq \left(\int \hat{r}(\omega)^2 \exp(-\omega^2)\,\d \omega\right) \left(\int \hat{w}(\omega)^2 \exp(\omega^2) \,\d\omega\right) \\
    & \leq C \exp(-\Omega(\omega_0))
\end{align*}
where the last inequality follow from~\eqref{eq:property-1} in \Cref{lem:gauss-lb-helper} and the fact that $\|w\|_{\cH_\gauss} \leq 1$.

Finally, to show that $\scerror(v) \geq \varepsilon^{5/2}$, we take a test function $w(v) := \varepsilon r(v)$ which is Lipschitz, and observe that 
\begin{equation*}
    \E (y-v) w(v) = \int_{1/4}^{3/4} r(v)^2 \geq \varepsilon^{5/2}
\end{equation*}
by~\eqref{eq:property-3}.
\end{proof}
\subsection{Weak Continuity at Zero}

\begin{lemma}
\label{lem:gauss-upper-cont}
    For any $1$-Lipschitz function $w : [0,1] \to [-1,1]$, and any $\varepsilon$, there is a function $\tilde{w} : \R \to \R$ satisfying 
    \begin{equation*}
        \|\tilde{w}|\|_{\cH_\gauss} \leq \exp(-\varepsilon^{-4})
    \end{equation*}    
    and 
    \begin{equation*}
            \forall t\in [0,1], |w(t) - \tilde{w}(t)| \leq \varepsilon.
    \end{equation*}
\end{lemma}
\begin{proof}
Let us take a $1$-Lipschitz extension of $w$ to the entire $\R$, such that $\supp(w) \subset [-1, 2]$. We can assume without loss of generality that $w$ is differentiable (indeed, otherwise, as in the proof of \Cref{lem:lipschitz-are-laplace-bounded}, we can convolve it with mollifier $g_{\varepsilon}$ to get a $1$-Lipschitz and differentiable approximation of $w$ up to error $\varepsilon/2$.

Note that
\begin{equation}
    \int \hat{w}(\omega)^2 \omega^2 \, \d \omega = \|w'\|_2^2 \leq 3. \label{eq:first-derivative-bound}
\end{equation}

Let us now let fix some cut-off frequency $\omega_0$, and define
\begin{equation*}
    \hat{\tilde{w}}(\omega) = \hat{w}(\omega) \mathbf{1}_{[-\omega_0, \omega_0]}(\omega).
\end{equation*}

We have $\tilde{w}(t)$ given by the inverse Fourier transform
\begin{equation*}
    \tilde{w}(t) = \int_{-\infty}^{\infty} \exp(-i \omega t) \hat{\tilde{w}}(\omega) \, \d \omega.
\end{equation*}    
    We wish to show that $|\tilde{w}(t) - w(t)| \leq \varepsilon$. We have
    \begin{align*}
        w(t) - \tilde w(t) & = \int \mathbf{1}[|\omega| \geq \omega_0] \hat{w}(\omega) \exp(-i \omega t)\, \d t \\
        & = \int  \hat{w}(\omega) \omega \exp(-i\omega t) \omega^{-1}\mathbf{1}[|\omega| \geq \omega_0] \, \d t \\
        & \leq \left(\int \hat{w}(\omega)^2 \omega^2 \, \d t\right)^{1/2} \left(\int \exp(-i \omega t) \omega^{-2} \mathbf{1}[|\omega| > \omega_0]\, \d t\right)^{1/2}.
    \end{align*}
    The first term is bounded by~\eqref{eq:first-derivative-bound}, for the second
    \begin{equation*}
        \int_{\omega_0}^{\infty} \exp(-i \omega t) \omega^{-2} \, \d \omega \leq \int_{\omega_0}^\infty \omega^{-2} \, \d \omega = \frac{1}{\omega_0}.
    \end{equation*}
    
    Therefore $|w(t) - \tilde{w}(t)| \leq \frac{6}{\sqrt{\omega_0}}$.
    
    Taking $\omega_0 = \varepsilon^2$, we obtain the desired approximation.
    
    Finally, we need to show the bound on $\|\tilde{w}\|_{\cH_\gauss}$. We have
    \begin{align*}
        \|\tilde{w}\|_{\cH_\gauss}^2 & = \int \hat{\tilde{w}}(\omega)^2 \exp(\omega^2) \, \d \omega \\
        & \leq \exp(\omega_0^2) \int \hat{\tilde{w}}(\omega)^2 \, \d \omega \\
        & \leq 3 \|w\|_2^2 \exp(\omega_0^2) \leq 9 \exp(\omega_0^2).\qedhere
    \end{align*}
\end{proof}

\begin{lemma}
\label{lem:gauss-lower-cont}
For any two predictors $f, g$ on the same space $\cX$ satisfying $\E |f-g| \leq \varepsilon$, we have
\begin{equation*}
    |\kCEG(f) - \kCEG(g)| \leq \Oh(\sqrt{\varepsilon}).
\end{equation*}
\end{lemma}
\begin{proof}
The proof is identical to~\Cref{lem:laplace-kernel-continous}.

We need to show
\begin{equation*}
    \| \E(y-f) \phi(f) - (y-g) \phi(g) \|_{\cH} \leq \Oh(\sqrt{\varepsilon}),
\end{equation*}
by triangle inequality, and Jensen inequality, we need to show only
\begin{equation*}
    \left(\E \| \phi(f) - \phi(g)\|_{\cH}^2\right)^{1/2} + \left(\E \|f \phi(f) - g \phi(g)\|_{\cH}^2\right)^{1/2} \leq \Oh(\sqrt{\varepsilon}).
\end{equation*}
Indeed,
\begin{equation*}
    \E \| \phi(f) - \phi(g)\|_{\cH}^2 = 2 - \E 2 K(f,g) \leq \E |f-g|^2 \leq \E |f-g| \leq \varepsilon,
\end{equation*}
and similarly
\begin{equation*}
    \E \| f\phi(f) - g\phi(g)\|_{\cH}^2 = \E[ f^2 + g^2 - 2fg K(f,g)] \leq \E[ f^2 + g^2 - 2fg(1 - |f-g|^2)] \leq \E[3 (f-g)^2] \leq 3 \varepsilon. 
\end{equation*}
\end{proof}

\begin{proof}[Proof of~\Cref{thm:gauss-continuity}]
The statement follows from~\Cref{lem:gauss-lower-cont} and \Cref{lem:gauss-upper-cont} in the exactly same way as \Cref{thm:laplace-kernel-bounds}.
\end{proof}
\end{document}